\documentclass{article}
\usepackage{amsthm}

\usepackage{graphicx}
\usepackage[utf8]{inputenc} % allow utf-8 input
\usepackage[T1]{fontenc}    % use 8-bit T1 fonts
\usepackage{hyperref}       % hyperlinks
\usepackage{url}            % simple URL typesetting
\usepackage{booktabs} 
\usepackage{float}  % professional-quality tables
\usepackage{amsfonts}       % blackboard math symbols
\usepackage{amsmath}
\usepackage{amssymb}
\usepackage{nicefrac}       % compact symbols for 1/2, etc.
\usepackage{microtype}      % microtypography
\usepackage{xcolor} 
\usepackage{bbm}

\newtheorem{thm}{Theorem}[section]
\newtheorem{cor}[thm]{Corollary}
\newtheorem{lem}[thm]{Lemma}
\newtheorem{prop}[thm]{Proposition}

\newtheorem{defn}{Definition}[section]
\newtheorem{asmp}{Assumption}[section]
\newtheorem{rmk}[defn]{Remark}

\newcommand{\eps}{\epsilon}

\newcommand{\dist}{{\sf dist}}

\newcommand{\Law}{\mathrm{Law}}

\newcommand{\TV}{\mathrm{TV}}
\newcommand{\op}{\mathrm{op}}

\newcommand{\E}{\mathbb E}

\DeclareMathOperator*{\argmax}{arg\,max}

\def\lb{\mathopen{}\mathclose\bgroup\left}
\def\rb{\aftergroup\egroup\right}

\numberwithin{equation}{section}

\def\to{\rightarrow}

\usepackage{footnote}

\def\eps{\varepsilon}

\def\E{{\bf E}}

\def\cA{\mathcal{A}}

\def\cF{\mathcal{F}}

\def\cX{\mathcal{X}}

\def\d{{\mathrm{d}}}

\def\E{\mathbb{E}}

\def\sE{{\mathbb{E}}}
\def\sF{{\mathbb{F}}}

\def\sP{\mathbb{P}}
\def\sQ{{\mathbb{Q}}}
\def\sR{{\mathbb R}}

\usepackage{subcaption}

\title{Expressivity and Statistical Trade-offs in Diffusion Policy Learning
}

\author{%
  Viet Vu \thanks{Stanford University, 
  Emails: \url{vietvu01@stanford.edu} and \url{renyuanxu@stanford.edu}}
  \and
  Renyuan Xu $^*$
  \and
   Jiacheng Zhang \thanks{Chinese University of Hong Kong, Email: \url{jiachengzhang@cuhk.edu.hk}} 
  \and
   Yufei Zhang \thanks{Imperial College London, Email: \url{yufei.zhang@imperial.ac.uk}} 
}
\date{\today}

\providecommand{\DIFdel}[1]{}

\providecommand{\DIFdelFL}[1]{}

\begin{document}
\maketitle

\begin{abstract}
Diffusion-based policies have recently emerged as powerful policy parameterizations for reinforcement learning, representing state-conditioned action distributions as terminal laws of diffusion processes with parameterized drifts. This terminal-law representation has shown substantial expressive flexibility in practice, enabling diffusion policies to model complex, multimodal, and highly non-Gaussian action distributions; however, it remains unclear what mathematically drives this expressivity and how to fully exploit it when the policy is learned from finite data. 
In this paper, we identify the drift Lipschitz budget \(K\) as a central quantity governing the expressivity and statistical behavior of diffusion policies. We quantify expressivity through approximation: diffusion policies with \(K\)-Lipschitz drifts can concentrate near optimal deterministic policies and achieve value approximation error of order \(1/K\); moreover, we prove a matching lower bound under nondegenerate diffusion noise. This increased expressivity comes with a statistical cost. When the drift is parameterized by neural networks, increasing \(K\) improves approximation but increases statistical complexity. Balancing these two terms yields a finite-sample performance gap of order \(\tilde{\mathcal O}(n^{-2/(m+6)})\) for generic neural-network drifts, and a sharper rate \(\tilde{\mathcal O}(n^{-2/(m+4)})\) for one-sided dissipative drift classes, where $n$ is the sample size and $m$ is the dimension of the state space. Numerical experiments provide empirical evidence for the sample-dependent trade-off in $K$, supporting both theoretical regimes. Our framework also suggests a practical implementation principle: choose the diffusion budget \(K\) according to the available sample size, and then select a neural-network architecture with the corresponding fixed Lipschitz coefficient.

%Diffusion-based policies have recently emerged as powerful policy parameterizations for reinforcement learning, but their theoretical properties remain largely unexplored. In this paper, we study the expressivity and finite-sample behavior of diffusion policies generated by state-conditioned diffusion processes with parameterized drifts. We show that diffusion policies with $K$-Lipschitz drifts approximate optimal policies with error of order $1/K$, where $K$ measures the complexity of the policy class, and establish a matching lower bound showing that this rate is optimal. This expressivity gain comes at a statistical cost: a larger Lipschitz budget improves approximation but makes learning more challenging, leading to a finite-sample trade-off. When the drift is parameterized by neural networks, we decompose the policy error into three components: diffusion approximation error, neural network realization error, and statistical estimation error. Balancing these terms yields a performance gap of order \(\widetilde{\mathcal O}(n^{-2/(m+6)})\) for a generic neural drift class. We also identify a sharper regime for one-sided dissipative neural drift classes: the improved stability removes the extra \(K\)-dependence in the statistical term and leads to the faster rate \(\widetilde{\mathcal O}(n^{-2/(m+4)})\). Numerical experiments illustrate both regimes and support the resulting sample-aware choice of the diffusion budget \(K\).

\end{abstract}

\section{Introduction}
\label{sec:introduction}

Many reinforcement learning (RL) problems concern learning a stochastic policy that optimizes an expected objective through repeated interaction with the environment \cite{sutton1998reinforcement}. Policy-gradient   and its variants   are among the most widely used approaches for this task~\cite{haarnoja2018soft, kakade2001natural, 
schulman2015trust, schulman2017proximal,sutton1999policy}. Their performance, however,  depends critically on the choice of policy parameterization. An effective policy class
must satisfy three requirements:
(i) it should be expressive enough to approximate good, or even optimal,  policies;
(ii) it should allow efficient sampling of actions
during interaction with the environment; and
(iii)  it should support tractable
likelihood evaluation for constructing policy-gradient estimators
(e.g.,~\cite{sutton1999policy}).

Designing such policy classes becomes especially challenging beyond finite action spaces. In continuous-action RL, classical methods often rely on
tractable stochastic policy families, such as Gaussian, Beta,
Gaussian-mixture, or exponential-family policies, because these distributions
are easy to sample from and have tractable likelihoods
\cite{haarnoja2018soft,hansen2023idql,schulman2015trust,schulman2017proximal}. This computational convenience, however, comes with limited expressivity.
Gaussian policies may be poorly matched to bounded or constrained action spaces, log-concave families cannot represent many multimodal action distributions, and mixture models can introduce numerical and optimization difficulties. As a result, these simple policy classes can struggle to represent the complex, multimodal, or strongly non-Gaussian action distributions that arise in modern
RL tasks, such as multi-stage manipulation
\cite{chi2025diffusion,ke20243d,scheikl2024movement} and locomotion
\cite{huang2024diffuseloco}.

Diffusion-based policies provide a more expressive alternative by representing state-conditioned action distributions as terminal laws of diffusion processes with parameterized drifts. The key idea is to transform a simple reference distribution into a complex action law through stochastic dynamics.  This approach has shown
strong expressive power in generative modeling, including image and video generation \cite{ho2020denoising,sohl2015deep,song2019generative}. Building on these advances, recent works have used diffusion models as policy parameterizations
in control and RL
\cite{balim2025model,chi2025diffusion,ding2024diffusion,ke20243d,ma2025efficient,psenka2024learning,
ren2025diffusion,scheikl2024movement,wang2024diffusion}. These methods represent the action distribution at each state through a parameterized diffusion process. This allows the policy to
capture rich and potentially multimodal action distributions that are difficult to express with standard policy families, while still allowing actions to be sampled by simulating the diffusion process.

Despite this empirical success, the theoretical understanding of diffusion-based policies in RL remains limited. Existing theory for diffusion models mainly focuses on generative modeling and does not directly explain what drives the expressivity of diffusion policies for decision-making objectives. In particular, it remains unclear which structural quantity controls their ability to approximate optimal policies, and how to fully exploit this expressive power when the policy is learned from finite data. This leads to two fundamental questions:
\begin{enumerate}
    \item {\bf Expressivity of diffusion policies}: What mathematically governs the ability of diffusion policies to approximate optimal policies for a given RL objective?
    \item 
{\bf Impact on learning and statistical guarantees:} 
How does the increased expressivity of diffusion policies affect the finite-sample
statistical error of policy learning?
\end{enumerate}
Answering these questions is challenging because the expressivity of diffusion policies is not merely a qualitative property of the representation. It depends on the Markov decision process (MDP) structure, the regularity of the optimal policy, and the quantitative complexity of the diffusion dynamics. Moreover, increasing expressivity may improve approximation while enlarging the statistical complexity of the learned policy class. This calls for a finite-sample theory that identifies the structural quantity driving expressivity and quantifies its role in the resulting approximation--estimation trade-off.

\subsection{Our contributions}
In this paper, we identify the drift Lipschitz budget \(K\) as a central quantity governing the expressivity and statistical behavior of diffusion policies. We consider diffusion policies defined by state-conditioned stochastic differential equations with parameterized drifts and given volatility,  and measure the complexity of the policy class through the drift Lipschitz constant $K$. Our main contributions are four-fold.

First, we establish the expressivity of diffusion policies through a sharp
\(K\)-dependent approximation theory. We prove that diffusion policies with
\(K\)-Lipschitz drifts can approximate deterministic optimal policies as
\(K\to\infty\), with mean-squared localization error of order \( \mathcal{O}(1/K)\)
(Lemma~\ref{lem:mean_rev}). We then show that this rate is optimal and cannot
be improved in general under nondegenerate diffusion noise
(Proposition~\ref{prop:lb_diffusion}). We further translate these localization
results into upper and lower bounds for RL value-function approximation,
namely Theorem~\ref{thm:ub_soft} and
Theorem~\ref{thm:lb_diffusion_initial}. 

Second, we show that this increased expressivity comes with a statistical cost when the drift is parameterized by ReLU neural networks. For a generic \(K\)-Lipschitz
neural drift class, the learned policy error is bounded by three terms:
diffusion approximation  \(1/K\), neural-network realization error  \(s^{-2/m}\),
and statistical estimation error  \(\sqrt{sK/n}\), where \(s\) is the network
class size, \(n\) is the sample size, and \(m\) is the state dimension. Balancing these terms yields a value gap of order
\(\widetilde{\mathcal O}(n^{-2/(m+6)})\)
(Theorem~\ref{thm:lipschitz-selector-complete-tradeoff}). We further show that,
if the neural drift class satisfies a one-sided dissipativity condition, the
stability estimate improves and the statistical term becomes \(\sqrt{s/n}\), up
to logarithmic factors. This gives the sharper rate
\(\widetilde{\mathcal O}(n^{-2/(m+4)})\); see
Remark~\ref{rem:dissipative-comparison-section4}. In both regimes, these rates are achieved by choosing the diffusion budget \(K\) depend properly on the sample size \(n\).

Third, we derive a policy-gradient formula for continuous-time diffusion
policies. Since the policy is defined implicitly as the
terminal law of a diffusion process, its log-density function (also known as the score function)   is generally not analytically tractable.  We show that
the score function can be substituted with a pathwise quantity derived from the diffusion process via Girsanov's theorem. This yields a tractable policy-gradient expression in terms of the drift sensitivity 
(Proposition~\ref{prop:policygradient}), and provides the training identity used in our numerical experiments.

Fourth, we validate the two statistical regimes through numerical experiments.
The first experiment trains a generic neural drift 
using our policy-gradient estimator in a nonlinear multi-step MDP setting. It illustrates the generic
sample-dependent scaling \(K\asymp n^{2/(m+6)}\). The second experiment uses a
mean-reverting, hence dissipative, diffusion-policy class in a controlled
contextual bandit setting. It isolates the approximation and estimation terms and
illustrates the sharper scaling \(K\asymp (n/\log n)^{2/(m+4)}\). Together,
the experiments provide empirical evidence for the sample-dependent role of \(K\), and suggest a practical implementation principle: choose the
diffusion budget \(K\) according to the available sample size and then a    $K$-Lipschitz  neural drift class.

\subsection{Related literature}

Our work is connected to several strands of literature on diffusion models, RL, and continuous-time stochastic analysis. The most closely related line studies diffusion models as flexible policy parameterizations for continuous-action RL. We also discuss diffusion-based planning and imitation learning, theory of score-based generative models, and the approximation, statistical learning, and PDE/SDE tools used in our analysis.

\iffalse
\paragraph{Diffusion policy in RL.} One line of work directly represents the policy by a conditional diffusion
model and trains or optimizes the parameters of the diffusion model \cite{
 ding2024diffusion, fang2025diffusion,
ma2025efficient, psenka2024learning,ren2025diffusion,wang2024diffusion,wang2023diffusion, yang2023policy}. A second line uses diffusion models primarily as samplers, generating actions from a given intermediate policy while training a classical RL algorithm \cite{chen2024score,chen2023offline,chen2024diffusion,   hansen2023idql, lu2023contrastive,
mao2024diffusion}. Both lines of work are largely empirical, but they consistently demonstrate the strong expressive power of diffusion policies on large-scale RL tasks. In contrast, we provide theoretical foundations for this empirical observation by studying how the diffusion Lipschitz budget controls both expressivity and finite-sample statistical error in diffusion policy learning. To the best of our knowledge, this is among the first theoretical results of this kind in the literature. 
\fi

\paragraph{Diffusion policy in RL.} One line of work directly represents the policy by a conditional diffusion
model and trains or optimizes the parameters of the diffusion model \cite{
 ding2024diffusion, fang2025diffusion,
ma2025efficient, psenka2024learning,ren2025diffusion,wang2024diffusion,wang2023diffusion, yang2023policy}. 
Recent extensions further develop likelihood-based on-policy diffusion policy optimization, continuous-time \(Q\)-score matching, and maximum-entropy diffusion-policy training via adjoint matching
\cite{ding2025genpo,hua2025continuous,thilges2026scalable}.
A second line uses diffusion models primarily as samplers, generating actions from a given intermediate policy while training a classical RL algorithm \cite{chen2024score,chen2023offline,chen2024diffusion,   hansen2023idql, lu2023contrastive,
mao2024diffusion}. 
Both lines of work are largely empirical, but they consistently demonstrate the strong expressive power of diffusion policies on large-scale RL tasks. In contrast, we provide theoretical foundations for this empirical observation by studying how the diffusion Lipschitz budget controls both expressivity and finite-sample statistical error in diffusion policy learning. To the best of our knowledge, this is among the first theoretical results of this kind in the literature.

\paragraph{Diffusion planning, trajectory diffusion, and diffusion imitation learning.}
Diffusion models have also been used as trajectory-level planners and
sequence/action generators for offline decision making and robot imitation
learning
\cite{janner2022planning, ajay2023conditional, li2023hierarchical,
chi2025diffusion, ke20243d, scheikl2024movement,
huang2024diffuseloco, balim2025model}. These works provide important
empirical evidence that diffusion parameterizations can capture multimodal,
high-dimensional, and temporally correlated behavior distributions. Our setting
differs in that we study Markov diffusion policies through their one-step
state-conditioned action laws and prove explicit expressivity and statistical
trade-offs.

\paragraph{Theory of diffusion and score-based generative models.}
Our analysis is also related to the theory of diffusion and score-based
generative models
\cite{sohl2015deep, song2019generative, ho2020denoising,
song2021scorebased, tzen2019deep, deBortoli2022convergence,
lee2022convergence, chen2023sampling, chen2023score}. Existing theory in this
area primarily studies sampling, score estimation, and distribution recovery.
In contrast, we connect the diffusion drift budget to RL performance by showing
that $K$-Lipschitz diffusion policies achieve, and in general cannot improve
upon, a $1/K$ value-approximation rate near deterministic optima.

\paragraph{Neural-network approximation, statistical learning, and PDE/SDE tools.}
Finally, the finite-sample part of our work builds on neural-network
approximation and statistical learning theory for ReLU networks
\cite{yarotsky2017error, petersen2018optimal, schmidt2020nonparametric,
bartlett2019nearly, ou2024covering}, while the continuous-time analysis uses tools from stochastic calculus, stochastic control, and parabolic PDE
theory, including Girsanov transformations, Fokker--Planck smoothing, and
density/lower-estimate arguments for uniformly elliptic diffusions
\cite{stroockvaradhan1979multidimensional, karatzas1991brownian,
oksendal2003stochastic,
Krylov2021Ldplus1, 
bogachev2015fokker,
aronson1967bounds, kazamaki2006continuous, carmona2016lectures,
hajek1985mean}.

\paragraph{Flow-based policy parameterizations.}
A closely related recent direction replaces stochastic diffusion sampling by ODE or flow-based generative policies, motivated in part by faster sampling and more direct pathwise or likelihood-ratio training. This includes online flow policies, \(Q\)-guided flow optimization, offline-to-online flow matching, reversible likelihood-ratio methods, MeanFlow and rectified-flow policies, and reparameterization-gradient training \cite{lyu2025flowrl,doo2026qflow,shin2026fino,hu2026genpoplusplus,dong2026mfpo,kim2026som,zhou2026trfp,zhong2026rfo}. These methods are complementary to ours: pathwise training typically relies on a differentiable critic or action-value gradient, while likelihood-ratio training for generic flow policies requires inversion and Jacobian control, or special reversible constructions. In contrast, our analysis focuses on SDE terminal-law diffusion policies and proves a \(K\)-dependent approximation--estimation trade-off.

\subsection{Roadmap}
The remainder of the paper is organized as follows. Section~\ref{sec:prelim} introduces the RL setting, the diffusion-policy parameterization, and the assumptions used throughout. Section~\ref{sec:expressivity} develops the expressivity theory for diffusion policies with Lipschitz drifts, including the resulting upper and lower bounds for value-function approximation. Section~\ref{sec:statistical-estimation} studies finite-sample learning when the drift is represented by ReLU neural networks, decomposing the error into diffusion approximation, neural-network realization, and statistical estimation terms. Section~\ref{sec:experiments} presents the policy-gradient training framework and numerical experiments illustrating both finite-sample regimes: the generic \(n^{-2/(m+6)}\) trade-off and the sharper dissipative \(n^{-2/(m+4)}\) trade-off. Section~\ref{sec:conclusion} concludes with the main takeaways and limitations of the work. %Technical highlights of the proofs are given in Appendix \ref{sec:technical_highlights}.

Throughout, \(\|\cdot\|\) denotes the Euclidean norm on finite-dimensional
vector spaces. For a map \(h:E\to\mathbb{R}^q\), we define
\(\|h\|_\infty:=\sup_{x\in E}\|h(x)\|\) and
\(\operatorname{Lip}(h):=\sup_{x\neq x'}
\|h(x)-h(x')\|/\|x-x'\|\).
We use \(\|\cdot\|_{\mathrm{op}}\) and \(\|\cdot\|_{\mathrm{HS}}\) for the
operator and Hilbert--Schmidt norms, respectively, and
\(\|\cdot\|_{L^p}\) for function-space \(L^p\) norms. The absolute constant $C>0$ is generic and may vary line by line.
\section{Preliminaries on MDP and diffusion policy}
\label{sec:prelim}

\paragraph{Markov decision process.}
Consider an infinite horizon Markov decision process (MDP) $(\mathcal X,\mathcal A, P,r, \gamma)$,
where 
 the state space
$\mathcal X$ is a measurable space, 
the action space is 
$\mathcal A=\sR^d$,
$P \in   \mathcal P(\mathcal X|\mathcal X\times \mathcal A)$ is a probability transition kernel, 
$r:\mathcal X\times \mathcal A\to \mathbb R$
is a bounded and continuous
reward function, and 
$\gamma\in [0,1)$ is the discount factor.
Let $\mathcal P(\mathcal A|\mathcal X)$
be the space of Markov stochastic policies. {That is, each policy \(\pi\in\mathcal P(\mathcal A|\mathcal X)\) assigns a probability distribution
\(\pi(\cdot|x)\) over actions to every state \(x\in\mathcal X\).}
Given a policy $\pi \in \mathcal P(\mathcal A|\mathcal X)$, the initial state $x_0$
is sampled from  $\rho \in \mathcal P(\mathcal X)$.
At each time $h\ge 0$,
the agent observes $x_h$, selects an action $a_h\sim \pi(\cdot|x_h)$, and receives a reward $r(x_h, a_h)$. The next   state is then sampled as 
$x_{h+1}\sim P(\cdot|x_h, a_h)$. 
We consider the following maximization problem: 
\begin{equation}\label{eq:value-function-mdp}
\sup_{\pi\in \mathcal{P}(\mathcal A|\mathcal X)}V^\pi(\rho)\,,
\quad 
\textnormal{with}
\quad 
V^\pi(\rho)
:=
\mathbb E_\rho^\pi\!\left[
\sum_{h=0}^\infty \gamma^h r(x_h,a_h)
\right],
\end{equation}
where 
$\mathbb E_\rho^\pi$
denotes the expectation with respect to the law of the state-action process induced by $\pi$. 
For each $x\in \mathcal X$, 
we write $V^\pi(x)=V^\pi(\delta_x)$, where  $\delta_x$  denotes the Dirac measure at $x$.

%Working with stochastic policies enables the use of various policy-gradient algorithms to solve \eqref{eq:value-function-mdp},  even though the optimal policy may, in principle, be deterministic  (e.g.,\cite{sutton1999policy,schulman2015trust, schulman2017proximal, haarnoja2018soft}). These methods introduce a differentiable policy parameterization to approximate the optimal policy and then search for optimal parameters using gradient-based optimization.

{A deterministic Markov policy is a special case of a stochastic policy, with
\(\pi(\cdot|x)=\delta_{g(x)}\) for some measurable map \(g:\mathcal X\to\mathcal A\), often referred to as an action selector.
Under mild continuity conditions \cite[Theorem 4.2.3]{hernandez2012discrete},
optimal policies can be chosen to be deterministic. In the sequel, we assume the optimal value function $V^{\star}:\mathcal X\to \mathbb R$ for the MDP \eqref{eq:value-function-mdp} is well-defined and  measurable. Define  
the corresponding  optimal  $Q$-function by $Q^{\star}(x,a):=r(x,a)+\gamma\int_\cX V^{\star}(x')P(\d x'|x,a)$.
We assume  that there exists a measurable function $g^\star:\mathcal X\to\mathcal A$ such that 
\begin{equation}\label{eq:g_star}
g^\star(x)\in A^\star(x),
\quad 
\textnormal{with}
\quad 
A^{\star}(x):=\argmax_{a\in \mathbb{R}^{d}}Q^{\star}(x, a), \quad \forall x\in \mathcal X.
\end{equation}
The function $g^\star$ is   an optimal deterministic policy, where $g^\star (x)$  is an optimal   action at 
state $x$.  

\paragraph{Diffusion policy parametrization.}
%An efficient implementation of stochastic policy-gradient algorithms requires (i) a policy parameterization that is sufficiently expressive to approximate the optimal policy, (ii) the ability to efficiently sample actions from the stochastic policy to interact with the environment, and (iii) the ability to efficiently evaluate the policy’s log-density in order to compute the policy-gradient estimator \cite{sutton1999policy}. 

Although an optimal deterministic policy exists, 
it is useful to employ 
stochastic policies in RL because they regularize the objective and enable gradient-based policy optimization 
\cite{haarnoja2018soft,schulman2015trust,schulman2017proximal,sutton1999policy}.}
 Motivated by the need for expressive stochastic policies that remain compatible with sampling and gradient-based optimization, we consider diffusion-model-based policies,
 which define $\pi(\cdot|x)$ as the terminal law of a diffusion process. 
Specifically, 
let  $T>0$ be a fixed time horizon,  let   $(\Omega,\mathcal F,\mathbb P_B)$ be a probability space   supporting a
$d$-dimensional Brownian motion, 
and let $\mathbb F$ be
the natural filtration of $B$ and the initial action $a_0$.
For  
each $x\in\mathcal X$,
consider  the following    dynamics in $\mathcal A=\mathbb R^d$:
\begin{equation}\label{eq:sde-policy}
d\bar a_t^{x}
=
f(t,\bar a_t^{x},x)\,dt
+
\sigma(t,\bar a^x_t,x)\,dB_t,
\qquad
\bar a_0^{x}\sim \mu_0.
\end{equation}
%\vv{I think it would be important, if possible, that we allow general initializations. I try to do this in Section 3. I think literature in diffusion models really cares about initialization and we do not want to give the impression that we can only do Gaussian initialization.} \renyuan{agree and modified}
where  
$ 
f:[0,T]\times \mathcal A\times \mathcal X\to \mathbb R^d$ and 
$\sigma:[0,T]\times \mathcal A\times \mathcal X\to \mathbb R^{d\times d}$  
are some measurable functions such that
\eqref{eq:sde-policy}
admits a unique weak solution 
 $(\bar a^x_t)_{t\ge 0}$. The initial distribution $\mu_0$
 is assumed to {be absolutely continuous w.r.t.  the Lebesgue measure} and have finite second moment, denoted by $m_2(\mu_0):=\int \|a\|^2\mu_0(da)<\infty$. In practice, a Gaussian initial distribution $\mu_0$ is a common choice. Throughout we assume that
\begin{asmp}[Volatility assumption]\label{ass:vol_bdd}
 There exist absolute constants $0<\kappa\leq \Lambda<\infty$ such that 
 
 \begin{align}\label{sigmabound}\kappa I_{d}\preceq \sigma(t, a,x)\sigma(t, a,x)^{\top}\preceq \Lambda I_{d}
 \end{align}
 
 for all $t\in[0,T]$, $a\in \mathcal A$ and $x\in \mathcal{X}$.
\end{asmp}

Define the corresponding policy  by  
\begin{equation}\label{eq:diffusion-policy-law}
\pi(\cdot\mid x)
:=
\Law_{\mathbb P_B}\bigl(\bar a_T^{x}\bigr).
\end{equation}
The collection of all such policies,
obtained by varying 
$f$, defines the diffusion policy class, which arises as the natural continuous-time limit of the diffusion-based policies used in
\cite{balim2025model,ding2024diffusion,
  ma2025efficient,psenka2024learning, 
ren2025diffusion,wang2024diffusion}.  In the sequel, we measure the complexity of a diffusion policy through the Lipschitz size of its drift in the action variable. Unless otherwise stated, \(K\) denotes a uniform Lipschitz bound in \(a\), namely
\[
\|f(t,a,x)-f(t,a',x)\|\le K\|a-a'\|,
\qquad t\in[0,T],\ x\in\mathcal X .
\]
We refer to \(K\) as the drift Lipschitz budget.
%For simplicity, and in line with this literature, we restrict attention to state-independent volatility coefficients.

\section{Expressivity of diffusion policy}\label{sec:expressivity}
We now turn to the approximation power of diffusion policies. 
The goal of this section is to show how the drift Lipschitz budget \(K\) controls their ability to concentrate near deterministic optimal actions, and how this translates into value-function approximation. 

\subsection{Upper bound on value function approximation with $K$-Lipschitz drifts}
In this section, we show that diffusion policies can approximate deterministic policies arbitrarily well as the Lipschitz constant of the drift increases. This is relevant because a larger Lipschitz constant allows the drift to represent more flexible neural networks. Under Assumption~\ref{ass:vol_bdd} and a fixed initialization, we construct a $K$-Lipschitz drift whose expressivity gap vanishes as $K\rightarrow\infty$. The key step is the following lemma. Although the volatility $\sigma$ depends on the state 
$\bar a_t^x$, so that the terminal law is generally not Gaussian, we can still obtain the following concentration estimate.
\begin{lem}[Concentration with mean-reverting drift]\label{lem:mean-reversion-concentration}
\label{lem:mean_rev} Assume Assumption \ref{ass:vol_bdd} holds.
Fix \(\tau>0\), \(K>0\), \(c\in\mathbb R^d\), and \(x\in\mathcal X\). Consider a linear drift $ f(s, a,x) = K(c/(1-e^{-K\tau})- a)$ in \eqref{eq:sde-policy}, and the following diffusion process
\begin{align}\label{eq:simpleOU}
    d\bar a^x_s
    =
    K\left(\frac{c}{1-e^{-K\tau}}-\bar a_s^x\right)ds
    +
    \sigma(s,\bar a^x_s,x)\,dB_s, \quad s\in[0,\tau],
\end{align}
where \(\bar a^x_0\sim \mu_{0}\) is independent of \(B\) .
%Assume that
%\[
%    \sup_{0\le s\le t}
%    \sup_{a\in\mathbb R^d,\;x\in\mathcal X}
%    \|\sigma(s,a,x)\|_{\mathrm{op}}
%    \le \Lambda
%\] for some \(\Lambda<\infty\).  \renyuan{we can show this result hold under Assumption 2.1}
Then
\[
    \mathbb E\|\bar a^x_\tau-c\|^2
    \le e^{-2K\tau}m_{2}(\mu_{0})
    +
    \dfrac{d\Lambda}{2K}
    \bigl(1-e^{-2K\tau}\bigr),
\]
where $m_2(\mu_0)$ is the second moment of $\mu_0$. 
\end{lem}
Consequently, for fixed \(\tau>0\), \(c\), and \(\Lambda\) independent of \(K\), we have 
$$\mathbb E\|\bar a_\tau-c\|^2
    \lesssim 1/K,$$
 with the omitted constant only depending on $\tau, d, m_{2}(\mu_{0}), \Lambda$. The proof is given in Appendix \ref{sec:proof_mean_rev}. This lemma shows that $\bar{a}^x_{\tau}$ can approximate $c$ at any time $\tau$, as long as $K$ is sufficiently large. By designing the drift to be $f(s,a,x) = K(c/(1-e^{-K\tau})-a)$,  we can approximate the optimal action $g^\star(x)$  arbitrarily well. An important property of $f$ is that it is $K$-Lipschitz in $a$.

\iffalse
To rigorously show this phenomenon, we have to show that $f(a; x)$ satisfies the assumptions. To this end, we have
\begin{align*}
\|f(a; x)\|\leq& K(\|g^\star(x)\|+\|a\|)\\
\|f(a; x)-f(a'; x)\|\leq& K\|a-a'\|
\end{align*}
Hence, to satisfy the assumptions, we need $\sup_{x}\|g^\star(x)\|\leq L<\infty$, or that optimal actions are bounded in norm. We provide a simple assumption that realizes this result.
\begin{asmp}\label{ass:bdd_opt_act}
Suppose there exists a compact convex set $S\subseteq \mathcal{A}$ such that
\begin{align*}
r(x, \Pi_{S}(a))\geq& r(x, a)\\
p(\cdot|x, a)=&p(\cdot|x, \Pi_{S}(a))
\end{align*}
for all $t, x, a$. 
\end{asmp}
This means that the environment only "sees" the projected action, and using large-norm actions is not beneficial after a certain threshold. Assumption \ref{ass:bdd_opt_act} ascertains that the optimal action $g^\star(x)\in S$ for every $x$.
\fi
We next translate this localization estimate into value-function approximation.  This requires the reward and transition kernel to be continuous near the optimal action \(g^\star(x)\), as formalized below.
\begin{asmp}[Moduli of continuity]\label{ass:cont_env_soft}
%\item Reward function $r(\cdot,\cdot)$ is continuous, bounded in both arguments; \renyuan{I moved this to section 2.}
 For every state $x$, there exists $R_{x}>0$ and moduli of continuity $\omega_{r, x}(z), \omega_{P, x}(z)\downarrow 0$ as $z\downarrow 0$ such that for all actions $a$ satisfying $\|a-g^\star(x)\|\leq R_{x}$,
\begin{align*}
|r(x, a)-r(x, g^\star(x))|\,\leq&\, \omega_{r,x}(\|a-g^\star(x)\|)\\
\TV\left(P(\cdot|x, a), P(\cdot|x, g^\star(x))\right)\,\leq&\, \omega_{P,x}(\|a-g^\star(x)\|)
\end{align*}
\end{asmp}
{This is a mild local condition: it requires continuity only near \(g^\star(x)\), rather than global %Lipschitz or
smoothness assumptions on the MDP.} With this assumption, we can achieve value function convergence. 
\begin{thm}[Value approximation by mean-reverting diffusion policies] 
\label{thm:converge}
Assume Assumptions \ref{ass:vol_bdd} and \ref{ass:cont_env_soft} hold.
\label{thm:ub_soft}
For each \(K\ge1\), define
\begin{eqnarray}\label{eq:cK}
    c_K(x):=\frac{g^\star(x)}{1-\exp(-KT)}.
\end{eqnarray}
Suppose that, for each \(K\ge1\) and \(x\in\mathcal X\), the mean-reverting diffusion
\begin{eqnarray}
\label{eq:mean-reversion}
    d\bar a_t^{K,x}
    =
    K\bigl(c_K(x)-\bar a_t^{K,x}\bigr)\,dt
    +
    \sigma(t,\bar a_t^{K,x},x)\,dB_t,
    \qquad
    \bar a_0^{K,x}\sim\mu_{0},
\end{eqnarray}
admits a weak solution on \([0,T]\), where \(\bar a_0^{K,x}\) is independent of \(B\) and \(\mu_0\) has finite second moment. 
Assume moreover that the terminal laws define a measurable stochastic policy
$\pi_K(\cdot\mid x):=\Law(\bar a_T^{K,x}).$  Then, for every
initial state law \(\rho\),
\[
    \lim_{K\to\infty}
    \left|
        V^{\pi_K}(\rho)-V^{\star}(\rho)
    \right|
    =
    0.
\]
\end{thm}
The proof is in Appendix \ref{sec:proof_thm_ub_soft}. We remark that the same argument can yield uniform convergence of the value function when the radius and moduli of continuity are uniform over $x$, by replacing the pointwise convergence step in the proof with uniform convergence over states $x$. %(c.f. Theorem \ref{thm:ub_hard} in Appendix \ref{app:expressivity}).

\subsection{Lower bound on value function approximation for $K$-Lipschitz drifts}\label{subsec:lb_vfa}
%We consider the large-$K$ setting, and examine what the \textit{best} value function approximation possible for $K$-Lipschitz drifts can be. We state our main result - a lower bound on the expected squared distance of a diffusion $\bar{a}_{t}$ to any finite set $\mathcal{Z}$. This finite set is intended to be the optimal action set in RL, as in the statement of Theorem \ref{thm:lb_diffusion_initial}.
The upper-bound result shows that a \(K\)-Lipschitz mean-reverting drift can concentrate a diffusion policy near an optimal deterministic action with squared error of order \(1/K\). 
We now show that this rate is intrinsic: under nondegenerate diffusion noise, no \(K\)-Lipschitz drift can localize the terminal action around a finite set of target actions at a faster rate.
\begin{prop}[Matching lower bound for finite-set localization]
\label{prop:lb_diffusion}
Let \(\mathcal Z=\{z_1,\ldots,z_J\}\subset\mathbb R^d\) be a finite set and fix \(R>0\).  Let
$(\bar a^x_t)_{t\ge0}$ follow \eqref{eq:sde-policy}, with $\bar a^x_0$ independent of $B$. Suppose that
$f(t,\cdot,x)$ is $K$-Lipschitz uniformly in $t,x$ and $f(\cdot, 0,x)\in L^{1}_{\text{loc}}$. Additionally, assume either Assumption \ref{ass:vol_bdd} or
\begin{enumerate}
\item[(A2)]\makeatletter \def\@currentlabel{A2} \makeatother \label{ass:state-indep} (State-independent volatility) $\sigma(t, a,x)\equiv \sigma(t)$ is
  deterministic and locally square-integrable, and
$\sigma(t)\sigma(t)^\top\succeq \kappa I_d$ for almost all $t\ge0$, where $\kappa>0$ is independent of $K$.
\end{enumerate}

Then there exist $c_R>0$ and $K_R<\infty$, depending on
$J,d,\kappa,\mu_0,R$, and additionally on \(\Lambda\) in the state-dependent
case, such that, for every \(K\ge K_R\),
$$
\inf_{f,\sigma,x,T\ge0}
\mathbb E\left[\dist(\bar a^x_T,\mathcal{Z})^2\wedge R^2\right]
\ge
\frac{c_R}{K}.
$$
Here $\dist(a,\mathcal Z):=\inf_{z\in\mathcal Z}\|a-z\|$. The constant $c_R$ does not depend on the locations of the points in $\mathcal{Z}$, nor
on the minimum distance between distinct points of $\mathcal{Z}$.
\end{prop}
The proofs are given in Appendix \ref{sec:proof_prop_34_first} for Assumption \ref{ass:state-indep}, and Appendix \ref{sec:proof_prop_34_second} for Assumption \ref{ass:vol_bdd}. The key observation, which appears to be new, is to split at the intrinsic time \(K^{-1}\). On time scales shorter than \(K^{-1}\), a \(K\)-Lipschitz drift cannot move an absolutely continuous initial law into arbitrarily small neighborhoods of finitely many target points. 
On longer horizons, one can look only at the final window of length \(K^{-1}\). 
After rescaling this window by
\[
    \bar a_{T-1/K+\ell/K}
    =
    \phi_\ell + K^{-1/2}Y_\ell,
    \qquad 0\le \ell\le 1,
\]
the residual process \(Y_\ell\) has \(\mathcal{O}(1)\)-Lipschitz drift and uniformly nondegenerate volatility.
\iffalse
Regarding \(T\le K^{-1}\), over such a short window, the diffusion cannot transform an absolutely continuous initial law
into arbitrarily small neighborhoods of finitely many target points.
Consequently, a small neighborhood of \(\mathcal Z\) can capture only a
fixed fraction of the terminal mass, giving actually a constant lower bound on the expected squared distance. If \(T>K^{-1}\), set \(s=T-K^{-1}\) and condition on \(\mathcal F_s\).
On the final window, write
$
\bar a_{s+\ell/K}=\phi_\ell+K^{-1/2}Y_\ell$ ($0\le \ell\le 1$) where
\[
\phi_0=\bar a_s,\qquad
\phi_\ell
=
\bar a_s+\frac1K\int_0^\ell b(s+r/K,\phi_r,x)\,dr .
\]
Conditional on \(\mathcal F_s\), the path \(\phi_\ell\) is nonrandom, while \(Y_\ell\) records the residual
fluctuation \(K^{-1/2}\). In these coordinates,
\(Y_\ell\) satisfies
\[
dY_\ell=\beta_K(\ell,Y_\ell)\,d\ell
+\widetilde\sigma_K(\ell,Y_\ell)\,dW_\ell,
\]
where the residual drift \(\beta_K\) is \(\mathcal{O}(1)\)-Lipschitz and the volatility $\tilde{\sigma}_{K}$
remains uniformly nondegenerate and bounded.  
\fi
As a result, \(Y_1\) has a nondegenerate density uniformly in $K$, hence having distance at least $c$ from any finite set with constant probability. Scaling back, \(\bar a_T\) is at distance at
least \(cK^{-1/2}\) from \(\mathcal Z\) with constant probability, yielding
the \(c/K\) lower bound after squaring.

\begin{rmk}
We give a few remarks about Proposition \ref{prop:lb_diffusion}.

\begin{itemize}
\item[(i)]  \textbf{Role of bounded volatility.} 
For state-independent volatility, uniform ellipticity is enough. 
For state-dependent volatility, the upper bound in Assumption~\ref{ass:vol_bdd} is essential: Proposition~\ref{prop:unbounded_sigma_counterexample} shows that, if this upper bound is removed, the \(1/K\) lower bound can fail even when the volatility remains uniformly elliptic from below.

\item[(ii)] \textbf{Sharpness.} 
The lower bound matches the \(\mathcal{O}(1/K)\) upper bound from the mean-reverting construction in Lemma~\ref{lem:mean_rev}. 
Indeed, with finite second moment initialization and uniformly bounded volatility, a linear mean-reverting drift can achieve
$\mathbb E\|\bar a_T-z\|^2 \lesssim K^{-1}$
for any fixed \(z\in\mathbb R^d\), and hence also
$ \mathbb E[\dist(\bar a_T,\mathcal Z)^2]\lesssim K^{-1}$
for any finite set \(\mathcal Z\). 
Proposition~\ref{prop:lb_diffusion} shows that this rate cannot be improved in general.

\item[(iii)] \textbf{Relevance for RL.}
In RL, \(\mathcal Z\) should be viewed as the Bellman-optimal action set \(A^\star(x)\). 
Although the upper bound only needs one selector \(g^\star(x)\in A^\star(x)\), the lower bound covers the case where \(A^\star(x)\) contains several isolated optimal actions. 
Thus, near deterministic optima, the drift Lipschitz budget \(K\) is the expressivity bottleneck for diffusion policies.
\end{itemize}
\end{rmk}

Proposition \ref{prop:lb_diffusion} immediately yields a value-function lower bound. The only additional assumption is that the Bellman gap grows at least quadratically locally near the finite set of optimal actions.

\begin{thm}[Finite-well diffusion lower bound]\label{thm:lb_diffusion_initial}
%Consider the infinite-horizon discounted RL setting with $\gamma\in(0,1)$ and initial state law $\rho$. 
For each $K\geq 1$ and $T>0$, let
$\pi_{K,T}$ be the Markov policy whose action law at state $x$ is the
terminal law at time $T$ of \eqref{eq:sde-policy}. Assume that the conditions in
Proposition~\ref{prop:lb_diffusion} hold. %, with ellipticity constants$(\kappa,\Lambda)$ independent of $(t,x)$.

Fix $M\geq 1$ and $R>0$. Suppose that, for every $x\in\mathcal X$, %$A^{\star}(x)$is nonempty and satisfies 
$|A^\star(x)|\leq M$, and that 
there exists a constant $\lambda_{M,R}>0$, independent of $x$, such that for all $(x, a)\in \mathcal X \times \mathcal A$, 
\begin{align}\label{cond:lower_quad}
V^\star(x)-Q^\star(x,a)
\geq
\lambda_{M,R}
\left(
\dist(a,A^\star(x))^2\wedge R^2
\right).
\end{align}
Then there exist constants $c_{M,R}>0$ and $K_0<\infty$, depending only on
$M,R,d,\kappa,\Lambda,\mu_0$, such that  
\[
\inf_{T>0}
\left(
V^\star(\rho)-V^{\pi_{K,T}}(\rho)
\right)
\geq
\frac{c_{M,R}\lambda_{M,R}}{K},
\quad \forall K\geq K_0.
\]
\end{thm}
The proof is given in Appendix \ref{sec:proof_thm_34}, which applies Proposition \ref{prop:lb_diffusion} with $\mathcal{Z}=A^{\star}(x)$.

\begin{rmk}[Bellman gap assumption in \eqref{cond:lower_quad}]
Condition~\eqref{cond:lower_quad} is a mild finite-well condition on the Bellman gap.  The clipping makes the assumption local in nature: near \(A^\star(x)\), it requires the Bellman gap to grow at least quadratically with the distance to the optimal action set, while away from \(A^\star(x)\), it only requires a positive suboptimality margin. It does not require concavity, uniqueness of the
optimizer, or growth at infinity. The quadratic exponent is chosen to
match the squared-distance lower bound in Proposition~\ref{prop:lb_diffusion}.
More generally, a clipped \(p\)-th order gap would yield the corresponding \(K^{-p/2}\) rate; for example, quadratic wells give \(K^{-1}\), while quartic wells give \(K^{-2}\).
\end{rmk}

\begin{rmk}[A note on policy initialization]
 We note that the lower bound over \textit{all} $t>0$ in Theorem \ref{thm:lb_diffusion_initial} cannot hold under arbitrary action initialization $\mu_0$. A counterexample is given in Appendix \ref{sec:statement_proof_cor_35}. However, if we allow for burn-in time $t_{0}\gtrsim 1/K$, the conclusion of Theorem \ref{thm:lb_diffusion_initial} would still hold regardless of initialization. %This assumption is reasonable because in diffusion literature we often consider large time horizons. 
We state this extension in Corollary ~\ref{cor:arb_init} at Appendix~\ref{sec:statement_proof_cor_35}. 
%{\color{brown} (Jiacheng: We can specify this extension as Corollary \ref{cor:arb_init})}
\end{rmk}

\section{Approximation--estimation trade-off in \(K\)}
\label{sec:statistical-estimation}

Section~\ref{sec:expressivity} shows that a \(K\)-Lipschitz diffusion policy can approximate a deterministic optimal selector with value error of order \(K^{-1}\). We now develop the finite-sample theory of learning such policies from data when
the drift is represented by $K$-Lipschitz ReLU networks.
% The network class is chosen so that
% each admissible drift has Lipschitz constant at most \(K\) in the action
% variable. 

For the sake of exposition,  we take the state space $\mathcal X=[0,1]^m$, and impose the following technical assumptions on the 
MDPs. 

\begin{asmp}[Bounded and Lipschitz optimal deterministic policy]\label{ass:lipschitsz_g}
 The optimal selector $g^\star$, defined in \eqref{eq:g_star}, satisfies $\|g^\star\|_\infty\le R_0$ and $\mathrm{Lip}(g^\star)\le L_{g^\star}$.
\end{asmp}

\begin{asmp}[MDP environment and diffusion policy volatility]
\label{ass:lipschitsz}
% Assume the environment is Lipschitz in state and action. 
 
There exist constants $L_r^x,L_r^a<\infty$ such that, for all $x,x'\in\mathcal X$ and $a,a'\in\mathbb R^d$,
\[
|r(x,a)-r(x',a')|
\le L_r^x\|x-x'\|+L_r^a\|a-a'\|.
\]
 The   state transition is given by  a Lipschitz generative model, i.e.,  
there exists a measurable map $\Phi_P:\mathcal X\times \cA\times[0,1]\mapsto\mathcal X$ such that  for $U\sim{\rm Unif}[0,1]$ and every $(x,a)\in \mathcal X \times \cA$,
\(
\Phi_P(x,a,U)\sim P(\cdot\mid x,a),
\)
and there exist constants $L_P^x,L_P^a<\infty$ such that
\begin{equation}
\label{eq:transition-sampler-l2-section4}
\left(
\mathbb E_U
\left[
\left\|
\Phi_P(x,a,U)-\Phi_P(x',a',U)
\right\|^2
\right]
\right)^{1/2}
\le
L_P^x\|x-x'\|
+
L_P^a\|a-a'\|.
\end{equation}
Moreover, the diffusion volatility $\sigma$ is Lipschitz in $(x,a)$: there exists $L_\sigma<\infty$ such that
\begin{align}\label{sigmalipschitz}
\|\sigma(t,a,x)-\sigma(t,a',x')\|_{\mathrm{HS}}
\le L_\sigma\big(\|x-x'\|+\|a-a'\|\big),
\end{align}
where $\|\cdot\|_{\rm HS}$ is the Hilbert--Schmidt norm.
\end{asmp}
Assumption~\ref{ass:lipschitsz} is stated before the empirical construction because the variables below use the fixed transition generator \(\Phi_P\). The map \(\Phi_P\) gives the synchronous $L^2$ coupling used in the stability proof.

\paragraph{ReLU drift classes.}
We first introduce some constants that will be used to construct the policy class. 
Fix a sufficiently large  constant \(C_0\ge2\), independent of \((n,K,s,m,d,T,\delta)\). Choose constants \(\ell_1,\ell_2<\infty\), independent of \((K,s)\), satisfying
\begin{equation}
\label{eq:ell-lower-section4}
 \ell_1\ge \frac{\sqrt d\,R_0}{1-e^{-T}},
 \qquad
 \ell_2\ge \frac{C_0^m\sqrt d\,L_{g^\star}}{1-e^{-T}},
\end{equation}
where  $R_0$ is given in Assumption 
\ref{ass:lipschitsz_g}.
For each \(s\ge1\), define
\begin{eqnarray}
 p_{s,1}&:= &C_0^m d(m+d+1)s\log(e s),\label{eq:p-s-1-section4}\\
 p_{s,2}&:=&C_0^m\log(e s),\label{eq:p-s-2-section4}\\
 p_{s,3}&:=&C_0^m d(m+d+1)s,\label{eq:p-s-3-section4}\\
 p_{s,4}&:=&C_0^m\{1+\sqrt d(R_0+L_{g^\star})\}(1+s)^{C_0^m}(1-e^{-T})^{-1}.
\label{eq:p-s-4-section4}
\end{eqnarray}

Let \(\mathcal N_s\) be the set of vector-valued multi-layer ReLU networks
\(\Psi_\theta:[0,T]\times\mathbb R^d\times[0,1]^m\to\mathbb R^d\) satisfying the following conditions: $\Psi_\theta$ has at most \(p_{s,1}\) nonzero scalar parameters, depth at most \(p_{s,2}\), width at most \(p_{s,3}\), and parameter bound \(|\theta_i|\le p_{s,4}\). For every affine region \(S\) intersecting \([0,T]\times\mathbb R^d\times[0,1]^m\),
\begin{equation}
\label{eq:affine-piece-section4}
\Psi_\theta(t,a,x)=M^1_{\theta,S}t+M^2_{\theta,S}a+M^3_{\theta,S}x+b_{\theta,S},
\end{equation}
for some  \(M^2_{\theta,S}\in\mathbb R^{d\times d}\) and \(M^3_{\theta,S}\in\mathbb R^{d\times m}\) satisfying  
\[
\|M^2_{\theta,S}\|_{\mathrm{op}}\le1,
\qquad
\|M^3_{\theta,S}\|_{\mathrm{op}}\le\ell_2,
\]
and
\[
\sup_{t\in[0,T],\,x\in[0,1]^m}\|\Psi_\theta(t,0,x)\|\le \ell_1.
\]
Define the drift class
\begin{equation}
\label{eq:relu-drift-class-section4}
 \mathcal F_{K,s}:=\Big\{f_\theta(t,a,x)=K\Psi_\theta(t,a,x):\Psi_\theta\in\mathcal N_s\Big\}.
\end{equation}

The architecture bounds, especially $C^m_0$, in the definition of \(\mathcal F_{K,s}\) are chosen so that the class contains the mean-reverting drift of the form \eqref{eq:mean-reversion}.  %Section~\ref{sec:expressivity}. 
Specifically, Lemma~\ref{lem:relu-selector-approx}   shows that
for all $K\ge 1$,
\(\mathcal F_{K,s}\) contains 
\begin{eqnarray}\label{eq:comp-section4}
    f_{K,s}(t,a,x)
=
K\left(\frac{h_s(x)}{1-e^{-KT}}-a\right),
\end{eqnarray}
where  $h_s$ is  a suitable ReLU network of size $s$
 approximating the optimal selector \(g^\star\). Thus the finite-sample analysis can use the approximation property from Section~\ref{sec:expressivity} for mean-reverting drifts. 
 
 % The tuning parameter controlling the network size is \(s\).

\begin{lem}[ReLU approximation and mean-reverting drift]
\label{lem:relu-selector-approx}
Under Assumption~\ref{ass:lipschitsz_g},  there exists a constant $C_0\ge 0$ such that for every \(s\ge1\) there is a vector-valued ReLU network \(h_s:[0,1]^m\to\mathbb R^d\) satisfying 
\begin{align}
 \sup_{x\in[0,1]^m}\|h_s(x)-g^\star(x)\|^2&\le C_0^m m d L_{g^\star}^2s^{-2/m},\label{eq:h-error-section4}\\
 \sup_{x\in[0,1]^m}\|h_s(x)\| &\le \sqrt{d}\,R_0,\label{eq:h-bound-section4}\\
 \mathrm{Lip}(h_s)&\le C_0^m\sqrt d\,L_{g^\star}.
\label{eq:h-lip-section4}
\end{align}
Moreover, 
% \(h_s\) is exactly representable by a ReLU network with 
$h_s$ has
at most (i) \(C_0^m d s\log(e s)\) nonzero scalar parameters, (ii) depth at most \(C_0^m\log(e s)\), (iii) width at most \(C_0^mds\), and (iv)weights bounded by $ C_0^m(1+R_0+L_{g^\star})(1+s)^{C_0^m}$.

It holds that $ f_{K,s}\in\mathcal F_{K,s} $, 
where $f_{K,s}$ is defined in \eqref{eq:comp-section4}.
\end{lem}

%{\color{blue} Yufei: Is the definition   $\|h_s\|_\infty=\sup_{x\in \cX}\|h_s(x)\| $? Is it $\sup_{x}\|h_s(x)\|?$ Just to make  sure there is no $\sqrt d$ missing when applying \eqref{eq:h-bound-section4} to  get $\ell_1$ in the norm $\|\cdot\|$.}

The proof is given in Appendix \ref{app:relu}.

\paragraph{Induced diffusion policies and statistical samples.}
For a drift \(f\in\mathcal F_{K,s}\), let \(\pi_f(\cdot\mid x)\) denote the terminal law at time \(T\) of the following SDE: for each fixed \(x\), write \(\bar a^x\) for the solution of
\[
d\bar a_t^x=f(t,\bar a_t^x,x)\,dt+
\sigma(t,\bar a_t^x,x)\,dB_t,
\qquad \bar a_0^x\sim\mu_0 .
\]
The dependence on \(f\) is only through the displayed drift. Write \(V(f):=V^{\pi_f}(\rho)\). For the empirical value function, fix i.i.d. random variables
\[
\xi_i:=\bigl(X_{i,0},\{\zeta_{i,j},B_{i,j},U_{i,j}\}_{j\ge0}\bigr),
\qquad i=1,\ldots,n,
\]
where \(X_{i,0}\sim\rho\), \(\zeta_{i,j}\sim\mu_0\), the \(B_{i,j}\)'s are independent Brownian motions on \([0,T]\), and \(U_{i,j}\sim{\rm Unif}[0,1]\), all mutually independent. Given \(f\), set \(X_{i,0}^f:=X_{i,0}\). Recursively, given \(X_{i,j}^f\), let \(\bar a_{i,j}\) solve
\[
d\bar a_{i,j,t}=f(t,\bar a_{i,j,t},X_{i,j}^f)\,dt+
\sigma(t,\bar a_{i,j,t},X_{i,j}^f)\,dB_{i,j,t},
\qquad \bar a_{i,j,0}=\zeta_{i,j},
\]
and define
\[
A_{i,j}^f:=\bar a_{i,j,T},
\qquad
X_{i,j+1}^f:=\Phi_P\bigl(X_{i,j}^f,A_{i,j}^f,U_{i,j}\bigr).
\]
Then, for each fixed \(f\), \(\{(X_{i,j}^f,A_{i,j}^f)\}_{j\ge0}\) follows the MDP trajectory law induced by \(\pi_f\). Finally, define the empirical value function:
\begin{equation}
\label{eq:empirical_value}
\widehat V_n(f)
:=
\frac1n\sum_{i=1}^n\sum_{j=0}^{\infty}
\gamma^j r(X_{i,j}^f,A_{i,j}^f).
\end{equation}
For each fixed \(f\), \(\widehat V_n(f)\) is an unbiased estimator of \(V(f)\). 
Using the same fixed exogenous variables \(\xi_1,\ldots,\xi_n\) for all 
\(f\in\mathcal F_{K,s}\) makes \(f\mapsto\widehat V_n(f)\) a well-defined 
 random function, which is used in the proof to control the uniform deviation 
over \(f\in\mathcal F_{K,s}\).

Let
\begin{equation}
\label{eq:empirical-general-drift-optimizer}
\widehat f_{n,K,s}
\in
\argmax_{f\in\mathcal F_{K,s}}\widehat V_n(f).
\end{equation}
If an exact maximizer does not exist, take any measurable \(n^{-1/2}\)-maximizer; this changes the final bound only by an absorbed \(\mathcal{O}(n^{-1/2})\) term.

Define \(r_{\max}:=\|r\|_\infty\),
set 
 \(H_n=1\)
 if \(\gamma=0\),  and      $H_n
:=
\max\left\{
1,
\left\lceil
\frac{\log(8\sqrt n)}{\log(1/\gamma)}
\right\rceil
\right\}$
if \(0<\gamma<1\).
Define the logarithmic factor
\begin{align}
\label{eq:Gamma-section4}
\Gamma_n(K,s,m,d,T)
:=
\log\Big(&1+2n^2
\Big\{
1+T+\sqrt m
+\frac{r_{\max}}{1-\gamma}
\sqrt{H_n}
\left(m_2(\mu_0)+K^2\ell_1^2T^2+d\Lambda T\right)^{1/2}
e^{C_3KT}
\Big\}
\notag\\
&\times
(p_{s,1}p_{s,4})^{p_{s,2}+2}
C_3H_n
\Big\{
C_3\big[
1+(1+\ell_2)\sqrt{KT}\,e^{C_3KT}
\big]
\Big\}^{H_n}
\Big),
\end{align}
where \(C_3\ge1\) is a sufficiently large   constant
depending only on \((L_\sigma,L_r^x,L_r^a,L_P^x,L_P^a,\gamma)\).
Then we have the following main result.
\begin{thm}[Finite-sample oracle inequality for Lipschitz drifts]
\label{thm:lipschitz-selector-complete-tradeoff}
Assume Assumptions~\ref{ass:vol_bdd}, \ref{ass:lipschitsz_g}, and \ref{ass:lipschitsz} hold. Suppose also that the Bellman gap has the quadratic upper bound
\begin{equation}
\label{eq:quadratic-upper-gap-section4}
0\le V^\star(x)-Q^\star(x,a)
\le L_Q\|a-g^\star(x)\|^2,
\qquad x\in[0,1]^m,
\ a\in\mathbb R^d .
\end{equation}
Then there exists a constant \(C_1<\infty\), independent of \((n,K,s,m,d,T,\delta)\), such that, for all \(K\ge1\), \(s\ge1\), and \(\delta\in(0,1)\), with probability at least \(1-\delta\),
\begin{align}
\label{eq:ks-tradeoff-main}
 V^\star(\rho)-V(\widehat f_{n,K,s})
 \le C_1\Bigg[&
 e^{-2KT}m_2(\mu_0)+\frac{d\Lambda}{K}+C_0^m m d L_{g^\star}^2s^{-2/m}
\notag\\
&+\frac{r_{\max}}{1-\gamma}\sqrt{\frac{p_{s,1}\Gamma_n(K,s,m,d,T)+\log(2/\delta)}{n}}
\notag\\
&+\frac{r_{\max}}{1-\gamma}\frac{p_{s,1}\Gamma_n(K,s,m,d,T)+\log(2/\delta)}{n}
\Bigg].
\end{align}
The constant \(C_1\) depends only on \((L_Q,L_\sigma,L_r^x,L_r^a,L_P^x,L_P^a,\Lambda,\gamma)\) and numerical constants.
\end{thm}

The proof combines localization, covering-number estimates, and a
finite-horizon uniform-convergence argument. The details are deferred to Appendix~\ref{app:proof-lipschitz-selector-complete-tradeoff}. Note that the dependence on \(d\) and \(m\) through \(\Gamma_n(K,s,m,d,T)\) does not affect the leading-order terms in \eqref{eq:ks-tradeoff-main}.

\begin{rmk}[Balancing the leading terms]
\label{rem:tradeoff-section4}
Assume \(m_2(\mu_0)\le C_2d\) and treat all problem parameters other than
\((m,d,K,s,n)\) as fixed. Since
\(p_{s,1}=\widetilde O\!\left(C_0^m d(m+d+1)s\right)\) and
\(\Gamma_n(K,s,m,d,T)=\widetilde O(KT)\), the leading terms in
\eqref{eq:ks-tradeoff-main}, after suppressing logarithmic and fixed
problem-dependent factors, are
\[
\frac{d}{K},
\qquad
C_0^mmd\,s^{-2/m},
\qquad
\sqrt{\frac{C_0^m d(m+d+1)sK}{n}}.
\]
Balancing the first two terms gives
\(s\asymp(C_0^m mK)^{m/2}\). Substituting this choice into the statistical
term and balancing with \(d/K\) yields $K\asymp
\left(
\frac{dn}
{C_0^m(m+d+1)(C_0^m m)^{m/2}}
\right)^{2/(m+6)}$ and 
\begin{eqnarray}
 \eqref{eq:ks-tradeoff-main} = 
\widetilde {\mathcal O}\!\left(
d^{(m+4)/(m+6)}
\left(
\frac{C_0^m(m+d+1)(C_0^m m)^{m/2}}{n}
\right)^{2/(m+6)}
\right).   \label{eq:leading_order}
\end{eqnarray}
In particular, the exponent of \(n\) is \(-2/(m+6)\).
\end{rmk}

\begin{rmk}[Dissipative drift classes and better rate]
\label{rem:dissipative-comparison-section4}
Imposing a one-sided dissipativity condition on the drift class
$\mathcal F_{K,s}$ can improve the error bound
\eqref{eq:leading_order}. Specifically, assume that
for all $f\in\mathcal F_{K,s}$,
\begin{equation}
\label{eq:optional-dissipativity-section4}
\big\langle a-a',
f(t,a,x)-f(t,a',x)
\big\rangle
\le
-\lambda K\|a-a'\|^2 ,
\end{equation}
for some $\lambda>0$ independent of $(K,s)$.
This condition is satisfied by the mean-reverting
drift $f_{K,s}$ defined in \eqref{eq:comp-section4}, for which
$\lambda=1$.

In this case, the approximation error terms in
\eqref{eq:ks-tradeoff-main} are unchanged, but the statistical
error term reduces to
$\sqrt{\frac{C_0^m d(m+d+1)s}{n}}$, compared with
$\sqrt{\frac{C_0^m d(m+d+1)sK}{n}}$ in the generic
$K$-Lipschitz case. Thus the leading term in
\eqref{eq:ks-tradeoff-main} becomes
\[
\frac{d}{K}
+
C_0^mmds^{-2/m}
+
\sqrt{\frac{C_0^md(m+d+1)s}{n}} .
\]
Since the last two terms do not depend on $K$, balancing them gives $s
\asymp
\left(
\frac{C_0^m m^2dn}{m+d+1}
\right)^{m/(m+4)}.$
For this choice, it suffices to take
\[
K
\gtrsim
\left(
\frac{dn}
{C_0^m(m+d+1)(C_0^m m)^{m/2}}
\right)^{2/(m+4)},
\]
so that $d/K$ is no larger than the other two terms. Thus, up to
logarithmic factors and fixed problem constants, the exponent of
$n$ improves from $-2/(m+6)$ to $-2/(m+4)$. Here, the displayed
order of $K$ is a minimum sufficient scale rather than an optimizer:
any larger $K$ achieves the same leading rate. This sharper rate
comes from the stronger condition
\eqref{eq:optional-dissipativity-section4}; without it,
Theorem~\ref{thm:lipschitz-selector-complete-tradeoff} applies to the
larger generic ReLU drift class, but retains the extra $K$-dependence
in the statistical term.
\end{rmk}

\section{Policy gradient}\label{sec:policy gradient}\label{sec:policy_gradient}
% Sections~\ref{sec:expressivity} and~\ref{sec:statistical-estimation} characterize
% how the drift Lipschitz budget \(K\) controls the expressivity and finite-sample
% behavior of diffusion policies.
% {We now turn to the complementary
% algorithmic question:
% \begin{itemize}
%     \item  once a diffusion-policy class is chosen, can it be trained
% by a tractable policy-gradient formula?
% \end{itemize}
% This is nontrivial because
% \(\pi_\theta(\cdot|x)\) is defined implicitly as the terminal law of an SDE.
% Unlike Gaussian or exponential-family policies, the terminal density of
% \(\bar a_T^x\), and hence the usual score
% \(\nabla_\theta\log \pi_\theta(a|x)\), is generally unavailable. The main point
% of this section is that the missing score can be replaced by a pathwise
% likelihood-ratio term obtained from Girsanov's theorem. In order to tackle the problem, we propose the following assumptions.}

Sections~\ref{sec:expressivity} and~\ref{sec:statistical-estimation} have characterized how the drift Lipschitz budget \(K\) governs the expressivity and finite-sample behavior of diffusion policies. We now address the complementary algorithmic question:
\begin{itemize}
    \item Given a selected diffusion-policy class, can it be trained using a tractable policy-gradient formula?
\end{itemize}
Let the drift in \eqref{eq:sde-policy} be parameterized by
\(f_\theta(t,\bar a_t^x,x)\), where \(\theta\in\sR^k\) may represent a generic
differentiable functional class, and denote the corresponding policy by
\(\pi_\theta\).  

This inquiry is nontrivial because \(\pi_\theta(\cdot|x)\) is defined implicitly as the terminal law of an SDE. Unlike Gaussian or exponential-family policies, the terminal density of \(\bar a_T^x\), and consequently the standard score 
, \(\nabla_\theta\log \pi_\theta(a|x)\), is generally unavailable. The primary contribution of this section is demonstrating that the missing score can be replaced by a pathwise likelihood-ratio term derived via Girsanov's theorem. Accordingly, the main result state in Proposition \ref{prop:policygradient} provides a policy-gradient formula expressed in terms of the drift sensitivity
\(\nabla_\theta f_\theta=(\partial_{\theta_1}f_\theta,\cdots,\partial_{\theta_k}f_\theta)^\top\). To establish this result, we introduce the following assumptions:

\medskip

%In order to identify $\nabla_\theta V^{\pi_\theta}(\rho)$. 

\begin{asmp}\label{assum_light:f_theta}
    Assume that $f$ is continuously differentiable in $\theta$ and 
    $$
    \sup_{s,\theta,x}\|f(s,0,\theta,x)\|+\|\nabla_\theta f(s,0,\theta,x)\|<\infty.
    $$
    Moreover, $f$ and $\nabla_{\theta} f$ are $\alpha$-H\"older continuous in $a$; that is, there exists a constant $C>0$  and $\alpha\in (0, 1]$ such that for all  $s\in[0,T],x\in\mathcal{X}$, $\theta\in\sR^k$ and $a,a'\in \sR^d$,  
    $$
    \|f(s,a,\theta,x)-f(s,a',\theta,x)\|+\|\nabla_\theta f(s,a,\theta,x)-\nabla_\theta f(s,a',\theta,x)\|\leq C\left(\|a-a'\|^{\alpha}\vee \|a-a'\|\right).
    $$
    Furthermore, we impose an additional condition on the initial distribution $\mu_0$:
    \begin{align}\label{initial}
    \mathbb{E}_{\mu_0}\left[\exp\left(ca^2\right)\right] < \infty,
    \qquad \text{where } a \sim \mu_0,
    \end{align}
for some sufficiently small constant $c>0$.
\end{asmp}

By Assumption \ref{assum_light:f_theta}, and holding $a'=0$ fixed, there exists constant $C>0$ such that
$$\|f(s, a, \theta, x)\| + \|\nabla_{\theta} f(s, a, \theta, x)\|\leq C(\|a\|+1).$$
% so that $f, \nabla_{\theta} f$ have at most linear growth in $a$.
{This growth bound is used to
verify the martingale and differentiability conditions required by the
Girsanov-based argument.}

\medskip

To achieve our primary objective of identifying the gradient of the value function $\nabla_\theta V^{\pi_\theta}(\rho)$, we begin with characterizing the derivative of 
$\sE^{\sP_{\theta}}\big[\varphi(\bar a_T)\big]$  with  a  bounded continuous  function $\varphi$.
\begin{lem}\label{lem:grad:fix}
Suppose Assumption \ref{assum_light:f_theta}, \eqref{sigmabound} and \eqref{sigmalipschitz} hold. Let $\varphi:\sR^d\to \sR$ be a bounded measurable  function, and define $U^\varphi_\theta(x)=\sE^{\sP_\theta}\big[\varphi(\bar a_T^x)\big]$. Then, 
$$
    \nabla_\theta U^\varphi_\theta(x)=\sE^{\sP_\theta}\bigg[\bigg(\int_0^T{\nabla_\theta f_\theta(t,\bar a_t,x)}\big(\sigma^\top(t,\bar a_t,x)\big)^{-1}\d B_t\bigg)\varphi(\bar a_T)\bigg].
    $$

% \begin{lem}
%     Suppose Assumption \ref{assum_light:f_theta} holds and let  be a continuous bounded function independent of $\theta$. Define
%     $U^\varphi(\theta)=\sE^{\sP_\theta}\big[\varphi(\bar a_T)\big]$, then, we have
%     $$
%     \nabla_\theta U^\varphi(\theta)=\sE^{\sP_\theta}\bigg[\bigg(\int_0^T{\nabla_\theta f_\theta(t,\bar a_t,x)}\sigma(t,\bar a_t,x)^{-1}\d B_t\bigg)\varphi(\bar a_T)\bigg],
%     $$
% \end{lem} 
 
\end{lem}

Hence, despite the unavailability of the terminal density of \(\pi_\theta(\cdot|x)\),
the likelihood-ratio derivative admits a representation via the martingale score
\[
S_\theta(x):=
\int_0^T
\nabla_\theta f_\theta(t,\bar a_t^x,x)
\big(\sigma^\top(t,\bar a_t^x,x)\big)^{-1}\,dB_t .
\]

\medskip

Second, applying the preceding identity with \(\phi(a)=Q^{\pi_\theta}(x,a)\) inside the discounted performance-difference identity yields the following main result.

\begin{prop}[Policy gradient theorem]\label{prop:policygradient}
    Suppose Assumption \ref{assum_light:f_theta},\eqref{sigmabound} and \eqref{sigmalipschitz} hold. Then the gradient of the value function with respect to the parameter $\theta$ has the following representation 
    \begin{align}
    \nabla_\theta V^{\pi_\theta}(\rho)=\dfrac{1}{1-\gamma}\int_\cX\sE\bigg[\bigg(\int_0^T{\nabla_\theta f_\theta(t,\bar a_t^x,x)}\big(\sigma^\top(t,\bar{a}^x_t,x)\big)^{-1}\d B_t\bigg)Q^{\pi_{\theta}} (x,\bar a_T^x)
   \bigg]d_\rho^{\pi_\theta} (\d x),
    \end{align}
    where we define $$Q^{\pi}(x,a):=r(x,a)+\gamma\int_\cX V^{\pi}(x')P(\d x'|x,a),$$ 
    and 
    $$d_\rho^{\pi}(\d x):=\int_\cX d^{\pi}(\d x|x')\rho(\d x'), \,\,\textit{ and }\,\, d^{\pi}(\d x|x'):=(1-\gamma)\sum_{n=0}^\infty\gamma^n P_{\pi}^n(\d x|x'),$$
    with $P^n_\pi$ being the $n$-times product of the kernel $P_\pi$ under policy $\pi$ with $P^0_\pi(\d x'|x):= \delta_{x'}(\d x)$. 
\end{prop}

% The proof has two steps. First, for a bounded test function \(\phi\), the
% Girsanov argument in Appendix~\ref{app:policy-gradient} shows that
% \[
% \nabla_\theta \mathbb E[\phi(\bar a_T^x)]
% =
% \mathbb E\left[
% \left(
% \int_0^T
% \nabla_\theta f_\theta(t,\bar a_t^x,x)
% \sigma(t,\bar a_t^x,x)^{-1}\,dB_t
% \right)
% \phi(\bar a_T^x)
% \right].
% \]

This identity is in the same spirit as the classical policy-gradient theorem in
RL~\cite{sutton1999policy} and related policy-gradient formulas for discrete-time
diffusion policies with Gaussian transition kernels~\cite{psenka2024learning}.
In the discrete-time setting, the internal denoising trajectory has an explicit
product density, so the policy score decomposes into a sum of transition-kernel
scores. In contrast, our policy is defined as the terminal law of a
continuous-time SDE with a general drift and state-dependent volatility. The
terminal action density is generally unavailable, and the corresponding pathwise
likelihood ratio must be represented through a Girsanov change of measure. This
replaces the discrete sum of internal scores by the stochastic integral $\int_0^T
    \nabla_\theta f_\theta(t,\bar a^x_t,x)
    \sigma(t,\bar a_t^x,x)^{-1}\,dB_t$,
which accounts for the additional technical work in
Appendix~\ref{app:policy-gradient}. At the same time, the continuous-time
expression gives a discretization-independent training identity, applies
directly to general SDE-based diffusion policies, and separates the role of the
drift sensitivity \(\nabla_\theta f_\theta\) from the sampling discretization
used in implementation. This gives a compact training formula for SDE-based
diffusion policies and supports the implementation of the diffusion-policy
classes studied above.

\section{Numerical experiments}\label{sec:experiments}
Using the policy-gradient expression specified in Proposition~\ref{prop:policygradient}, we provide two numerical studies illustrating the two finite-sample regimes identified in Section~\ref{sec:statistical-estimation}. The first experiment uses a generic neural drift class in a nonlinear multi-step MDP. This corresponds to Theorem~\ref{thm:lipschitz-selector-complete-tradeoff}, where the leading statistical term contains the extra \(K\)-factor and the optimal diffusion budget scales as \(K\asymp n^{2/(m+6)}\). The second experiment uses a mean-reverting, one-sided dissipative diffusion
class. This corresponds to
Remark~\ref{rem:dissipative-comparison-section4}, where we first choose
\(s=s(n)\asymp(n/\log n)^{m/(m+4)}\) and then study the effect of varying \(K\).
%Thus, the generic case exhibits a U-shaped trade-off in \(K\), whereas the dissipative case does not: the policy regret decreases toward a sample-dependent plateau, and \(K\asymp(n/\log n)^{2/(m+4)}\) is a minimum sufficient scale rather than an optimizer.

%The second experiment uses a mean-reverting, one-sided dissipative diffusion class. This corresponds to Remark~\ref{rem:dissipative-comparison-section4}, where the stability estimate removes the extra \(K\)-factor and improves the scaling to \(K_n\asymp (n/\log n)^{2/(m+4)}\). In both cases, the qualitative message is the same: larger \(K\) improves diffusion approximation, but the statistically stable choice of \(K\) must grow with the available sample size.

\subsection{Experiment with generic drifts}
This experiment is designed to test the generic neural-drift regime of Theorem~\ref{thm:lipschitz-selector-complete-tradeoff}. The drift is represented by a generic neural network and is trained directly by policy gradient, without imposing a one-sided dissipativity structure. Therefore the relevant guide for the sample-dependent diffusion budget is \(K(n)\asymp n^{2/(m+6)}\).

%Using the policy-gradient expression specified in Proposition~\ref{prop:policygradient}, we first test the approximation--estimation trade-off on a genuine multi-step stochastic MDP. In contrast to the controlled example in the next subsection, the policy is not restricted to a mean-reverting center form. The drift is a generic neural network depending on diffusion time, the current internal action, and the current MDP state, and it is trained directly from rollout returns. The reference feedback used below defines the cost landscape and the evaluation baseline, but it is not given to the actor during training.

{\bf Setup.} We consider a two-dimensional stochastic MDP with state
\(X_h=(X_h^{(1)},X_h^{(2)})\in\mathbb R^2\) and scalar bounded action
\(A_h\in[-R_a,R_a]\). The reward is the negative stage cost,
\[
    r(x,a)
    = -\Big(
    c_a(a-g_{\rm ref}(x))^2
    +c_x\big((x^{(1)})^2+0.8(x^{(2)})^2\big)
    +c_s\sin^2(\kappa_s x^{(1)})
    +c_u a^2\Big).
\]
Here \(g_{\rm ref}\) is a smooth saturated nonlinear feedback map. It is used to
define a nontrivial reward landscape and to compute the reference cost.%, but it is not used as supervision when training the actor.

The state transition is nonlinear and action-dependent:
\[
\begin{aligned}
X_{h+1}^{(1)}
&=
\alpha_{11}X_h^{(1)}
+\alpha_{12}X_h^{(2)}
+\beta_1\tanh(A_h)
+\delta_1\sin(\omega_1X_h^{(2)})
+\varepsilon_h^{(1)},\\
X_{h+1}^{(2)}
&=
\alpha_{22}X_h^{(2)}
-\alpha_{21}\sin(X_h^{(1)})
+\beta_2\tanh(\omega_a A_h)
+\delta_2\sin(\omega_2X_h^{(1)}X_h^{(2)})
+\varepsilon_h^{(2)}.
\end{aligned}
\]

In the experiment, we set $R_a=2.4,X_0\sim{\rm Unif}([-2.2,2.2]^2)$, and define
\[
\begin{aligned}
g_{\rm ref}(x)
=
R_a\tanh\Big(1.15\big[
&1.55\sin(2.25x^{(1)}+0.85x^{(2)})
+0.90\cos(1.65x^{(1)}-2.15x^{(2)})\\
&+0.55\sin(3.2x^{(1)}x^{(2)})
-0.35x^{(1)}
+0.25x^{(2)}
\big]\Big).
\end{aligned}
\]
The cost parameters are specified as $c_a=2.0,  c_x=0.055,  c_s=0.030, c_u=0.010$, and $\kappa_s=1.4.$
The transition parameters are $\alpha_{11}=0.78,\ \alpha_{12}=0.27, 
\alpha_{22}=0.70,  \alpha_{21}=0.22,$
$\beta_1=0.26, \beta_2=0.20,
\delta_1=0.08, \delta_2=0.06,
\omega_1=1.6, \omega_2=2$, and $\omega_a=1.2.$
The noise satisfies $(\varepsilon_h^{(1)},\varepsilon_h^{(2)})\sim N(0,0.03^2I_2),$ and the next state is clipped to \([-3,3]^2\). We use horizon \(H=10\) and
discount factor \(\gamma=0.95\).

{\bf Implementation details.}
At each MDP step, the policy samples the action from the terminal value of the internal action diffusion
\[
    d\bar a_t
    =
    f_{\theta,K}(t,\bar a_t,X_h)\,dt+\sigma\,dB_t,
\]
where \(T=1\), \(\sigma=0.70\), and \(\bar a_0\sim N(0,1.10^2)\).  The drift is represented by a two-hidden-layer ReLU network,
\[
    f_{\theta,K}(t,a,x)
    =
    K\,
    \tanh\!\left(
        {\rm MLP}_{\theta}\!\left(t,a,x\right)
    \right),
\]
with the network width chosen as $ w(K)=\min\{256,\max\{16,{\rm round}(16+2K)\}\},$ so the width grows approximately as \(K^{m/2}\) on the main part of the grid.

We train the actor using the path-likelihood score estimator from the Euler-discretized action SDE, together with a learned critic baseline evaluated at the pre-action state \(X_h\). We use Adam with actor learning rate \(2.5\times10^{-3}\), critic learning rate \(5\times10^{-3}\), gradient clipping at \(8\), weight clipping at \(4\), and \(50\) policy-gradient updates. We evaluate
$
K\in\{2,4,8,12,16,24,32,48,64,96,128,256,512\}$,
use $n\in\{64,256,2048\}$ trajectories per policy-gradient update, average over \(50\) independent seeds, and evaluate each trained policy using \(1024\) held-out episodes. We report held-out discounted cost relative to the deterministic reference feedback \(g_{\rm ref}\). The reference feedback is used only for constructing the cost and for evaluation, not for actor training.

\begin{figure}[H]
    \centering
    \begin{subfigure}[t]{0.32\textwidth}
        \centering
        \includegraphics[width=\textwidth]{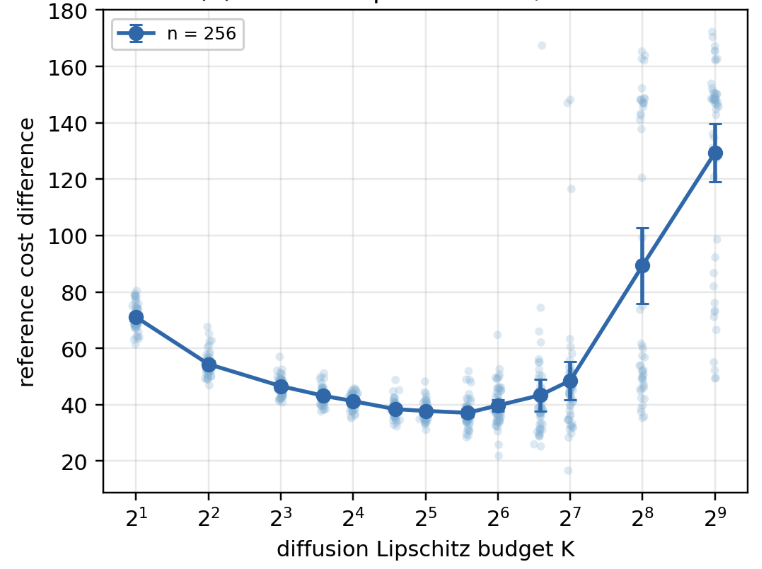}
        \caption{Decomposition under \(n=512\).}
        \label{fig:k-tradeoff-fixed-n-mdp}
    \end{subfigure}
    \begin{subfigure}[t]{0.32\textwidth}
        \centering
        \includegraphics[width=\textwidth]{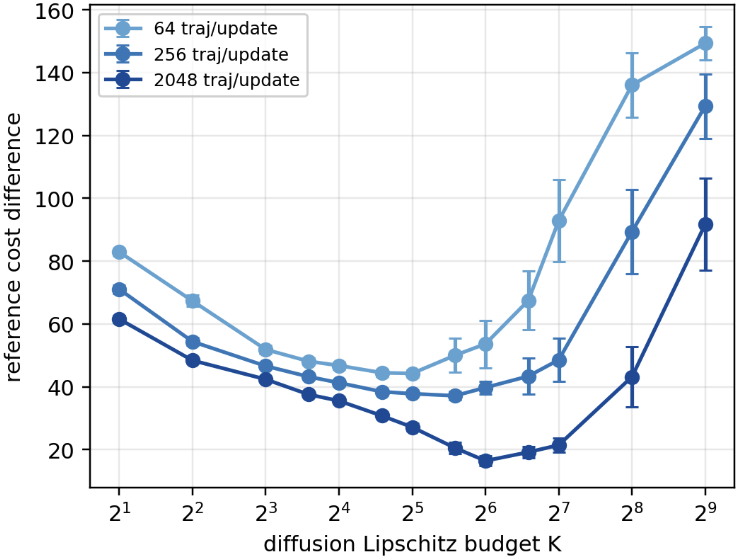}
        \caption{Value difference $V^*(\rho)-V({f}_{\widehat \theta,K})$ as a function of $K$, for different \(n\).}
        \label{fig:k-tradeoff-multi-n-mdp}
    \end{subfigure}
    \begin{subfigure}[t]{0.32\textwidth}
        \centering
        \includegraphics[width=\textwidth]{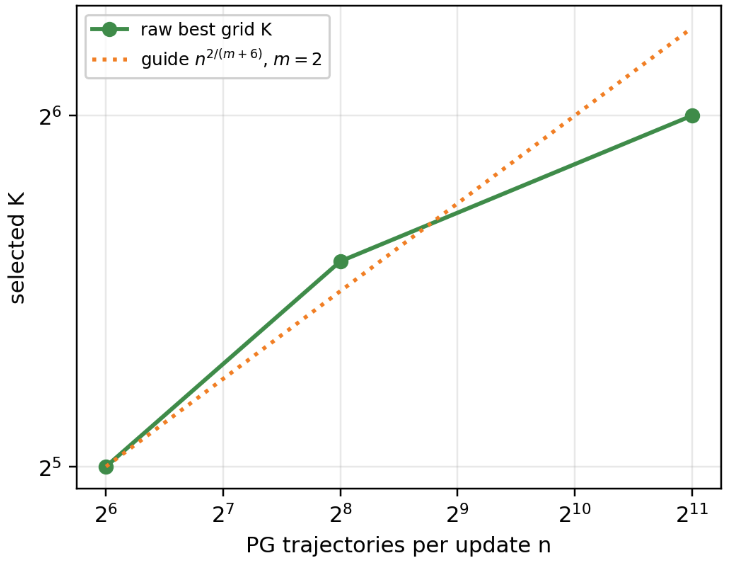}
        \caption{Empirical optimal \(K(n)\) versus \(n\).}
        \label{fig:k-optimal-scaling-mdp}
    \end{subfigure}
    \caption{
       Empirical approximation–estimation trade-off in $K$.
    Panel~(a) fixes \(n=256\) trajectories per policy-gradient update and shows a U-shaped dependence on the diffusion Lipschitz budget \(K\).
    Panel~(b) compares different values of \(n\), showing that larger policy-gradient batches tolerate larger values of \(K\).
    Panel~(c) reports the best grid value of \(K\), together with the generic guide \(K(n)\asymp n^{2/(m+6)}\) for \(m=2\).
    Markers show averages over \(50\) independent seeds and error bars indicate \(95\%\) confidence intervals.
    }
    \label{fig:k-tradeoff-mdp}
\end{figure}

{\bf Results.}
Figure~\ref{fig:k-tradeoff-mdp}(a) shows the fixed-sample trade-off at \(n=256\). The reference cost difference decreases from about \(71.06\) at \(K=2\) to \(37.14\) at \(K=48\). It then increases for larger budgets, reaching \(89.36\) at \(K=256\) and \(129.36\) at \(K=512\). Thus, increasing \(K\) is initially beneficial, but overly large \(K\) makes policy-gradient learning unstable under a fixed sample budget. Figure~\ref{fig:k-tradeoff-mdp}(b) shows that increasing the number of trajectories per update shifts the stable range of \(K\) to the right. With \(n=64\), the best grid value is \(K=32\), with mean reference cost difference \(44.24\). With \(n=256\), the best grid value is \(K=48\), with mean reference cost difference \(37.14\). With \(n=2048\), the best grid value is \(K=64\), with mean reference cost difference \(16.42\). The best-\(K\) summary in Figure~\ref{fig:k-tradeoff-mdp}(c) gives an increasing sequence \(K=32,48,64\) over \(n=64,256,2048\). Since \(K\) is selected from a coarse grid, the empirical curve is stepwise, but it follows the predicted monotone growth of the sample-dependent diffusion budget.

Overall, this MDP experiment supports the generic finite-sample trade-off in Theorem~\ref{thm:lipschitz-selector-complete-tradeoff}. Increasing \(K\) improves the expressive power of the neural diffusion policy in the low-to-moderate \(K\) regime, but a generic neural drift also leads to larger policy-gradient variability and stronger optimization instability at large \(K\). The empirically selected diffusion budget increases with \(n\) and is consistent with the guide \(K(n)\asymp n^{2/(m+6)}\).

\subsection{Experiment with dissipative drifts}
\label{sec:simple_exp}
In this section, we consider a simpler setting of one-step contextual bandit and parameterie the drift in the dissipative regime described in
Remark~\ref{rem:dissipative-comparison-section4}. Here the policy uses a
mean-reverting diffusion around a learned center, so the drift is
one-sided dissipative in the action variable. We first choose $s(n)\asymp \big({n}/{\log n}\big)^{m/(m+4)}$
and hold this center-class size fixed while varying \(K\). Consequently,
there is no U-shaped trade-off in \(K\): increasing \(K\) reduces the
diffusion error until the policy regret reaches the sample-dependent center
estimation error.

{\bf Setup.}
We consider a simpler setting of one-step contextual bandit with
\(X\sim {\rm Unif}([0,1]^m)\), action space
\(\mathcal A=\mathbb R^d\), and reward
\(r(x,a)=-\|a-g^\star(x)\|^2\), where \(g^\star\) is a bounded
Lipschitz selector generated from random Fourier features. The optimal
value is zero, so the value gap is exactly the mean-squared distance from
the sampled action to \(g^\star(X)\).
%For a learnable center \(h\), we use a mean-reverting diffusion around \(h\).

{\bf Implementation details.}
We use \(m=4\), \(d=2\), \(T=1\), and simulate the action diffusion with
\(24\) time steps. The initial action has scale \(1\), and the volatility
takes values in \([0.10,1.00]\). 
We construct the oracle selector $g^\star$
 using \(2048\) random Fourier features and choose its output scale so that the target actions have magnitude at most $R_0=1.5$.

For each sample size \(n\), we parameterize the center by
\[
h_W(x)=\Phi_{s(n)}(x)^\top W,\qquad
\Phi_s(x)=(1,\phi_1(x),\ldots,\phi_s(x)),
\]
where $s(n)
=
\lfloor
16\left(\frac{n}{\log n}\right)^{m/(m+4)}
\rfloor$. Given noisy samples
\(Y_i=g^\star(X_i)+\xi_i\), with
\(\xi_i\sim N(0,0.10^2I_d)\), we fit \(W\) by ridge regression with
regularization $\lambda
=
10^{-6}\max\left\{1,\frac{s(n)+1}{n}\right\}$.
Note that for each sample size \(n\), we choose \(s=s(n)\), fit the center
\(\widehat h_n\), and keep it fixed while varying \(K\). Thus, changing
\(K\) affects the diffusion approximation error but not the center
estimation error.

We evaluate $K\in
\{1,2,4,6,8,12,16,24,32,48,64,96,128,192,256,512\}$ and
\(n\in\{128,256,512,1024,2048\}\), using five independent training
seeds. Evaluation uses \(1024\) contexts and four terminal-action samples
per context.

We report {\bf (i)} the policy regret $R_n(K)
:=
\mathbb E\big[
\big\|
\bar a_T^{\widehat h_n,X}-g^\star(X)
\big\|^2
\big],$
{\bf (ii)} the diffusion error around the fitted center $\mathbb E\!\big[
\big\|
\bar a_T^{\widehat h_n,X}-\widehat h_n(X)
\big\|^2
\big],$
and {\bf (iii)} the direct center MSE $\mathbb E\!\big[
\big\|
\widehat h_n(X)-g^\star(X)
\big\|^2
\big]$. We mark \(\underline K\) as the smallest value of \(K\) for which the
diffusion error is comparable to the center estimation error. Beyond this
point, increasing \(K\) gives little further improvement in policy regret.

{\bf Results.} Figure~\ref{fig:k-tradeoff-main}(a) shows that, for each fixed sample
size, the learned policy gap decreases monotonically with \(K\) and
approaches the corresponding direct center-MSE plateau. The plateaus
decrease as \(n\) increases because a larger sample size supports a more
accurate fitted center. In particular, unlike the generic neural-drift
experiment, there is no U-shaped dependence on \(K\) in this dissipative
experiment. Figure~\ref{fig:k-tradeoff-main}(b) shows that the direct center MSE
decreases consistently with the theoretical guide
\((n/\log n)^{-2/(m+4)}\). This is the approximation--estimation
trade-off that determines \(s(n)\), and it is independent of \(K\). Figure~\ref{fig:k-tradeoff-main}(c) shows the behavior at \(n=512\).
The diffusion error around the fitted center decreases with \(K\), while
the direct center MSE remains constant because the same center
\(\widehat h_n\) is used for every \(K\). Consequently, the policy regret
decreases toward the center-estimation floor. The marked
\(\underline{K}\) is the smallest diffusion scale for which the diffusion
term is comparable to this floor; taking a larger \(K\) does not improve
the leading statistical rate.
\begin{figure}[H]
    \centering
    \begin{subfigure}[t]{0.32\textwidth}
        \centering
         \includegraphics[width=\textwidth]{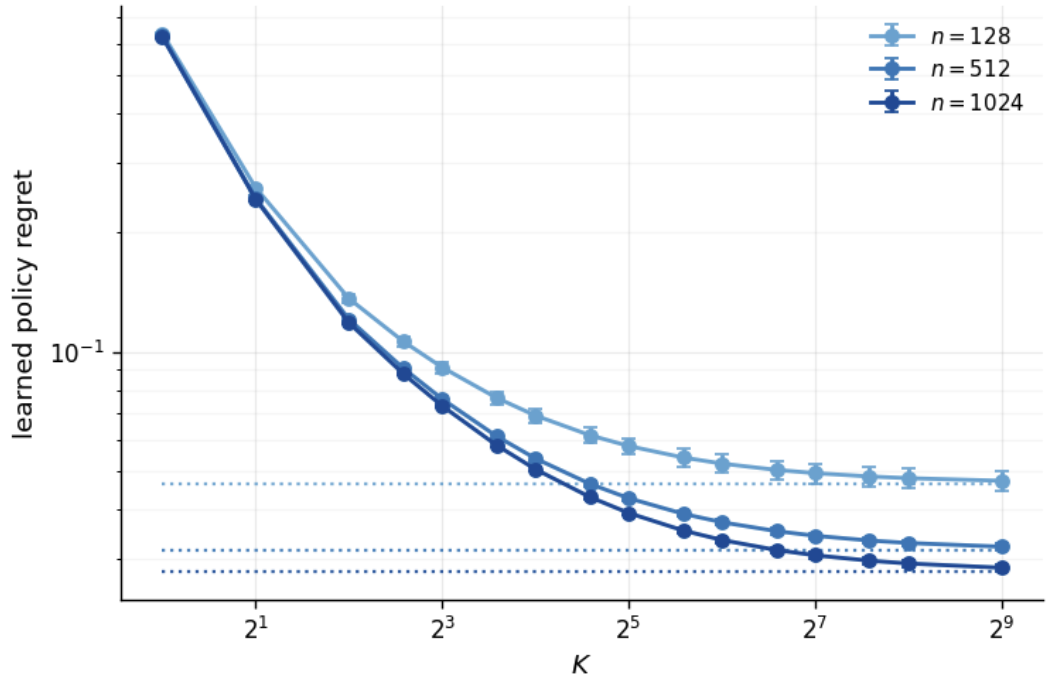}
        \caption{Policy regret for different \(n\).}
        \label{fig:k-tradeoff-multi-n}
    \end{subfigure}
    \begin{subfigure}[t]{0.32\textwidth}
        \centering
        
         \includegraphics[width=\textwidth]{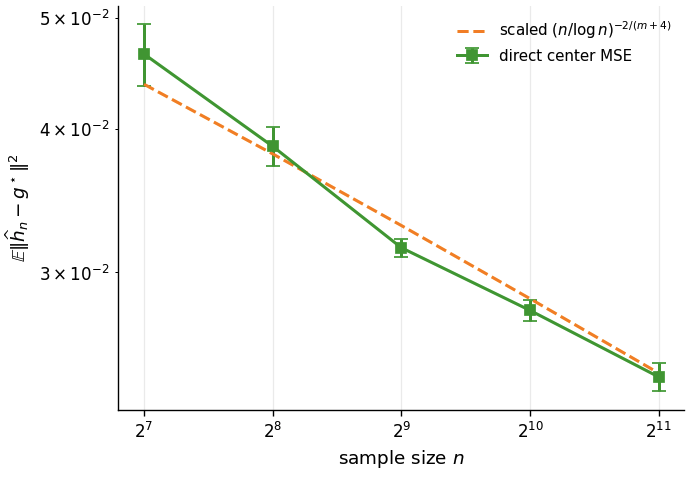}
        \caption{Direct center MSE versus \(n\).}
        \label{fig:center-mse-rate}
    \end{subfigure}
    \begin{subfigure}[t]{0.32\textwidth}
        \centering
         \includegraphics[width=\textwidth]{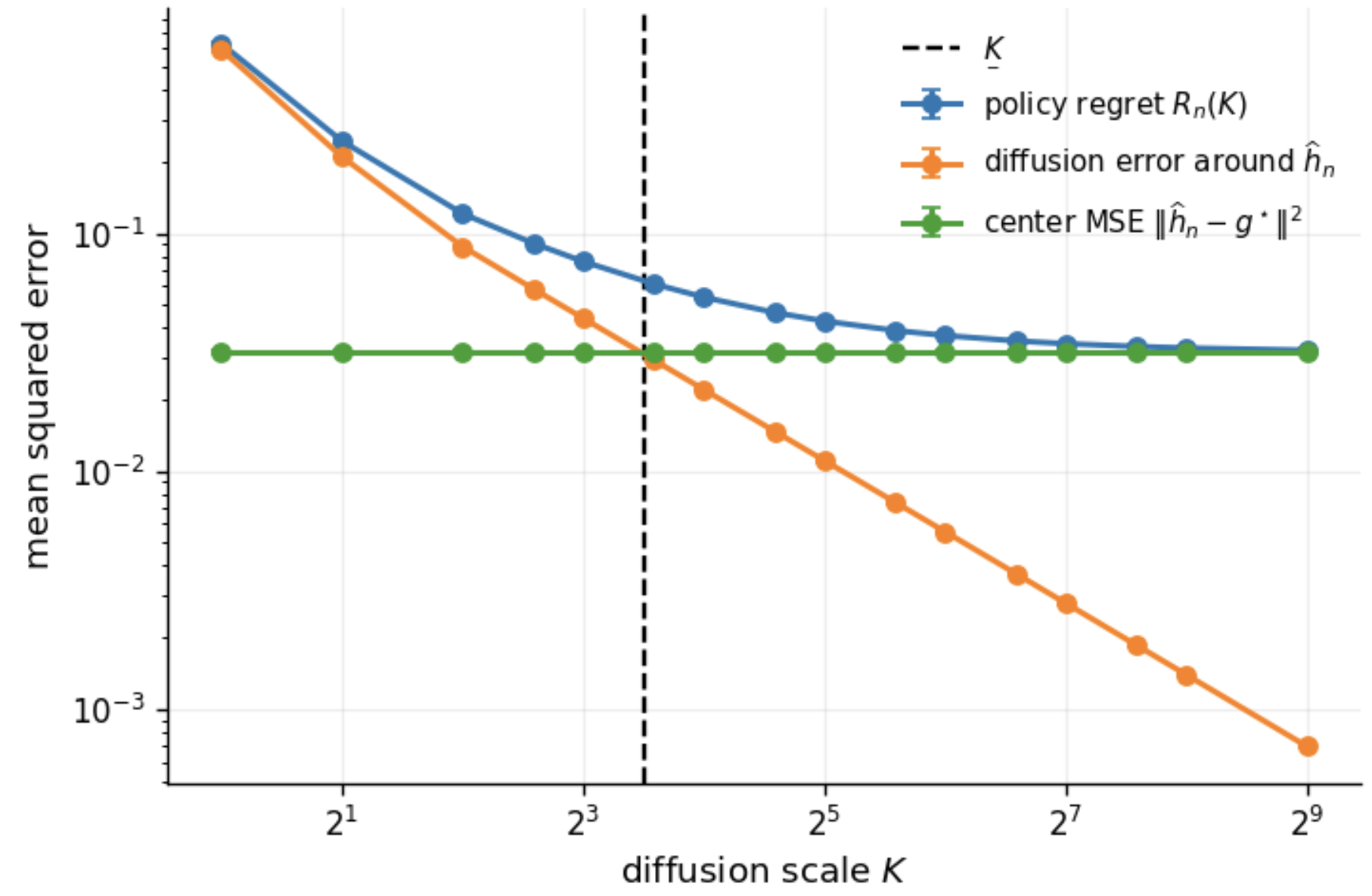}
        \caption{Error components at \(n=512\).}
        \label{fig:k-tradeoff-fixed-n}
    \end{subfigure}
    \caption{
    Dissipative regime with \(s=s(n)\) fixed across \(K\).
    Panel~(a) shows that the learned gaps decrease with \(K\) toward the
    sample-dependent direct center-MSE plateaus, indicated by the dotted
    lines. Panel~(b) compares the direct center MSE with the scaled rate
    \((n/\log n)^{-2/(m+4)}\). Panel~(c) compares the policy regret,
    diffusion error around the fitted center, and direct center MSE at
    \(n=512\); the dashed line marks \(\underline{K}\), the
smallest diffusion scale at which the diffusion error becomes comparable to the
center MSE. (Markers show averages over five independent repetitions, and error bars
indicate standard errors.)}  
    \label{fig:k-tradeoff-main}
\end{figure}

Overall, the two experiments illustrate distinct roles of \(K\). In the generic neural-drift regime, increasing \(K\) improves expressivity but also increases statistical and optimization difficulty, producing a U-shaped trade-off. In the dissipative regime, the additional mean-reverting structure stabilizes the policy class: once \(s=s(n)\) is fixed, increasing \(K\) mainly reduces the diffusion error until the policy regret reaches the center-estimation floor. Thus \(K\asymp (n/\log n)^{2/(m+4)}\) is a minimum sufficient scale under this additional structure, not an optimizer for a generic drift class. In general neural implementations, such dissipativity would need to be enforced through  projection of the parameters which may be computationally expensive; in the toy numerical example of Section \ref{sec:simple_exp}, it holds by construction.

\section{Conclusion}\label{sec:conclusion}
We develop a theoretical framework for understanding the expressivity and finite-sample behavior of diffusion policies in reinforcement learning. Our results show that the drift Lipschitz budget \(K\) controls a fundamental approximation--estimation trade-off: larger \(K\) improves diffusion approximation, but also increases the statistical complexity of learning. For ReLU drift classes, this yields an explicit finite-sample oracle inequality and a sample-dependent choice of \(K\), supported qualitatively by our numerical experiments.

This work opens several directions for further study. One interesting question is whether the finite-sample bound in Section~\ref{sec:statistical-estimation} can be sharpened using Bernstein-type concentration, under stronger variance or stability assumptions. It would also be valuable to  incorporate optimization error from practical policy-gradient training.

\paragraph{Acknowledgment.} The authors would like to thank Hao Liu (Stanford MS\&E) for helpful comments on an earlier version. R.X. is partially supported by an NSF CAREER Award DMS-2614933
and a gift fund from Point72. 
Y.Z.~is partially supported by the Imperial Global Connect Fund.
J.Z. is partially supported by the Hong Kong RGC Early Career Scheme (ECS) under grant no. 24203525.

\bibliographystyle{plain}
\bibliography{bibliography}

\newpage
\appendix
\section{Proofs for Section \ref{sec:expressivity}}
This appendix contains the proofs for the expressivity results in Section~\ref{sec:expressivity}. We first prove the constructive upper bound showing that mean-reverting diffusion policies concentrate near deterministic selectors, then prove the matching lower bound showing that the resulting \(K^{-1}\) rate is intrinsic under nondegenerate diffusion noise.
\label{app:expressivity}
\subsection{Proof of Lemma~\ref{lem:mean_rev}}
\label{sec:proof_mean_rev}

Let \(c\in\mathbb R^d\) denote the target action and define, as shorthand,
\[
    z_K:=z_{K}(c,\tau)=\frac{c}{1-e^{-K\tau}}.
\]
as $c, \tau$ do not change in this proof. Consider
\[
    d\bar a^x_s
    =
    K(z_K-\bar a^x_s)\,ds
    +
    \sigma(s,\bar a^x_s,x)\,dB_s,
    \qquad
    \bar a^x_0\sim \mu_0 .
\]
Define the moving center
\[
    m_s:=(1-e^{-Ks})z_K,\qquad 0\le s\le \tau .
\]
Then \(m_0=0\), \(m_\tau=c\), and
$
    \dot m_s=K(z_K-m_s).
$ Set
$
    Y_s:=\bar a^x_s-m_s$, so that
\[
\begin{aligned}
    dY_s
    &=
    d\bar a^x_s-\dot m_s\,ds  \\
    &=
    K(z_K-\bar a^x_s)\,ds
    -
    K(z_K-m_s)\,ds
    +
    \sigma(s,\bar a^x_s,x)\,dB_s  \\
    &=
    -K(\bar a^x_s-m_s)\,ds
    +
    \sigma(s,\bar a^x_s,x)\,dB_s  \\
    &=
    -K Y_s\,ds
    +
    \sigma(s,\bar a^x_s,x)\,dB_s .
\end{aligned}
\]
Moreover, \(Y_0=\bar a^x_0\) and \(Y_\tau=\bar a^x_\tau-c\). By It\^o's formula,
\[
    d\|Y_s\|^2
    =
    \left(
        -2K\|Y_s\|^2
        +
        \operatorname{tr}\bigl(
            \sigma(s,\bar a^x_s,x)\sigma(s,\bar a^x_s,x)^\top
        \bigr)
    \right)ds
    +
    2\langle Y_s,\sigma(s,\bar a^x_s,x)\,dB_s\rangle .
\]
Under Assumption~\ref{ass:vol_bdd},
\[
    \operatorname{tr}\bigl(
        \sigma(s,\bar a^x_s,x)\sigma(s,\bar a^x_s,x)^\top
    \bigr)
    \le d\Lambda .
\]
Taking expectations gives
\[
    \frac{d}{ds}\mathbb E\|Y_s\|^2
    \le
    -2K\mathbb E\|Y_s\|^2+d\Lambda .
\]
Gronwall's inequality yields
\[
\begin{aligned}
    \mathbb E\|Y_\tau\|^2
    &\le
    e^{-2K\tau}\mathbb E\|Y_0\|^2
    +
    d\Lambda\int_0^\tau e^{-2K(\tau-r)}\,dr  \\
    &=
    e^{-2K\tau}m_2(\mu_0)
    +
    \frac{d\Lambda}{2K}
    \bigl(1-e^{-2K\tau}\bigr).
\end{aligned}
\]
Since \(Y_\tau=\bar a^x_\tau-c\), we obtain
\[
    \mathbb E\|\bar a^x_\tau-c\|^2
    \le
    e^{-2K\tau}m_2(\mu_0)
    +
    \frac{d\Lambda}{2K}
    \bigl(1-e^{-2K\tau}\bigr).
\]
This proves the claim.

\subsection{Proof of Theorem \ref{thm:ub_soft}}\label{sec:proof_thm_ub_soft}

We now prove the value-approximation theorem stated in Theorem \ref{thm:ub_soft}. The proof lifts the terminal-action concentration from Lemma~\ref{lem:mean_rev} to convergence of the induced one-step reward and transition kernel, and then propagates this convergence through the discounted Bellman recursion.

\begin{proof}
Below, we allow the volatility to depend on \(K\) and write it as \(\sigma_K\);
if the volatility is fixed, simply read \(\sigma_K=\sigma\). We prove the result
by converting the controlled MDP into a Markov reward process. The main point is
that, when the volatility is state-dependent, the terminal noise is no longer
independent of \(x\). We therefore use only the concentration of the terminal
action around \(g^\star(x)\).

From Lemma \ref{lem:mean_rev}, we have
\[
    \E\|\bar a_T^{K,x}-g^\star(x)\|^2
    \le
    e^{-2KT}m_2(\mu_0)
    +
    \frac{d\Lambda}{2K}
    \bigl(1-e^{-2KT}\bigr)
    =:\varepsilon_K .
\]
Since \(\varepsilon_K\to0\), we have, for every fixed \(x\),
\[
    \Delta_K(x):=\|\bar a_T^{K,x}-g^\star(x)\|
    \longrightarrow 0
    \qquad\text{in }L^2\text{ and hence in probability.}
\]

Define the averaged reward and transition kernel induced by the diffusion policy:
\[
    r_K(x)
    :=
    \E\left[r(x,\bar a_T^{K,x})\right],
\]
and
\[
    P_K(A\mid x)
    :=
    \E\left[
        P(A\mid x,\bar a_T^{K,x})
    \right],
    \qquad A\in\mathcal B(\mathcal X).
\]
For the deterministic policy \(\pi^\star(\cdot\mid x)=\delta_{g^\star(x)}\), define
\[
    r^\star(x):=r(x,g^\star(x)),
    \qquad
    P^\star(A\mid x):=P(A\mid x,g^\star(x)).
\]

We first show pointwise convergence of the one-step reward and transition kernel.
Fix \(x\in\mathcal X\), and let
\[
    A_{K,x}:=\{\Delta_K(x)\le R_x\},
\]
where \(R_x\) is the radius from Assumption~\ref{ass:cont_env_soft}. Recall
$
    r_{\max}:=\sup_{x,a}|r(x,a)|<\infty .
$
On \(A_{K,x}\), Assumption~\ref{ass:cont_env_soft} gives
\[
    |r(x,\bar a_T^{K,x})-r(x,g^\star(x))|
    \le
    \omega_{r,x}(\Delta_K(x)).
\]
On \(A_{K,x}^c\), we use the trivial bound
$
    |r(x,\bar a_T^{K,x})-r(x,g^\star(x))|
    \le
    2r_{\max}.
$
Hence
\[
    |r_K(x)-r^\star(x)|
    \le
    \E\!\left[
        \omega_{r,x}(\Delta_K(x))\mathbf 1_{A_{K,x}}
    \right]
    +
    2r_{\max}\mathbb P(A_{K,x}^c).
\]
Since \(\Delta_K(x)\to0\) in probability,
$
    \mathbb P(A_{K,x}^c)
    =
    \mathbb P(\Delta_K(x)>R_x)
    \to0. 
$ 
Moreover, taking the modulus to be nondecreasing, for any \(0<\eta<R_x\),
\[
\begin{aligned}
    \E\!\left[
        \omega_{r,x}(\Delta_K(x))\mathbf 1_{A_{K,x}}
    \right]
    &\le
    \omega_{r,x}(\eta)
    +
    \omega_{r,x}(R_x)\mathbb P(\Delta_K(x)>\eta).
\end{aligned}
\]
Letting \(K\to\infty\) and then \(\eta\downarrow0\) gives
$
    r_K(x)\to r^\star(x).
$

Similarly, by convexity of total variation under mixtures,
\[
\begin{aligned}
    \operatorname{TV}\!\left(P_K(\cdot\mid x),P^\star(\cdot\mid x)\right)
    &\le
    \E\left[
        \operatorname{TV}\!\left(
            P(\cdot\mid x,\bar a_T^{K,x}),
            P(\cdot\mid x,g^\star(x))
        \right)
    \right]  \\
    &\le
    \E\!\left[
        \omega_{P,x}(\Delta_K(x))\mathbf 1_{A_{K,x}}
    \right]
    +
    2\mathbb P(A_{K,x}^c).
\end{aligned}
\]
The same argument as above yields
\[
    \operatorname{TV}\!\left(P_K(\cdot\mid x),P^\star(\cdot\mid x)\right)
    \to 0.
\]

Now define the finite-horizon value functions for the two Markov reward processes.
Set
$
    V_0^K=V_0^\star=0,
$ and recursively define
\[
    V_{n+1}^K(x)
    =
    r_K(x)+\gamma\int_{\mathcal X}V_n^K(y)P_K(dy\mid x),
\]
and
\[
    V_{n+1}^\star(x)
    =
    r^\star(x)+\gamma\int_{\mathcal X}V_n^\star(y)P^\star(dy\mid x).
\]
We prove by induction that, for every fixed \(n\ge0\) and every fixed
\(x\in\mathcal X\),
\[
    V_n^K(x)\to V_n^\star(x).
\]
The case \(n=0\) is immediate. Suppose the claim holds for \(n\). Since rewards are bounded,
\[
    \sup_K\|V_n^K\|_\infty
    \vee
    \|V_n^\star\|_\infty
    \le
    \frac{r_{\max}}{1-\gamma}.
\]
For fixed \(x\),
\begin{eqnarray*}
    &&\left|
        \int V_n^K(y)P_K(dy\mid x)
        -
        \int V_n^\star(y)P^\star(dy\mid x)
    \right| \\
   &\le&
    \left|
        \int
        \bigl(V_n^K(y)-V_n^\star(y)\bigr)
        P^\star(dy\mid x)
    \right|
    +
    2\|V_n^K\|_\infty
    \operatorname{TV}\!\left(P_K(\cdot\mid x),P^\star(\cdot\mid x)\right).
\end{eqnarray*}
The first term goes to zero by dominated convergence under the fixed measure
\(P^\star(\cdot\mid x)\). The second term goes to zero by the total-variation
convergence shown above. Together with \(r_K(x)\to r^\star(x)\), this gives
\[
    V_{n+1}^K(x)\to V_{n+1}^\star(x).
\]
The induction is complete.

It remains to pass from finite horizon to infinite horizon. Let
$V^K:=V^{\pi_K}$ and $V^\star:=V^{\pi^\star}$. For every \(n\ge0\),
\[
    \|V^K-V_n^K\|_\infty
    \vee
    \|V^\star-V_n^\star\|_\infty
    \le
    \frac{\gamma^nr_{\max}}{1-\gamma}.
\]
Therefore, for every fixed \(x\),
\[
\begin{aligned}
    \limsup_{K\to\infty}|V^K(x)-V^\star(x)|
    &\le
    \frac{2\gamma^nr_{\max}}{1-\gamma}
    +
    \limsup_{K\to\infty}|V_n^K(x)-V_n^\star(x)| \\
    &=
    \frac{2\gamma^nr_{\max}}{1-\gamma}.
\end{aligned}
\]
Letting \(n\to\infty\) yields
\[
    V^K(x)\to V^\star(x)
    \qquad\text{for every }x\in\mathcal X.
\]
Finally,
\[
    |V^K(x)-V^\star(x)|
    \le
    \frac{2r_{\max}}{1-\gamma},
\]
so dominated convergence with respect to the initial state law \(\rho\) gives
\[
    V^{\pi_K}(\rho)
    =
    \int_{\mathcal X}V^K(x)\rho(dx)
    \longrightarrow
    \int_{\mathcal X}V^\star(x)\rho(dx)
    =
    V^{\pi^\star}(\rho).
\]
Since \(\pi^\star(\cdot\mid x)=\delta_{g^\star(x)}\) is an optimal deterministic policy,
\(V^{\pi^\star}(\rho)=V^\star(\rho)\). This proves the theorem.
\end{proof}
\subsection{Smoothing and anti-concentration for state-independent volatility}
This subsection proves the smoothing estimate needed for the final-window lower-bound argument in Proposition \ref{prop:lb_diffusion}. The goal is to show that, after rescaling the last time interval of length \(K^{-1}\), the residual diffusion has an endpoint density uniformly bounded from above. This prevents the terminal action from concentrating too much near any finite set of target actions.

We proceed in two steps. Lemma \ref{lem:c2} first gives an a priori \(L^2\) estimate for solutions of the Fokker-Planck equation. Lemma \ref{lem:additive-anti-concentration} then combines this estimate with Nash's inequality and a duality argument to obtain a uniform \(L^\infty\) bound on the density at the terminal time. This \(L^\infty\)  bound implies the anti-concentration estimate used in Appendix \ref{sec:proof_prop_34_first}.

\begin{lem}[A priori $L^2$ finiteness for SDE marginal density under state-independent volatility]\label{lem:c2}
Fix $0 \le r < t \le 1$. Let
\[
\Gamma(s) := \sigma(s)\sigma(s)^\top, \qquad \Gamma_{ij}(s) = [\Gamma(s)]_{ij}.
\]
Assume that $\beta(s,\cdot)$ is smooth and $1$-Lipschitz uniformly in $s$,
that $\beta(\cdot,0)\in L^1([r,t])$, and that $\Gamma(s)$ is deterministic,
positive semidefinite, and bounded on $[r,t]$. Let $p_s$ be a nonnegative
probability density solving
\[
\partial_s p_s(x)
=
-\sum_{i=1}^d \partial_i\bigl(\beta_i(s,x)p_s(x)\bigr)
+
\frac12\sum_{i,j=1}^d \Gamma_{ij}(s)\partial_{ij}p_s(x),
\qquad p_r=\phi,
\]
where $\phi\in L^1(\mathbb R^d)\cap L^2(\mathbb R^d)$ is a probability density.
Then, for every $u\in[r,t]$, it holds that
\[
\|p_u\|_2^2
\le
e^{d(u-r)}\|\phi\|_2^2.
\]
In particular, $p_u\in L^2(\mathbb R^d)$ for every $u\in[r,t]$.
\end{lem}

\begin{proof}
For $M>0$, define the truncated quadratic function
\[
\Psi_M(z)
=
\begin{cases}
z^2, & 0\le z\le M,\\
2Mz-M^2, & z>M.
\end{cases}
\]
Then $\Psi_M$ is convex, $C^1$, and satisfies
$
0\le \Psi_M(z)\le 2Mz$,
$\Psi_M(z)\le z^2,
$
and $\Psi_M(z)\uparrow z^2$ as $M\to\infty$.

Also define
\[
H_M(z):=z\Psi_M'(z)-\Psi_M(z).
\]
A direct computation gives
\[
H_M(z)=z^2\wedge M^2,
\qquad
0\le H_M(z)\le \Psi_M(z).
\]

Let $\chi_R\in C_c^\infty(\mathbb R^d)$ be a standard cutoff such that
\[
0\le \chi_R\le 1,\qquad
\chi_R=1 \text{ on } B_R,\qquad
\operatorname{supp}\chi_R\subset B_{2R},
\]
and
\[
|\nabla \chi_R|\le \frac{C}{R},
\qquad
|D^2\chi_R|\le \frac{C}{R^2}.
\]
Set
\[
F_{M,R}(s):=\int_{\mathbb R^d}\Psi_M(p_s(x))\chi_R(x)\,dx.
\]
This quantity is finite because $\Psi_M(p_s)\le 2Mp_s$ and $p_s$ has total
mass one.

We differentiate $F_{M,R}(s)$ and use the Fokker--Planck equation:
\[
\frac{d}{ds}F_{M,R}(s)
=
\int \Psi_M'(p_s)\chi_R\,\partial_s p_s.
\]
For the drift term,
\[
-\int \Psi_M'(p_s)\chi_R
\sum_{i=1}^d \partial_i(\beta_i p_s)
=
-\int (\nabla\cdot \beta_s)H_M(p_s)\chi_R
+
\int \Psi_M(p_s)\beta_s\cdot\nabla\chi_R.
\]
For the diffusion term, since $\Gamma_{ij}(s)$ does not depend on $x$,
\[
\frac12
\int \Psi_M'(p_s)\chi_R
\sum_{i,j=1}^d \Gamma_{ij}(s)\partial_{ij}p_s
=
-\frac12
\int \Psi_M''(p_s)\chi_R
\sum_{i,j=1}^d \Gamma_{ij}(s)\partial_i p_s\,\partial_j p_s
+
\frac12
\int \Psi_M(p_s)
\sum_{i,j=1}^d \Gamma_{ij}(s)\partial_{ij}\chi_R.
\]
The first diffusion term is nonpositive because $\Psi_M$ is convex and
$\Gamma(s)$ is positive semidefinite. Since $\beta_s$ is $1$-Lipschitz,
\[
-\nabla\cdot\beta_s\le d.
\]
Therefore,
\[
\frac{d}{ds}F_{M,R}(s)
\le
dF_{M,R}(s)+E_{M,R}(s),
\]
where
\[
E_{M,R}(s)
:=
\int \Psi_M(p_s)\beta_s\cdot\nabla\chi_R
+
\frac12
\int \Psi_M(p_s)
\sum_{i,j=1}^d \Gamma_{ij}(s)\partial_{ij}\chi_R.
\]

We now send $R\to\infty$. Let
\[
A_R:=\{x\in\mathbb R^d:R\le |x|\le 2R\}.
\]
The support of $\nabla\chi_R$ and $D^2\chi_R$ is contained in $A_R$. Since
$\beta_s$ is $1$-Lipschitz,
\[
|\beta(s,x)|\le |\beta(s,0)|+|x|.
\]
Using $\Psi_M(p_s)\le 2Mp_s$, the cutoff bounds, and boundedness of $\Gamma(s)$,
we get
\[
|E_{M,R}(s)|
\le
C M\bigl(1+|\beta(s,0)|+\|\Gamma(s)\|_{\mathrm{op}}\bigr)
\int_{A_R}p_s(x)\,dx.
\]
For each fixed $s$, $\int_{A_R}p_s(x)\,dx\to0$ as $R\to\infty$, because
$p_s$ is a probability density. The right-hand side is dominated by an
integrable function of $s$ on $[r,t]$. Hence, by dominated convergence,
\[
\int_r^u |E_{M,R}(s)|\,ds\to0
\qquad\text{as }R\to\infty.
\]

Integrating the differential inequality from $r$ to $u$ gives
\[
F_{M,R}(u)
\le
F_{M,R}(r)
+
d\int_r^u F_{M,R}(s)\,ds
+
\int_r^u E_{M,R}(s)\,ds.
\]
Letting $R\to\infty$ yields
\[
F_M(u)
\le
F_M(r)
+
d\int_r^u F_M(s)\,ds,
\]
where
\[
F_M(s):=\int_{\mathbb R^d}\Psi_M(p_s(x))\,dx.
\]
By Gronwall's inequality,
\[
F_M(u)\le e^{d(u-r)}F_M(r).
\]
At the initial time, $p_r=\phi$, and $\Psi_M(\phi)\le \phi^2$. Therefore,
\[
F_M(u)
\le
e^{d(u-r)}\|\phi\|_2^2.
\]
Finally, since $\Psi_M(z)\uparrow z^2$, monotone convergence gives
\[
\int_{\mathbb R^d}p_u(x)^2\,dx
=
\lim_{M\to\infty}
\int_{\mathbb R^d}\Psi_M(p_u(x))\,dx
\le
e^{d(u-r)}\|\phi\|_2^2.
\]
This proves the claim.

The computation with $\Psi_M$ can be made fully classical by first replacing
$\Psi_M$ by smooth convex approximations with the same bounds and then
letting the approximation parameter tend to zero.
\end{proof}
\begin{lem}
\label{lem:additive-anti-concentration}
Let $(Y_t)_{0\le t\le 1}$ solve
$$
dY_t=\beta(t,Y_t)\,dt+\sigma(t)\,dW_t,
\qquad
Y_0=y\in\mathbb R^d .
$$
Assume that $\beta(t,\cdot)$ is $1$-Lipschitz uniformly in $t$, $\beta(\cdot, 0)\in L^{1}([0, 1])$, and that
$\sigma(t)$ is deterministic, state-independent, and square-integrable on
$[0,1]$. Assume also that there is $\kappa>0$ such that
$$
\sigma(t)\sigma(t)^\top\succeq \kappa I_d
$$
for a.e. $t\in[0,1]$.

Then the law of $Y_1$ has a density $p_1^y$ with respect to Lebesgue measure,
and there is a constant $C_{d,\kappa}<\infty$, depending only on $d$ and
$\kappa$, such that
$$
\sup_{y\in\mathbb R^d}\|p_1^y\|_\infty\le C_{d,\kappa}.
$$
Consequently, for any $z_1,\ldots,z_m\in\mathbb R^d$ and any $\rho>0$,
$$
\mathbb P_y\left(
Y_1\in \bigcup_{j=1}^m B(z_j,\rho)
\right)
\le
C_{d,\kappa}m v_d\rho^d,
$$
where $v_d=\lambda_d(B(0,1))$, with $\lambda_{d}$ the Lebesgue measure in $\mathbb{R}^{d}$.
\end{lem}
\begin{proof}
We first prove the result under the additional assumptions that $\beta$ is
smooth in the spatial variable, $\sigma$ is bounded, and the initial law has
a smooth probability density. We then remove these assumptions by approximation.
All the integration by parts in Steps 1-2 can be justified by inserting standard cut-offs $\chi_R$ as in Lemma \ref{lem:c2} and then let $R\rightarrow\infty$.

{\bf Step 1: forward $L^1\to L^2$ smoothing for smooth probability densities.}

Fix $0\le r<t\le 1$. Suppose the law of $Y_r$ has smooth (compactly supported) probability density
$\varphi$. Let $p_t$ be the density of $Y_t$. We write
$$
p_t=P_{r,t}^\ast \varphi,
$$
where $P_{r,t}^\ast$ denotes the forward Fokker--Planck evolution of densities.
Since $\varphi$ is a probability density, $p_t$ is also a probability density:
$$
p_t\ge 0,
\qquad
\int_{\mathbb R^d}p_t(x)\,dx=1.
$$
Furthermore, Lemma \ref{lem:c2} ensures that $p_{t}\in L^{2}(\mathbb{R}^{d})$ for $t>r$. The Fokker--Planck equation is
$$
\partial_t p_t
=
-\nabla\cdot(\beta_t p_t)
+
\frac12\sum_{i,j=1}^d
[\sigma(t)\sigma(t)^\top]_{ij}\partial_{ij}p_t,
$$
where $\beta_t(x)=\beta(t,x)$.

We compute the evolution of
$$
\|p_t\|_2^2=\int_{\mathbb R^d}p_t(x)^2\,dx.
$$
Multiplying the Fokker--Planck equation by $2p_t$ and integrating gives
$$
\frac{d}{dt}\|p_t\|_2^2
=
-2\int p_t\nabla\cdot(\beta_t p_t)
+
\int p_t
\sum_{i,j=1}^d
[\sigma(t)\sigma(t)^\top]_{ij}\partial_{ij}p_t .
$$

For the drift term, use
$$
\nabla\cdot(\beta_t p_t)
=
(\nabla\cdot\beta_t)p_t+\beta_t\cdot\nabla p_t .
$$
Then
$$
-2\int p_t\nabla\cdot(\beta_t p_t)
=
-2\int(\nabla\cdot\beta_t)p_t^2
-
2\int p_t\beta_t\cdot\nabla p_t .
$$
Since $2p_t\beta_t\cdot\nabla p_t=\beta_t\cdot\nabla(p_t^2)$,
integration by parts gives
$$
-2\int p_t\beta_t\cdot\nabla p_t
=
-\int \beta_t\cdot\nabla(p_t^2)
=
\int(\nabla\cdot\beta_t)p_t^2 .
$$
Therefore
$$
-2\int p_t\nabla\cdot(\beta_t p_t)
=
-\int(\nabla\cdot\beta_t)p_t^2 .
$$

For the diffusion term, $\sigma(t)\sigma(t)^\top$ is independent of $x$, so
integration by parts gives
$$
\int p_t
\sum_{i,j=1}^d
[\sigma(t)\sigma(t)^\top]_{ij}\partial_{ij}p_t
=
-\int
\nabla p_t^\top\sigma(t)\sigma(t)^\top\nabla p_t .
$$

Combining the two identities,
$$
\frac{d}{dt}\|p_t\|_2^2
=
-\int(\nabla\cdot\beta_t)p_t^2
-
\int\nabla p_t^\top\sigma(t)\sigma(t)^\top\nabla p_t .
$$

Since $\beta_t$ is $1$-Lipschitz, we have $-\nabla\cdot\beta_t\le d$. And since
$\sigma(t)\sigma(t)^\top\succeq\kappa I_d$, we have
$$
\nabla p_t^\top\sigma(t)\sigma(t)^\top\nabla p_t
\ge
\kappa \|\nabla p_t\|^2 .
$$
Hence
$$
\frac{d}{dt}\|p_t\|_2^2
\le
d\|p_t\|_2^2-\kappa\|\nabla p_t\|_2^2 .
$$

Define $q_t:=e^{-d(t-r)/2}p_t$. Then $\|q_t\|_2^2=e^{-d(t-r)}\|p_t\|_2^2$. Differentiating it gives
$$
\frac{d}{dt}\|q_t\|_2^2
=
-d e^{-d(t-r)}\|p_t\|_2^2
+
e^{-d(t-r)}\frac{d}{dt}\|p_t\|_2^2 .
$$
Using the preceding inequality,
$$
\frac{d}{dt}\|q_t\|_2^2
\le
-d e^{-d(t-r)}\|p_t\|_2^2
+
e^{-d(t-r)}
\left(
d\|p_t\|_2^2-\kappa\|\nabla p_t\|_2^2
\right).
$$
The two terms involving $d\|p_t\|_2^2$ cancel, so
$$
\frac{d}{dt}\|q_t\|_2^2
\le
-\kappa e^{-d(t-r)}\|\nabla p_t\|_2^2 .
$$
Because $q_t=e^{-d(t-r)/2}p_t$ and the exponential factor is independent of
$x$, we have $\nabla q_t=e^{-d(t-r)/2}\nabla p_t$, and therefore
$$
\|\nabla q_t\|_2^2=e^{-d(t-r)}\|\nabla p_t\|_2^2 .
$$
Thus
$$
\frac{d}{dt}\|q_t\|_2^2
\le
-\kappa\|\nabla q_t\|_2^2 .
$$

Since $p_t$ is a probability density,
$$
\|q_t\|_1
=
e^{-d(t-r)/2}\|p_t\|_1
=
e^{-d(t-r)/2}
\le 1 .
$$

We now apply Nash's inequality on $\mathbb R^d$: there exists
$N_d<\infty$, depending only on $d$, such that for every
$u\in L^1(\mathbb R^d)\cap H^1(\mathbb R^d)$,
$$
\|u\|_2^{2+4/d}
\le
N_d\|\nabla u\|_2^2\|u\|_1^{4/d}.
$$
Applying this to $u=q_t$ gives
$$
\|q_t\|_2^{2+4/d}
\le
N_d\|\nabla q_t\|_2^2\|q_t\|_1^{4/d}
\le
N_d\|\nabla q_t\|_2^2 .
$$
Therefore
$$
\|\nabla q_t\|_2^2
\ge
\frac1{N_d}\|q_t\|_2^{2+4/d}.
$$
Combining with the energy inequality,
$$
\frac{d}{dt}\|q_t\|_2^2
\le
-\frac{\kappa}{N_d}\|q_t\|_2^{2+4/d}.
$$

Set $F(t):=\|q_t\|_2^2$. Then $\|q_t\|_2^{2+4/d}=F(t)^{1+2/d}$. Thus
$$
F'(t)\le -\frac{\kappa}{N_d}F(t)^{1+2/d}.
$$
Let $\theta=2/d$ and $c_{d,\kappa}=\kappa/N_d$, where $N_d$ is the coefficient in Nash's inequality. Then
$$
F'(t)\le -c_{d,\kappa}F(t)^{1+\theta}.
$$
If $F(t)=0$, the desired bound is trivial. Otherwise,
$$
\frac{d}{dt}F(t)^{-\theta}
=
-\theta F(t)^{-\theta-1}F'(t)
\ge
\theta c_{d,\kappa} .
$$
Integrating from $r$ to $t$ gives
$$
F(t)^{-\theta}-F(r)^{-\theta}
\ge
\theta c_{d,\kappa} (t-r).
$$
Since $F(r)^{-\theta}\ge 0$, we have $F(t)^{-\theta}\ge \theta c_{d,\kappa}(t-r)$. Therefore
$$
F(t)
\le
(\theta c_{d,\kappa}(t-r))^{-1/\theta}.
$$
Substituting $\theta=2/d$ and $c_{d,\kappa}=\kappa/N_d$, we have $F(t)
\le
\left(\frac{2\kappa}{d\,N_d}(t-r)\right)^{-d/2}.$
Hence
$$
\|q_t\|_2
\le
\left(\frac{d\,N_d}{2\kappa}\right)^{d/4}(t-r)^{-d/4}.
$$
Since $p_t=e^{d(t-r)/2}q_t$ and $t-r\le 1$,
$$
\|p_t\|_2
\le
C_{d,\kappa}(t-r)^{-d/4}.
$$
This proves the forward smoothing estimate.

\medskip

\textbf{Step 2: a backward estimate and the $L^\infty$ bound.}

We now obtain an $L^2\to L^\infty$ estimate for the forward evolution.
Let $P_{r,t}$ denote the backward Markov semigroup:
$$
(P_{r,t}g)(x)=\mathbb E[g(Y_t)\mid Y_r=x].
$$
For this step, $g$ is a test function. It need not be a density.

Fix $0\le r<t\le 1$ and first take $g\ge 0$. Define
$$
w_s(x)=P_{t-s,t}g(x),
\qquad
0\le s\le t-r .
$$
Then $w_0=g$, and $w_s$ satisfies the backward equation, written forward in
the variable $s$:
$$
\partial_s w_s
=
\beta(t-s,x)\cdot\nabla w_s
+
\frac12\sum_{i,j=1}^d
[\sigma(t-s)\sigma(t-s)^\top]_{ij}\partial_{ij}w_s .
$$
Here
$$
\beta(t-s,x)\cdot\nabla w_s(x)
=
\sum_{i=1}^d \beta_i(t-s,x)\partial_{x_i}w_s(x).
$$

Repeating the same $L^2$ integration by parts as in Step 1 gives
$$
\frac{d}{ds}\|w_s\|_2^2
\le
d\|w_s\|_2^2-\kappa\|\nabla w_s\|_2^2 .
$$

We also need an $L^1$ bound on $w_s$. Since $g\ge0$ and $P_{t-s,t}$ preserves
nonnegativity, we have $w_s\ge0$. Hence $\|w_s\|_1=\int w_s$. Integrating the
backward equation in $x$ gives
$$
\frac{d}{ds}\|w_s\|_1
=
\int \beta(t-s,x)\cdot\nabla w_s(x)\,dx .
$$
The second-order term integrates to zero. By integration by parts,
$$
\int \beta(t-s,x)\cdot\nabla w_s(x)\,dx
=
-\int(\nabla\cdot\beta)(t-s,x)w_s(x)\,dx .
$$
Since $-\nabla\cdot\beta\le d$ and $w_s\ge0$,
$$
\frac{d}{ds}\|w_s\|_1
\le
d\|w_s\|_1 .
$$
By Gronwall's inequality,
$$
\|w_s\|_1\le e^{ds}\|w_0\|_1=e^{ds}\|g\|_1 .
$$

Define $v_s=e^{-ds/2}w_s$. Exactly as in Step 1,
$$
\frac{d}{ds}\|v_s\|_2^2
\le
-\kappa\|\nabla v_s\|_2^2 .
$$
Moreover,
$$
\|v_s\|_1
=
e^{-ds/2}\|w_s\|_1
\le
e^{-ds/2}e^{ds}\|g\|_1
=
e^{ds/2}\|g\|_1
\le
e^{d/2}\|g\|_1 .
$$
Applying Nash's inequality to $v_s$ gives
$$
\|v_s\|_2^{2+4/d}
\le
N_d\|\nabla v_s\|_2^2\|v_s\|_1^{4/d}
\le
N_d e^2 \|\nabla v_s\|_2^2\|g\|_1^{4/d}.
$$
Therefore,
$$
\|\nabla v_s\|_2^2
\ge
\frac{1}{N_d e^2\|g\|_1^{4/d}}\|v_s\|_2^{2+4/d}.
$$
Repeating the differential inequality argument from Step 1 yields
$$
\|v_s\|_2
\le
C_{d,\kappa}s^{-d/4}\|g\|_1 .
$$
Since $w_s=e^{ds/2}v_s$ and $s\le1$,
$$
\|w_s\|_2
\le
C_{d,\kappa}s^{-d/4}\|g\|_1 .
$$
Taking $s=t-r$ gives
$$
\|P_{r,t}g\|_2
\le
C_{d,\kappa}(t-r)^{-d/4}\|g\|_1 .
$$
For signed $g$, apply this estimate to $g_+$ and $g_-$ and use linearity.
Thus the same estimate holds for all $g\in L^1$.

Now use duality. For \(h\in L^2(\mathbb R^d)\) and
\(g\in C_c^\infty(\mathbb R^d)\) with \(\|g\|_1\le 1\), adjointness gives
\[
\left|
\int_{\mathbb R^d} P_{r,t}^\ast h(x)g(x)\,dx
\right|
=
\left|
\int_{\mathbb R^d} h(x)P_{r,t}g(x)\,dx
\right|.
\]
Hence, by Cauchy--Schwarz and the backward estimate,
$$
\left|
\int P_{r,t}^\ast h\,g
\right|
\le
\|h\|_2\|P_{r,t}g\|_2
\le
C_{d,\kappa}(t-r)^{-d/4}\|h\|_2\|g\|_1 .
$$
Taking the supremum over $\|g\|_1\le1$ gives
$$
\|P_{r,t}^\ast h\|_\infty
\le
C_{d,\kappa}(t-r)^{-d/4}\|h\|_2 .
$$

Apply this over two half intervals. If the initial law at time $0$ has smooth
probability density $\varphi$, then
$$
p_1=P_{1/2,1}^\ast P_{0,1/2}^\ast\varphi .
$$
By Step 1,
$$
\|P_{0,1/2}^\ast\varphi\|_2\le C_{d,\kappa}.
$$
By the $L^2\to L^\infty$ estimate,
$$
\|p_1\|_\infty
=
\|P_{1/2,1}^\ast(P_{0,1/2}^\ast\varphi)\|_\infty
\le
C_{d,\kappa}\|P_{0,1/2}^\ast\varphi\|_2
\le
C_{d,\kappa}.
$$

\medskip

\textbf{Step 3: point-mass initial condition.}

We now keep the smooth drift and bounded volatility assumptions, but start
from $Y_0=y$.

Let $\eta\in C_c^\infty(\mathbb R^d)$ be such that
$\eta\ge0$ and $\int \eta=1$. Define
$$
\varphi_n(x)=n^d\eta(n(x-y)).
$$
Then each $\varphi_n$ is a smooth probability density and
$\varphi_n(x)\,dx\Rightarrow \delta_y$ weakly.

Let $X_n$ have density $\varphi_n$. We may write $X_n=y+Z/n$, where $Z$ has
density $\eta$, so $X_n\to y$ almost surely. Let $Y_t^{X_n}$ and $Y_t^y$ be
solutions driven by the same Brownian motion, started from $X_n$ and $y$.
Because the volatility is state-independent,
$$
d(Y_t^{X_n}-Y_t^y)
=
\left[\beta(t,Y_t^{X_n})-\beta(t,Y_t^y)\right]dt .
$$
Since $\beta(t,\cdot)$ is $1$-Lipschitz,
$$
|Y_1^{X_n}-Y_1^y|\le e|X_n-y|.
$$
Hence $Y_1^{X_n}\to Y_1^y$ almost surely, and therefore in distribution.

For each $n$, Step 2 gives
$$
\mathbb P(Y_1^{X_n}\in E)\le C_{d,\kappa}\lambda_d(E)
$$
for every Borel set $E$. Let $O$ be open. By the portmanteau theorem,
$$
\mathbb P(Y_1^y\in O)
\le
\liminf_{n\to\infty}\mathbb P(Y_1^{X_n}\in O)
\le
C_{d,\kappa}\lambda_d(O).
$$
For a Borel set $E$, choose open $O\supset E$ with
$\lambda_d(O)\le \lambda_d(E)+\varepsilon$. Then
$$
\mathbb P(Y_1^y\in E)
\le
\mathbb P(Y_1^y\in O)
\le
C_{d,\kappa}(\lambda_d(E)+\varepsilon).
$$
Letting $\varepsilon\downarrow0$ gives
$$
\mathbb P(Y_1^y\in E)\le C_{d,\kappa}\lambda_d(E).
$$
Thus the law of $Y_1^y$ has density bounded by $C_{d,\kappa}$.

\medskip

\textbf{Step 4: remove smoothness of $\beta$ and boundedness of $\sigma$.}

We now remove the two auxiliary assumptions used in Steps 1--3. The final
statement only assumes that $\beta(t,\cdot)$ is $1$-Lipschitz and that
$\sigma(t)\sigma(t)^\top\succeq \kappa I_d$.

\medskip

\textbf{Step 4 (a): remove smoothness of $\beta$, keeping $\sigma$ bounded.}

Let $\eta\in C_c^\infty(\mathbb R^d)$ be a standard mollifier:
$\eta\ge0$, $\int_{\mathbb R^d}\eta(z)\,dz=1$, and
$\operatorname{supp}\eta\subset B(0,1)$. Define
$$
\eta_n(z)=n^d\eta(nz).
$$
For each $n$, define the spatial mollification of $\beta$ by
$$
\beta_n(t,x)
=
(\eta_n\ast_x \beta(t,\cdot))(x)
=
\int_{\mathbb R^d}\eta_n(z)\beta(t,x-z)\,dz.
$$
The convolution is only in the spatial variable $x$, not in time.

For each fixed $t$, the function $\beta_n(t,\cdot)$ is smooth. Moreover,
$\beta_n(t,\cdot)$ is still $1$-Lipschitz. Indeed, for $x,x'\in\mathbb R^d$,
$$
\begin{aligned}
|\beta_n(t,x)-\beta_n(t,x')|
&=
\left|
\int_{\mathbb R^d}\eta_n(z)
\left[
\beta(t,x-z)-\beta(t,x'-z)
\right]dz
\right|  \\
&\le
\int_{\mathbb R^d}\eta_n(z)
|\beta(t,x-z)-\beta(t,x'-z)|\,dz  \\
&\le
\int_{\mathbb R^d}\eta_n(z)|x-x'|\,dz  \\
&=
|x-x'|.
\end{aligned}
$$
Thus the Lipschitz constant is not enlarged by mollification.

Also, $\beta_n$ converges uniformly to $\beta$. Since $\beta(t,\cdot)$ is
$1$-Lipschitz uniformly in $t$,
$$
\begin{aligned}
|\beta_n(t,x)-\beta(t,x)|
&=
\left|
\int_{\mathbb R^d}\eta_n(z)
\left[
\beta(t,x-z)-\beta(t,x)
\right]dz
\right| \\
&\le
\int_{\mathbb R^d}\eta_n(z)|z|\,dz.
\end{aligned}
$$
Changing variables $u=nz$ gives
$$
\int_{\mathbb R^d}\eta_n(z)|z|\,dz
=
\frac1n\int_{\mathbb R^d}\eta(u)|u|\,du.
$$
Hence, with
$$
C_\eta:=\int_{\mathbb R^d}\eta(u)|u|\,du,
$$
we have
$$
\sup_{t\in[0,1]}\sup_{x\in\mathbb R^d}
|\beta_n(t,x)-\beta(t,x)|
\le
\frac{C_\eta}{n}.
$$

Let $Y^{(n)}$ solve
$$
dY_t^{(n)}
=
\beta_n(t,Y_t^{(n)})\,dt+\sigma(t)\,dW_t,
\qquad
Y_0^{(n)}=y,
$$
and let $Y$ solve
$$
dY_t
=
\beta(t,Y_t)\,dt+\sigma(t)\,dW_t,
\qquad
Y_0=y.
$$
Drive both equations by the same Brownian motion. Since the volatility is the
same and is state-independent,
$$
Y_t^{(n)}-Y_t
=
\int_0^t
\left[
\beta_n(s,Y_s^{(n)})-\beta(s,Y_s)
\right]ds.
$$
Therefore
$$
\begin{aligned}
|Y_t^{(n)}-Y_t|
&\le
\int_0^t
|\beta_n(s,Y_s^{(n)})-\beta_n(s,Y_s)|\,ds  
+
\int_0^t
|\beta_n(s,Y_s)-\beta(s,Y_s)|\,ds  \\
&\le
\int_0^t |Y_s^{(n)}-Y_s|\,ds
+
\frac{C_\eta t}{n}.
\end{aligned}
$$
By Gronwall's inequality,
$$
\sup_{0\le t\le1}|Y_t^{(n)}-Y_t|
\le
\frac{C_\eta e}{n}.
$$
In particular, $Y_1^{(n)}\to Y_1$ almost surely, hence also in distribution.

By Steps 1--3, applied to the smooth drift $\beta_n$, the endpoint law of
$Y_1^{(n)}$ satisfies
$$
\mathbb P(Y_1^{(n)}\in E)
\le
C_{d,\kappa}\lambda_d(E)
$$
for every Borel set $E$, with the same constant $C_{d,\kappa}$ for all $n$.
Let $O\subset\mathbb R^d$ be open. Since $Y_1^{(n)}\Rightarrow Y_1$, the
portmanteau theorem gives
$$
\mathbb P(Y_1\in O)
\le
\liminf_{n\to\infty}\mathbb P(Y_1^{(n)}\in O)
\le
C_{d,\kappa}\lambda_d(O).
$$
For a Borel set $E$, choose open $O\supset E$ with
$\lambda_d(O)\le\lambda_d(E)+\varepsilon$. Then
$$
\mathbb P(Y_1\in E)
\le
\mathbb P(Y_1\in O)
\le
C_{d,\kappa}(\lambda_d(E)+\varepsilon).
$$
Letting $\varepsilon\downarrow0$ yields
$$
\mathbb P(Y_1\in E)
\le
C_{d,\kappa}\lambda_d(E).
$$
Thus the density bound holds for Lipschitz $\beta$, as long as $\sigma$ is
bounded.

\medskip

\textbf{Step 4 (b): remove boundedness of $\sigma$.}

Now assume only that $\sigma$ is square-integrable on $[0,1]$ and $\sigma(t)\sigma(t)^\top\succeq \kappa I_d$ for a.e. $t\in[0,1]$.

Set
$$
\Gamma(t):=\sigma(t)\sigma(t)^\top.
$$
Then $\Gamma(t)$ is symmetric positive definite for a.e. $t$, and
$\Gamma(t)\succeq \kappa I_d$. Define
$$
O(t):=\Gamma(t)^{-1/2}\sigma(t).
$$
Then $O(t)O(t)^\top
=
\Gamma(t)^{-1/2}\sigma(t)\sigma(t)^\top \Gamma(t)^{-1/2}
=
I_d$, so $O(t)$ is orthogonal for a.e. $t$.

For $N\ge\kappa$, define $\Gamma_N(t)$ by truncating the eigenvalues of $\Gamma(t)$
above at $N$. More explicitly, if
$$
\Gamma(t)=U(t)\operatorname{diag}(\lambda_1(t),\ldots,\lambda_d(t))U(t)^\top,
$$
then
$$
\Gamma_N(t)
=
U(t)\operatorname{diag}(\lambda_1(t)\wedge N,\ldots,\lambda_d(t)\wedge N)
U(t)^\top.
$$
Because each $\lambda_i(t)\ge\kappa$, we have $\Gamma_N(t)\succeq \kappa I_d$. Define
$$
\sigma_N(t):=\Gamma_N(t)^{1/2}O(t).
$$
Then $\sigma_N(t)\sigma_N(t)^\top
=
\Gamma_N(t)^{1/2}O(t)O(t)^\top \Gamma_N(t)^{1/2}
=
\Gamma_N(t)$. Therefore
$$
\sigma_N(t)\sigma_N(t)^\top\succeq \kappa I_d.
$$
Also, $\sigma_N$ is bounded, since all eigenvalues of $\Gamma_N(t)$ are at most
$N$.

\medskip

We next show that $\sigma_N\to\sigma$ in $L^2([0,1])$. Since
$\sigma(t)=\Gamma(t)^{1/2}O(t)$ and $\sigma_N(t)=\Gamma_N(t)^{1/2}O(t)$,
orthogonality of $O(t)$ gives
$$
\|\sigma_N(t)-\sigma(t)\|_{\mathrm{HS}}
=
\|\Gamma_N(t)^{1/2}-\Gamma(t)^{1/2}\|_{\mathrm{HS}}.
$$
If $\lambda_1(t),\ldots,\lambda_d(t)$ are the eigenvalues of $\Gamma(t)$, then
$$
\|\Gamma_N(t)^{1/2}-\Gamma(t)^{1/2}\|_{\mathrm{HS}}^2
=
\sum_{i=1}^d
\left(\sqrt{\lambda_i(t)\wedge N}-\sqrt{\lambda_i(t)}\right)^2.
$$
For each fixed $t$, this converges to $0$ as $N\to\infty$. Moreover,
$$
\left(\sqrt{\lambda_i(t)\wedge N}-\sqrt{\lambda_i(t)}\right)^2
\le
\lambda_i(t).
$$
Hence
$$
\|\Gamma_N(t)^{1/2}-\Gamma(t)^{1/2}\|_{\mathrm{HS}}^2
\le
\sum_{i=1}^d\lambda_i(t)
=
\operatorname{tr}(\Gamma(t))
=
\|\sigma(t)\|_{\mathrm{HS}}^2.
$$
Since $\sigma$ is square-integrable on $[0,1]$, dominated convergence gives
$$
\int_0^1
\|\sigma_N(t)-\sigma(t)\|_{\mathrm{HS}}^2\,dt
\to0.
$$

Let $Y^{(N)}$ solve
$$
dY_t^{(N)}
=
\beta(t,Y_t^{(N)})\,dt+\sigma_N(t)\,dW_t,
\qquad
Y_0^{(N)}=y,
$$
and let $Y$ solve the original equation
$$
dY_t
=
\beta(t,Y_t)\,dt+\sigma(t)\,dW_t,
\qquad
Y_0=y.
$$
Drive both equations by the same Brownian motion. Define
$$
M_t^{(N)}
:=
\int_0^t(\sigma_N(s)-\sigma(s))\,dW_s.
$$
Then
$$
Y_t^{(N)}-Y_t
=
\int_0^t
\left[
\beta(s,Y_s^{(N)})-\beta(s,Y_s)
\right]ds
+
M_t^{(N)}.
$$
Using that $\beta(t,\cdot)$ is $1$-Lipschitz,
$$
|Y_t^{(N)}-Y_t|
\le
\int_0^t |Y_s^{(N)}-Y_s|\,ds
+
\sup_{0\le u\le t}|M_u^{(N)}|.
$$
By Gronwall's inequality,
$$
\sup_{0\le t\le1}|Y_t^{(N)}-Y_t|
\le
e\sup_{0\le t\le1}|M_t^{(N)}|.
$$
By the Burkholder--Davis--Gundy inequality, or by Doob's inequality and
Itô's isometry,
$$
\mathbb E\left[\sup_{0\le t\le1}|M_t^{(N)}|^2\right]
\le
C
\int_0^1
\|\sigma_N(t)-\sigma(t)\|_{\mathrm{HS}}^2\,dt
\to0.
$$
Therefore
$
Y_1^{(N)}\to Y_1
$ in probability, hence in distribution.

By Step 4a, applied to the bounded volatility $\sigma_N$, the endpoint law of
$Y_1^{(N)}$ satisfies
$$
\mathbb P(Y_1^{(N)}\in E)
\le
C_{d,\kappa}\lambda_d(E)
$$
for every Borel set $E$, with a constant independent of $N$. Passing to the
weak limit exactly as above gives
$$
\mathbb P(Y_1\in E)
\le
C_{d,\kappa}\lambda_d(E)
$$
for every Borel set $E$. Thus the endpoint law of $Y_1$ has a density bounded by $C_{d,\kappa}$.
This completes the removal of the auxiliary assumptions.
\end{proof}

\subsection{Proof of Proposition \ref{prop:lb_diffusion}, Assumption (A2): state-independent volatility}\label{sec:proof_prop_34_first}

We first prove the matching localization lower bound in the simpler case where the volatility is state-independent. The proof separates two regimes: short horizon, where the \(K\)-Lipschitz drift cannot concentrate an absolutely continuous initialization too quickly, and longer horizons, where the final \(K^{-1}\)-length time window yields an inevitable \(K^{-1/2}\)-scale fluctuation.

\begin{proof}
All constants below may depend on $m,d,\kappa,\mu_0,R$, but not on
$K,f,\sigma,T$, nor on the locations or pairwise separations of the points
in $\mathcal{Z}$.

\paragraph{Notational note.} Throughout this proof, we will suppress the state argument in the drift $f(t, a, x)$, so that it becomes $f(t, a)$. This is because we do not use the state $x$ in our analysis. Therefore, we clarify any indication of variables $x, z,\cdots \in \mathbb{R}^{d}$ in this section are actions, not states ($d$ is the action dimension).

\medskip

\textbf{Step 1: the case \(T\le K^{-1}\).}

Under (A2), subtracting the additive noise
\(M_t:=\int_0^t \sigma(s)\,dB_s\) reduces the SDE to a random ODE with
globally \(K\)-Lipschitz drift; hence the SDE admits a pathwise unique
strong realization, and for each deterministic initial action \(x\) we may
write \(\bar a_{t,x}\) for the solution driven by the same Brownian path. For two
initial actions \(x,z\), state-independence of \(\sigma(t)\) gives
\[
d(\bar a_{t,x}-\bar a_{t,z})
=
\left[f(t,\bar a_{t,x})-f(t,\bar a_{t,z})\right]dt.
\]
Thus \(t\mapsto \bar a_{t,x}-\bar a_{t,z}\) is absolutely continuous, and
for a.e. \(t\),
\[
\frac{d}{dt}|\bar a_{t,x}-\bar a_{t,z}|^2
=
2\left\langle
\bar a_{t,x}-\bar a_{t,z},
f(t,\bar a_{t,x})-f(t,\bar a_{t,z})
\right\rangle.
\]
Since \(f(t,\cdot)\) is \(K\)-Lipschitz,
\[
\frac{d}{dt}|\bar a_{t,x}-\bar a_{t,z}|^2
\ge
-2K|\bar a_{t,x}-\bar a_{t,z}|^2 .
\]
By Gronwall's inequality,
$
|\bar a_{T,x}-\bar a_{T,z}|
\ge
e^{-KT}|x-z|.
$ Therefore, conditional on the Brownian path, the map
\(x\mapsto \bar a_{T,x}\) is injective, and its inverse on its image is
\(e^{KT}\)-Lipschitz.

Hence, for every Borel set \(E\subset\mathbb R^d\),
$$
\lambda_d\left(\{x:\bar a_{T,x}\in E\}\right)
\le
e^{dKT}\lambda_d(E).
$$
If \(T\le K^{-1}\), then \(e^{dKT}\le e^d\), and therefore
$
\lambda_d\left(\{x:\bar a_{T,x}\in E\}\right)
\le
e^d\lambda_d(E).
$
Since \(\mu_0\ll\lambda_d\), there exists \(\alpha>0\) such that
\[
\lambda_d(E)\le \alpha
\quad\Longrightarrow\quad
\mu_0(E)\le \frac12 .
\]
For \(\delta>0\), define
\[
U_\delta=\bigcup_{j=1}^J B(z_j,\delta).
\]
Since \(\lambda_d(U_\delta)\le J v_d\delta^d\), choose \(\delta>0\),
depending only on \(J,d,\mu_0\), such that
$
e^d J v_d\delta^d\le \alpha .
$
Then, conditional on the Brownian path,
$
\lambda_d\left(\{x:\bar a_{T,x}\in U_\delta\}\right)
\le
e^d\lambda_d(U_\delta)
\le
\alpha.
$
Therefore
\[
\mu_0\left(\{x:\bar a_{T,x}\in U_\delta\}\right)\le\frac12.
\]
Since \(\bar a_0\sim\mu_0\) is independent of \(B\),
$
\mathbb P(\bar a_T\in U_\delta\mid B)\le 1/2.
$
Taking expectations,
$
\mathbb P(\bar a_T\notin U_\delta)\ge1/2.
$
On the event \(\{\bar a_T\notin U_\delta\}\), we have
\(\dist(\bar a_T,\mathcal Z)\ge\delta\), and hence
\[
\dist(\bar a_T,\mathcal Z)^2\wedge R^2
\ge
\delta^2\wedge R^2.
\]
Consequently, as $K\geq 1$,
\[
\mathbb E\left[\dist(\bar a_T,\mathcal Z)^2\wedge R^2\right]
\ge
\frac{\delta^2\wedge R^2}{2}
\ge
\frac{\delta^2\wedge R^2}{2K}.
\]
\medskip

\textbf{Step 2: the case \(T>K^{-1}\).}

Assume \(T>K^{-1}\) and set
$
s=T-1/K.
$
Condition on \(\mathcal F_s\), and write
$
y=\bar a_s.
$
Let \(\phi_t\), for \(t\in[s,T]\), solve the deterministic equation
\[
\frac{d}{dt}\phi_t=f(t,\phi_t),
\qquad
\phi_s=y.
\]
The existence of \((\phi_t)\) follows from the condition that
\(f(\cdot,0)\in L^1_{\mathrm{loc}}\). For \(0\le t\le1\), define
\[
Y_t=\sqrt K\left(\bar a_{s+t/K}-\phi_{s+t/K}\right),
\]
and
\[
W_t=\sqrt K\left(B_{s+t/K}-B_s\right).
\]
Conditionally on \(\mathcal F_s\), \(W\) is a standard Brownian motion. We
compute the equation for \(Y_t\). First,
\[
\bar a_{s+t/K}
=
y+\int_s^{s+t/K}f(r,\bar a_r)\,dr
+
\int_s^{s+t/K}\sigma(r)\,dB_r,
\]
while
\[
\phi_{s+t/K}
=
y+\int_s^{s+t/K}f(r,\phi_r)\,dr.
\]
Subtracting and multiplying by \(\sqrt K\) gives
\[
Y_t
=
\sqrt K
\int_s^{s+t/K}
\left[f(r,\bar a_r)-f(r,\phi_r)\right]dr
+
\sqrt K
\int_s^{s+t/K}\sigma(r)\,dB_r.
\]
Using the change of variables \(r=s+u/K\) and the identity
\[
\bar a_{s+u/K}
=
\phi_{s+u/K}+\frac{Y_u}{\sqrt K},
\]
we obtain
\[
dY_t
=
\beta_K(t,Y_t)\,dt+\sigma(s+t/K)\,dW_t,
\qquad
Y_0=0,
\]
where
\[
\beta_K(t,u)
=
\frac{
f(s+t/K,\phi_{s+t/K}+u/\sqrt K)
-
f(s+t/K,\phi_{s+t/K})
}{\sqrt K}.
\]
The drift \(\beta_K\) is \(1\)-Lipschitz in \(u\). Indeed,
\[
\|\beta_K(t,u)-\beta_K(t,v)\|
\le
\frac1{\sqrt K}
K\left\|\frac{u-v}{\sqrt K}\right\|
=\|u-v\|.
\]
Moreover,
$
\sigma(s+t/K)\sigma(s+t/K)^\top\succeq \kappa I_d.
$ By Lemma~\ref{lem:additive-anti-concentration}, conditionally on
\(\mathcal F_s\), the random variable \(Y_1\) has density bounded by
\(C_{d,\kappa}\). Therefore, for arbitrary \(y_1,\ldots,y_J\in\mathbb R^d\)
and any \(\rho>0\),
\[
\mathbb P\left(
Y_1\in\bigcup_{j=1}^J B(y_j,\rho)
\,\middle|\,
\mathcal F_s
\right)
\le
C_{d,\kappa}J v_d\rho^d.
\]
Choose
$
\rho_0=\left(2C_{d,\kappa}J v_d\right)^{-1/d}\wedge 1.
$
Then
$
C_{d,\kappa}J v_d\rho_0^d\le 1/2.
$
Now set
$
y_j=\sqrt K(z_j-\phi_T)
$ for $j=1,\ldots,J$. Since
$
\bar a_T=\phi_T+Y_{1}/\sqrt{K},
$
we have
\[
\dist(\bar a_T,\mathcal Z)\le\frac{\rho_0}{\sqrt K}
\quad\Longleftrightarrow\quad
Y_1\in\bigcup_{j=1}^J B(y_j,\rho_0).
\]
Therefore
\[
\mathbb P\left(
\dist(\bar a_T,\mathcal Z)\le\frac{\rho_0}{\sqrt K}
\,\middle|\,
\mathcal F_s
\right)
\le
\frac12,
\]
or equivalently
\[
\mathbb P\left(
\dist(\bar a_T,\mathcal Z)>\frac{\rho_0}{\sqrt K}
\,\middle|\,
\mathcal F_s
\right)
\ge
\frac12.
\]

If \(K\ge \rho_0^2/R^2\), then \(\rho_0/\sqrt K\le R\). On the event
$
\dist(\bar a_T,\mathcal Z)>\rho_{0}/\sqrt{K},
$
we have
\[
\dist(\bar a_T,\mathcal Z)^2\wedge R^2
\ge
\frac{\rho_0^2}{K}.
\]
Thus
\[
\mathbb E\left[
\dist(\bar a_T,\mathcal Z)^2\wedge R^2
\,\middle|\,
\mathcal F_s
\right]
\ge
\frac{\rho_0^2}{K}
\mathbb P\left(
\dist(\bar a_T,\mathcal Z)>\frac{\rho_0}{\sqrt K}
\,\middle|\,
\mathcal F_s
\right)
\ge
\frac{\rho_0^2}{2K}.
\]
Taking expectations gives
\[
\mathbb E\left[\dist(\bar a_T,\mathcal Z)^2\wedge R^2\right]
\ge
\frac{\rho_0^2}{2K}.
\]

\medskip

\textbf{Step 3: combine the two regimes.}

Let
\[
K_R=1\vee\frac{\rho_0^2}{R^2},
\]
and define
\[
c_R=
\frac12
\min\left\{
\delta^2\wedge R^2,
\rho_0^2
\right\}.
\]
If \(K\ge K_R\) and \(T\le K^{-1}\), Step 1 gives
\[
\mathbb E\left[\dist(\bar a_T,\mathcal Z)^2\wedge R^2\right]
\ge
\frac{\delta^2\wedge R^2}{2K}
\ge
\frac{c_R}{K}.
\]
If \(K\ge K_R\) and \(T>K^{-1}\), Step 2 gives
\[
\mathbb E\left[\dist(\bar a_T,\mathcal Z)^2\wedge R^2\right]
\ge
\frac{\rho_0^2}{2K}
\ge
\frac{c_R}{K}.
\]
Thus, for every \(T\ge0\),
\[
\mathbb E\left[\dist(\bar a_T,\mathcal Z)^2\wedge R^2\right]
\ge
\frac{c_R}{K}.
\]
Taking the infimum over all admissible \(f,\sigma,T\) proves
\[
\inf_{f,\sigma,T\ge0}
\mathbb E\left[\dist(\bar a_T,\mathcal Z)^2\wedge R^2\right]
\ge
\frac{c_R}{K}.
\]
The proof is complete.
\end{proof}
\subsection{Proof of Theorem 
\ref{thm:lb_diffusion_initial}}\label{sec:proof_thm_34}

We next translate the finite-set localization lower bound into an RL value lower bound. The only additional ingredient is the quadratic finite-well Bellman-gap condition, which converts squared distance from optimal action set into loss in value.

\begin{proof}
By Proposition~\ref{prop:lb_diffusion}, applied for each cardinality
\(J=1,\ldots,M\), there exist constants \(c_{J,R}>0\) and
\(K_{J,R}<\infty\) such that the clipped finite-set localization lower
bound holds for every finite set \(A\subset\mathbb R^d\) with
\(|A|=m\). Set
\[
    c_{M,R}:=\min_{1\le m\le M} c_{m,R},
    \qquad
    K_0:=\max_{1\le m\le M} K_{m,R}.
\]
Then, for every finite set \(A\subset\mathbb R^d\) with
\(1\le |A|\le M\), every admissible \(K\)-Lipschitz drift, every admissible
volatility satisfying the ellipticity bounds, and every terminal time
\(T>0\),
\begin{equation}
\label{eq:clipped-localization-uniform-M}
    \mathbb E\left[
        \operatorname{dist}(\bar a_T,A)^2\wedge R^2
    \right]
    \ge
    \frac{c_{M,R}}{K},
    \qquad K\ge K_0 .
\end{equation}
The constants \(c_{M,R}\) and \(K_0\) depend only on
\((M,R,d,\kappa,\Lambda,\mu_0)\), and not on the locations or pairwise
separations of the points in \(A\).

Fix \(K\ge K_0\) and \(T>0\), and write
\(\pi=\pi_{K,T}\). For each state \(x\),
\[
\begin{aligned}
V^\star(x)-V^\pi(x)
&=
V^\star(x)
-
\int_{\mathbb R^d}
\left[
r(x,a)+\gamma\int_{\mathcal X} V^\pi(y)P(dy\mid x,a)
\right]\pi(da\mid x)  \\
&=
\int_{\mathbb R^d}
\left[
V^\star(x)-Q^\star(x,a)
\right]\pi(da\mid x)  \\
&\quad+
\gamma
\int_{\mathbb R^d}
\int_{\mathcal X}
\left[
V^\star(y)-V^\pi(y)
\right]
P(dy\mid x,a)\pi(da\mid x).
\end{aligned}
\]
Since \(V^\star\ge V^\pi\), the second term is nonnegative. Hence
\begin{equation}
\label{eq:value-gap-one-step-lower}
V^\star(x)-V^\pi(x)
\ge
\int_{\mathbb R^d}
\left[
V^\star(x)-Q^\star(x,a)
\right]\pi(da\mid x).
\end{equation}
Integrating over \(x\sim\rho\) gives
\[
V^\star(\rho)-V^\pi(\rho)
\ge
\int_{\mathcal X}
\int_{\mathbb R^d}
\left[
V^\star(x)-Q^\star(x,a)
\right]\pi(da\mid x)\rho(dx).
\]

By the finite-well Bellman gap assumption in Theorem~\ref{thm:lb_diffusion_initial},
for all \((x,a)\),
\[
    V^\star(x)-Q^\star(x,a)
    \ge
    \lambda_{M,R}
    \left(
        \operatorname{dist}(a,A^\star(x))^2\wedge R^2
    \right).
\]
Therefore,
\begin{equation}
\label{eq:value-gap-to-localization}
V^\star(\rho)-V^\pi(\rho)
\ge
\lambda_{M,R}
\int_{\mathcal X}
\int_{\mathbb R^d}
\left(
    \operatorname{dist}(a,A^\star(x))^2\wedge R^2
\right)
\pi(da\mid x)\rho(dx).
\end{equation}

For each \(x\), the set \(A^\star(x)\) is nonempty and satisfies
\(|A^\star(x)|\le M\). Moreover, under \(\pi_{K,T}\), the action at state
\(x\) is the terminal value of the state-conditioned diffusion. Thus, by
\eqref{eq:clipped-localization-uniform-M} applied with
\(A=A^\star(x)\),
\[
\int_{\mathbb R^d}
\left(
    \operatorname{dist}(a,A^\star(x))^2\wedge R^2
\right)
\pi_{K,T}(da\mid x)
=
\mathbb E\left[
    \operatorname{dist}(\bar a_T^x,A^\star(x))^2\wedge R^2
\right]
\ge
\frac{c_{M,R}}{K},
\]
uniformly in \(x\) and \(T>0\). Substituting this into
\eqref{eq:value-gap-to-localization} yields
\[
V^\star(\rho)-V^{\pi_{K,T}}(\rho)
\ge
\lambda_{M,R}
\int_{\mathcal X}
\frac{c_{M,R}}{K}\rho(dx)
=
\frac{c_{M,R}\lambda_{M,R}}{K}.
\]
The bound is uniform over \(T>0\). Hence
\[
    \inf_{T>0}
    \left(
        V^\star(\rho)-V^{\pi_{K,T}}(\rho)
    \right)
    \ge
    \frac{c_{M,R}\lambda_{M,R}}{K},
\]
which proves the claim.
\end{proof}

\subsection{Analog of Theorem \ref{thm:lb_diffusion_initial} with arbitrary initialization}\label{sec:statement_proof_cor_35}

The previous lower bound uses the absolute continuity of the initial distribution to control very short horizons. This subsection explains why the statement fails for arbitrary, possibly discrete, initialization. And then we show that the lower bound is restored once a burn-in time of order \(K^{-1}\) is imposed.

\subsubsection{Counterexample with discrete initialization}
Indeed, consider a diffusion from source $\bar{a}_{0}=0$ to the target distribution $\delta_{1}$ (point mass at $1$). Let $f(t, a, x)=M$ constant for all $a,t$ and $x$, so that
$$\bar{a}_{t}=Mt+B_{t}$$
%At time $t=1/M$, we have $\bar{a}_{t}\sim \mathcal{N}(1, 1/M)$, so that $\bar{a}_{t}$ can be arbitrary close to $1$ (in expected norm) when $M\to\infty$, with $f$ being $K$-Lipschitz for any $K$. The issues are that (i) $t$ is arbitrary close to $0$ for fixed $K$, so there is degenerate volatility, and (ii) the initial law is discrete.
At time \(t=1/M\), we have \(\bar a_t\sim N(1,1/M)\), and hence
\(\mathbb E|\bar a_t-1|^2=1/M\to0\). The drift is constant in \(a\), hence
\(K\)-Lipschitz for every \(K\). Thus no lower bound uniform over all
\(t>0\) can hold for arbitrary initialization. The obstruction is that the
time horizon can be arbitrarily small, so the accumulated diffusion variance
vanishes, while the drift magnitude is not controlled by the Lipschitz budget.
\subsubsection{Arbitrary action initialization after burn-in}
\begin{cor}[Arbitrary action initialization after burn-in]
\label{cor:arb_init}
For each \(K\geq 1\) and \(T>0\), let \(\pi_{K,T}\) be the Markov
policy whose action law at state \(x\) is the terminal law at time \(T\)
of Eq.~\eqref{eq:sde-policy}, with \(\bar a_0^x\sim P_x\), where
\(P_x\) is an arbitrary probability law on \(\mathbb R^d\), possibly
depending on \(x\), and \(\bar a_0^x\) is independent of the Brownian
motion. Assume that, apart from the initialization law, the conditions in
Proposition~\ref{prop:lb_diffusion} hold for every \(x\), with
\((\kappa,\Lambda)\) independent of \((t,x)\).

Fix \(M\geq 1\) and \(R>0\). Suppose that, for every \(x\in\mathcal X\),
$
    |A^\star(x)|\leq M.
$
Suppose moreover that there exists a constant
\(\lambda_{M,R}>0\), independent of \(x\), such that
\[
    V^\star(x)-Q^\star(x,a)
    \geq
    \lambda_{M,R}
    \left(
        \dist(a,A^\star(x))^2\wedge R^2
    \right)
\]
for all \((x,a)\). Then there exist constants
\(c_{M,R}>0\) and \(K_0<\infty\), depending only on %\renyuan{ shall we remove ``burn'' from the superscript? it looks very chatgpt. change other similar notations as well.}{\color{brown} (Jiacheng: It does not bother me too much, but we can definitely change to a more delicate notation with no such confusion)}
\(M,R,d,\kappa,\Lambda\), such that for every \(K\geq K_0\),
\[
    \inf_{T\geq K^{-1}}
    \left(
        V^\star(\rho)-V^{\pi_{K,T}}(\rho)
    \right)
    \geq
    \frac{c_{M,R}\lambda_{M,R}}{K}.
\]
Consequently, for every burn-in sequence \(t_0(K)\) satisfying
\(t_0(K)\geq K^{-1}\) for all sufficiently large \(K\),
$$
    \inf_{T\geq t_0(K)}
    \left(
        V^\star(\rho)-V^{\pi_{K,T}}(\rho)
    \right)
    \geq
    \frac{c_{M,R}\lambda_{M,R}}{K}.
$$
\end{cor}

\begin{proof}
The only difference from Theorem~\ref{thm:lb_diffusion_initial} is that the
action diffusion may start from an arbitrary law \(P_x\), so the short-time
part of Proposition~\ref{prop:lb_diffusion} cannot be used. Instead, we use
only the final-window part of the proof of Proposition~\ref{prop:lb_diffusion}.

We first record the resulting burn-in localization bound. There exist
constants \(c_{M,R}>0\) and \(K_0<\infty\), depending only on
\(M,R,d,\kappa,\Lambda\), such that the following holds. For every finite set
\(A\subset\mathbb R^d\) with \(1\leq |A|\leq M\), every admissible
\(K\)-Lipschitz drift, every admissible volatility satisfying the ellipticity
bounds, every initial law, and every terminal time \(T\geq K^{-1}\),
\[
    \mathbb E\left[
        \dist(\bar a_T,A)^2\wedge R^2
    \right]
    \geq
    \frac{c_{M,R}}{K}.
\]
Indeed, if \(T\geq K^{-1}\), set \(s=T-K^{-1}\) and condition on
\(\mathcal F_s\). The rescaling argument in the final window
\([s,s+K^{-1}]\) gives, uniformly over the random value \(\bar a_s\) and
hence uniformly over the initial law, a probability bounded below of being at
distance at least order \(K^{-1/2}\) from any \(M\)-point target set. This is
exactly the conditional argument in the final-window step of
Proposition~\ref{prop:lb_diffusion}. The pigeonhole/Krylov step supplies a
constant probability \(c_0>0\) uniformly over the target locations. Enlarging
\(K_0\) so that \(r_0/\sqrt K\leq R\) gives the clipped lower bound.

Fix \(K\geq K_0\) and \(T\geq K^{-1}\). Write \(\pi=\pi_{K,T}\). For every
state \(x\),
\[
\begin{aligned}
V^\star(x)-V^\pi(x)
&=
V^\star(x)
-
\int_{\mathbb R^d}
\left[
r(x,a)+\gamma\int_{\mathcal X} V^\pi(y)P(dy\mid x,a)
\right]\pi(da\mid x) \\
&=
\int_{\mathbb R^d}
\left[
V^\star(x)-Q^\star(x,a)
\right]\pi(da\mid x) \\
&\quad+
\gamma
\int_{\mathbb R^d}
\int_{\mathcal X}
\left[
V^\star(y)-V^\pi(y)
\right]
P(dy\mid x,a)\pi(da\mid x).
\end{aligned}
\]
Since \(V^\star\geq V^\pi\), the second term is nonnegative. Hence
\[
V^\star(x)-V^\pi(x)
\geq
\int_{\mathbb R^d}
\left[
V^\star(x)-Q^\star(x,a)
\right]\pi(da\mid x).
\]
Integrating over \(x\sim\rho\) gives
\[
V^\star(\rho)-V^\pi(\rho)
\geq
\int_{\mathcal X}
\int_{\mathbb R^d}
\left[
V^\star(x)-Q^\star(x,a)
\right]\pi(da\mid x)\rho(dx).
\]

By the clipped quadratic-gap assumption, there exists
\(\lambda_{M,R}>0\), independent of \(x\), such that
\[
    V^\star(x)-Q^\star(x,a)
    \geq
    \lambda_{M,R}
    \left(
        \dist(a,A^\star(x))^2\wedge R^2
    \right)
\]
for all \((x,a)\). Therefore,
\[
V^\star(\rho)-V^\pi(\rho)
\geq
\lambda_{M,R}
\int_{\mathcal X}
\int_{\mathbb R^d}
\left(
    \dist(a,A^\star(x))^2\wedge R^2
\right)
\pi(da\mid x)\rho(dx).
\]

For every \(x\), the set \(A^\star(x)\) is nonempty and satisfies
\(|A^\star(x)|\leq M\). Moreover, the action law \(\pi(\cdot\mid x)\) is
the terminal law at time \(T\geq K^{-1}\) of an admissible diffusion started
from the arbitrary law \(P_x\). Therefore the burn-in localization bound gives
\[
\int_{\mathbb R^d}
\left(
    \dist(a,A^\star(x))^2\wedge R^2
\right)
\pi(da\mid x)
\geq
\frac{c_{M,R}}{K}
\]
uniformly in \(x\). Substituting,
\[
V^\star(\rho)-V^\pi(\rho)
\geq
\lambda_{M,R}
\int_{\mathcal X}
\frac{c_{M,R}}{K}\rho(dx)
=
\frac{c_{M,R}\lambda_{M,R}}{K}.
\]
The bound is uniform over \(T\geq K^{-1}\), so
\[
    \inf_{T\geq K^{-1}}
    \left(
        V^\star(\rho)-V^{\pi_{K,T}}(\rho)
    \right)
    \geq
    \frac{c_{M,R}\lambda_{M,R}}{K}.
\]
Consequently, if \(t_0(K)\geq K^{-1}\) for all sufficiently large \(K\), then
for all sufficiently large \(K\),
\[
    \inf_{T\geq t_0(K)}
    \left(
        V^\star(\rho)-V^{\pi_{K,T}}(\rho)
    \right)
    \geq
    \frac{c_{M,R}\lambda_{M,R}}{K},
\]
which proves the claim.
\end{proof}

\subsection{Lower bound with state-dependent volatility}

We now prepare the lower-bound proof for state-dependent volatility. Since the density argument used in the state-independent case no longer applies here directly, we replace it with a local Krylov estimate that gives a uniform probability of reaching small balls while remaining in a controlled region.

\begin{lem}[Local Krylov lower estimate]\label{lem:krylov-lower-local}
Let $d\ge 2$ and fix $0<\lambda\le \Lambda<\infty$. Let
\[
X_u=\int_0^u \alpha_r\,dW_r+\int_0^u \beta_r\,dr,
\qquad 0\le u\le 1,
\]
where $X_0=0$, the coefficients are progressively measurable, and
\[
\lambda I_d \preceq \alpha_r\alpha_r^\top \preceq \Lambda I_d,
\qquad 0\le r\le 1.
\]
Define the finite-horizon exit time
$
\tau_2^X:=\inf\{u\in[0,1]:X_u\notin B_2\}$, with $\inf\varnothing:=\infty,
$
and the no-exit event
\[
\mathcal E_2^X:=\{X_u\in B_2\text{ for all }0\le u\le 1\}
=\{\tau_2^X>1\}.
\]
Assume that
$
\|\beta_r\|\le 2$ for $0\le r\le1$ on $\{r<\tau_2^X\}$.
Then there exist constants $N<\infty$ and $\nu>0$, depending only on
$d,\lambda,\Lambda$, such that for every $q\in B_{1}(0)$ and every
$0<\rho\le1$,
\[
\mathbb P\bigl(X_1\in B_{2\rho}(q),\ \mathcal E_2^X\bigr)
\ge N^{-1}\rho^\nu .
\]
\end{lem}

\begin{proof}
We reduce the claim to Theorem~2.1 of Krylov~\cite{Krylov2021Ldplus1}. First define an auxiliary process on $[0,\infty)$ by stopping and modifying
the coefficients after exit from $B_2$ and after time $1$. Set
\[
\widehat\beta_r
:=
\begin{cases}
\beta_r\,\mathbf 1_{\{r<\tau_2^X\}}, & 0\le r\le 1,\\
0, & r>1,
\end{cases}
\]
and
\[
\widehat\alpha_r
:=
\begin{cases}
\alpha_r, & 0\le r\le 1 \text{ and } r<\tau_2^X,\\
\sqrt{\lambda}\,I_d, & \text{otherwise}.
\end{cases}
\]
Define
\[
\widehat X_u
:=
\int_0^u \widehat\alpha_r\,dW_r
+
\int_0^u \widehat\beta_r\,dr,
\qquad u\ge0.
\]
Then $\widehat X$ agrees with $X$ up to the first exit of $X$ from $B_2$
on the time interval $[0,1]$. If
\[
\widehat\tau_2:=\inf\{u\ge0:\widehat X_u\notin B_2\},
\]
then, for every Borel set $A\subseteq\mathbb R^d$,
$
\{\widehat X_1\in A,\ \widehat\tau_2>1\}
=
\{X_1\in A,\ \mathcal E_2^X\}.
$ Moreover,
$
\lambda I_d
\preceq
\widehat\alpha_r\widehat\alpha_r^\top
\preceq
\Lambda I_d$, $r\ge0,
$
and
$
\|\widehat\beta_r\|
\le
2\,\mathbf 1_{\{0\le r\le1\}}\mathbf 1_{\{\widehat X_r\in B_2\}} .
$

Next reduce to symmetric volatility. Set
$
a_r:=\widehat\alpha_r\widehat\alpha_r^\top$ and $S_r:=a_r^{1/2}$.
Then $S_r$ satisfies
$
\delta_{0} I_{d}\preceq S_{r}\preceq \delta_{0}^{-1}I_{d}$ with $
\delta_0:=\min\{\sqrt{\lambda},\Lambda^{-1/2},1/2\}.
$
Furthermore, the process
\[
\widetilde W_u
:=
\int_0^u S_r^{-1}\widehat\alpha_r\,dW_r
\]
is a Brownian motion. By associativity of Ito integrals, $\widehat X$ is thus in Krylov's diffusion-type form with symmetric
volatility with eigenvalues in $[\delta_{0}, \delta_{0}^{-1}]$.

It remains to verify Krylov's drift domination condition for the particular
function
\[
h(t,z):=
2\,\mathbf 1_{\{0\le t\le1\}}\mathbf 1_{\{z\in B_2\}} .
\]
We choose the admissible exponents
$
p_0=q_0=d+1,
$
so that Krylov's mixed norm $L_{p_0,q_0}$ coincides with the usual
space-time $L^{d+1}$ norm. Then
$
\|\widehat\beta_t\|\le h(t,\widehat X_t).
$ For a parabolic cylinder
$
C_\ell(t_0,z_0):=[t_0,t_0+\ell^2)\times B_\ell(z_0),
$
we have
\[
\begin{aligned}
\|h\|_{L^{d+1}(C_\ell(t_0,z_0))}^{d+1}
&=
\int_{C_\ell(t_0,z_0)}
2^{d+1}
\mathbf 1_{\{0\le t\le1\}}\mathbf 1_{\{z\in B_2\}}
\,dz\,dt  \\
&=
2^{d+1}
\left|[t_0,t_0+\ell^2)\cap[0,1]\right|
\left|B_\ell(z_0)\cap B_2\right|  \\
&\le
2^{d+1}v_d(\ell^2\wedge1)(\ell^d\wedge 2^d).
\end{aligned}
\]
where $v_d:=\lambda_d(B_1(0))=\pi^{d/2}/\Gamma(d/2+1)$ denotes the
Lebesgue volume of the unit ball in $\mathbb R^d$. If $0<\ell\le1$, this is bounded by $C_d\ell^{d+2}$; if
$\ell\ge1$, it is bounded by $C_d\le C_d\ell$. Hence, for all
$\ell>0$,
\[
\|h\|_{L^{d+1}(C_\ell(t_0,z_0))}^{d+1}
\le
C_d(\ell^{d+1}\wedge1)\ell .
\]
Thus Krylov's Assumption~2.1 holds with
$
\bar b_\ell:=C_d(\ell^{d+1}\wedge1).
$ The function $\bar b_\ell$ is bounded, nondecreasing, and satisfies
$\bar b_\ell\downarrow0$ as $\ell\downarrow0$. Therefore Krylov's
Assumption~2.2 holds at some sufficiently small scale
$R_\ast>0$ such that
$
\bar N\,\bar b_{R_\ast}<1.
$ depending only on $d$ and $\delta_0$.
In Krylov's notation, fix $\overline R=2$.

Now apply Krylov's Theorem 2.1 with application radius
$
\mathsf R=2$, $\kappa=\frac12$,
$\eta=\frac14$, and $t=1$.
The time condition is
$
\eta\mathsf R^2=1\le t=1\le \eta^{-1}\mathsf R^2=16.
$
Since $0,q\in B_1(0)$ with $\kappa {\mathsf R}=1$, Krylov's estimate gives
\[
\mathbb P\bigl(
\widehat X_1\in B_{\rho\mathsf R}(q),\
\widehat\tau_2>1
\bigr)
\ge N^{-1}\rho^\nu .
\]
Because $\rho\mathsf R=2\rho$ and
$
\{\widehat X_1\in B_{2\rho}(q),\widehat\tau_2>1\}
=
\{X_1\in B_{2\rho}(q),\mathcal E_2^X\},
$ we obtain
\[
\mathbb P\bigl(X_1\in B_{2\rho}(q),\ \mathcal E_2^X\bigr)
\ge N^{-1}\rho^\nu .
\]
All constants depend only on $d,\lambda,\Lambda$. Hence we are done.
\end{proof}
\subsection{Proof of Proposition \ref{prop:lb_diffusion}, assumption (a): state-dependent volatility}\label{sec:proof_prop_34_second}

We now complete the proof of Proposition \ref{prop:lb_diffusion} under the full state-dependent volatility assumption. The proof follows the same short-time vs. final-window decomposition as in the state-independent case, but the final concentration step is obtained from the local Krylov of Lemma \ref{lem:krylov-lower-local}.

\begin{proof}
All constants below may depend on $m,d,\kappa,\Lambda,\mu_0,R$, but not on
$K,f,\sigma,t$, nor on the locations or pairwise separations of the points
in $\mathcal{Z}$.
\paragraph{Notational note.} Throughout this proof, we will suppress the state argument in the drift $f(t, a, x)$ and $\sigma(t, a, x)$, so that they become $f(t, a)$ and $\sigma(t, a)$, respectively. This is because we do not use the state $x$ in our analysis. Therefore, we clarify any indication of variables $x, z,\cdots \in \mathbb{R}^{d}$ in this section are actions, not states ($d$ is the action dimension).\\

\textbf{Step 1: the case $t\le K^{-1}$.}
For $\delta>0$, define
$$
U_\delta:=\bigcup_{j=1}^m B(z_j,\delta).
$$
Let $\phi_t(x)$ be the deterministic flow generated by $f$:
$$
\frac{d}{dt}\phi_t(x)=f(t,\phi_t(x)),
\qquad
\phi_0(x)=x.
$$
Since $f(t,\cdot)$ is $K$-Lipschitz, for any $x,z\in\mathbb R^d$,
$$
\frac{d}{dt}|\phi_t(x)-\phi_t(z)|^2
\ge
-2K|\phi_t(x)-\phi_t(z)|^2.
$$
Therefore
$$
|\phi_t(x)-\phi_t(z)|\ge e^{-Kt}|x-z|.
$$
Thus $\phi_t$ is injective, and $\phi_t^{-1}$ is $e^{Kt}$-Lipschitz on
$\phi_t(\mathbb R^d)$. Hence, for every Borel set $E\subset\mathbb R^d$
and every $t\le K^{-1}$,
$$
\lambda_d(\phi_t^{-1}(E))
\le
e^{dKt}\lambda_d(E)
\le
e^d\lambda_d(E).
$$

Because $\mu_0\ll\lambda_d$, there exists $\alpha>0$ such that
$$
\lambda_d(E)\le \alpha
\quad\Longrightarrow\quad
\mu_0(E)\le \frac14.
$$
Choose $\delta>0$, depending only on $m,d,\mu_0$, such that
$
e^d\lambda_d(U_{2\delta})\le \alpha .
$
This is possible because
$$
\lambda_d(U_{2\delta})
\le
m v_d(2\delta)^d .
$$
Then, for every $t\le K^{-1}$,
$$
\mathbb P\{\phi_t(\bar a_0)\in U_{2\delta}\}
=
\mu_0(\phi_t^{-1}(U_{2\delta}))
\le
\frac14.
$$

Couple the diffusion and deterministic flow from the same initial point and set
$$
D_t:=\bar a_t-\phi_t(\bar a_0).
$$
Then
$$
dD_t
=
\{f(t,\bar a_t)-f(t,\phi_t(\bar a_0))\}\,dt
+
\sigma(t,\bar a_t)\,dB_t.
$$
By Itô's formula, the Lipschitz bound on $f$, and the upper volatility bound,
$$
\frac{d}{dt}\mathbb E\|D_t\|^2
\le
2K\mathbb E\|D_t\|^2+d\Lambda.
$$
Since $D_0=0$, Gronwall's inequality gives, for $t\le K^{-1}$,
$$
\mathbb E\|D_t\|^2
\le
\frac{d\Lambda}{2K}(e^{2Kt}-1)
\le
\frac{d\Lambda}{2K}(e^2-1).
$$
Thus, for all sufficiently large $K$,
$$
\mathbb P\{\|D_t\|\ge \delta\}
\le
\frac{\mathbb E\|D_t\|^2}{\delta^2}
\le
\frac14.
$$
If $\bar a_t\in U_\delta$ and $\|D_t\|<\delta$, then
$
\phi_t(\bar a_0)\in U_{2\delta}.
$
Hence
$
\mathbb P\{\bar a_t\in U_\delta\}
\le 1/4+1/4
=
1/2.
$
On $U_\delta^c$, we have
$$
\dist(x,\mathcal{Z})^2\wedge R^2
\ge
\delta^2\wedge R^2.
$$
Therefore, for every $t\le K^{-1}$ and all sufficiently large $K$,
$$
\mathbb E\left[\dist(\bar a_t, \mathcal{Z})^2\wedge R^2\right]
\ge
(\delta^2\wedge R^2)\,
\mathbb P\{\bar a_t\notin U_\delta\}
\ge
\frac{\delta^2\wedge R^2}{2}.
$$
In particular, for $K\ge 1$,
$$
\mathbb E\left[\dist(\bar a_t,\mathcal{Z})^2\wedge R^2\right]
\ge
\frac{\delta^2\wedge R^2}{2K}.
$$

\textbf{Step 2: the case $t>K^{-1}$.}
Fix $t>K^{-1}$ and set
$
s:=t-1/K.
$
We prove a lower bound conditionally on $\mathcal F_s$. Set
$
y:=\bar a_s.
$
Let $\phi_u$ solve
$$
\frac{d}{du}\phi_u=f(u,\phi_u),
\qquad
\phi_s=y.
$$
For $0\le \ell\le 1$, define
$
Y_\ell
:=
\sqrt K
\left(
\bar a_{s+\ell/K}-\phi_{s+\ell/K}
\right)
$ and 
$
W_\ell
:=
\sqrt K
\left(
B_{s+\ell/K}-B_s
\right)
$.
Conditionally on $\mathcal F_s$, $W$ is a standard Brownian motion. The
process $Y$ satisfies
$$
dY_\ell
=
\beta_K(\ell,Y_\ell)\,d\ell
+
\widetilde\sigma_K(\ell,Y_\ell)\,dW_\ell,
\qquad
Y_0=0,
$$
where
$$
\beta_K(\ell,u)
=
\frac{
f(s+\ell/K,\phi_{s+\ell/K}+u/\sqrt K)
-
f(s+\ell/K,\phi_{s+\ell/K})
}{\sqrt K}
$$
and
$$
\widetilde\sigma_K(\ell,u)
=
\sigma(s+\ell/K,\phi_{s+\ell/K}+u/\sqrt K).
$$
The $K$-Lipschitz property gives
$$
\|\beta_K(\ell,u)-\beta_K(\ell,v)\|
\le
\|u-v\|.
$$
Also $\beta_K(\ell,0)=0$. Hence, on $B_2$,
$
\|\beta_K(\ell,u)\|\le 2.
$
The rescaled covariance matrix satisfies, by Assumption \ref{ass:vol_bdd},
$$
\kappa I_d
\preceq
\widetilde\sigma_K(\ell,u)\widetilde\sigma_K(\ell,u)^\top
\preceq
\Lambda I_d .
$$

We first treat $d\ge 2$. Let
$$
\mathcal{E}_{2}^{Y}:=\{Y_{u}\in B_{2}\text{ for all }0\leq u\leq 1\}
$$
Construct a concatenated process $\widehat Y$ as follows. Up to time
$\tau_2$, set $\widehat Y_\ell=Y_\ell$. After $\tau_2$, continue with drift
zero and covariance matrix $\kappa I_d$. Then $\widehat Y$ agrees with $Y$ on
$\{\tau_2>1\}$, its covariance matrix remains between $\kappa I_d$ and
$\Lambda I_d$, and its drift satisfies
$$
\|\widehat \beta_\ell\|
\le
2\,\mathbf 1_{\{0\le \ell\le 1\}}\mathbf 1_{\{\widehat Y_\ell\in B_2\}} .
$$
Thus Lemma~\ref{lem:krylov-lower-local} applies conditionally on
$\mathcal F_s$, with ellipticity constants $\kappa,\Lambda$.

Consequently, there exist constants $N<\infty$ and $\nu>0$, depending only
on $d,\kappa,\Lambda$, such that for every $q\in B_1$ and every
$\rho\in(0,1]$,
$$
\mathbb P\left(
Y_1\in B_{2\rho}(q),
\,
\mathcal{E}_{2}^{Y}
\,\middle|\,
\mathcal F_s
\right)
\ge
N^{-1}\rho^\nu .
$$

Choose $m+1$ fixed points
$
q_0,\ldots,q_m\in B_1
$
with positive minimum separation
$
\Delta:=\min_{i\ne j}\|q_i-q_j\|>0.
$
Set
$
r_0:=\frac{\Delta}{8}.
$
Since $q_i\in B_1$, we may choose the points so that $r_0\le 1/4$.

Let $y_1,\ldots,y_m\in\mathbb R^d$ be arbitrary. We claim that there exists
$i\in\{0,\ldots,m\}$ such that
$$
\operatorname{dist}(q_i,\{y_1,\ldots,y_m\})>2r_0.
$$
If not, then each $q_i$ is within distance $2r_0$ of at least one $z_j$.
Since there are $m+1$ points $q_i$ and only $m$ points $y_j$, two distinct
$q_i,q_{i'}$ must be within distance $2r_0$ of the same $y_j$. Then
$$
\|q_i-q_{i'}\|
\le
\|q_i-y_j\|+\|y_j-q_{i'}\|
\le
4r_0
=
\frac{\Delta}{2},
$$
which contradicts the definition of $\Delta$. Thus such an $i$ exists.

For this $i$,
$$
B(q_i,r_0)
\cap
\bigcup_{j=1}^m B(y_j,r_0)
=
\varnothing .
$$
We apply the lower estimate with
$
q=q_i,
\rho=r_{0}/2.
$
Since $r_0\le 1/4$, $\rho\in(0,1]$, and
$$
B_{2\rho}(q_i)=B(q_i,r_0).
$$
Therefore
$$
\mathbb P\left(
Y_1\in B(q_i,r_0),
\,
\mathcal{E}_{2}^{Y}
\,\middle|\,
\mathcal F_s
\right)
\ge
N^{-1}\left(\frac{r_0}{2}\right)^\nu .
$$
Define
$$
c_0:=N^{-1}\left(\frac{r_0}{2}\right)^\nu>0.
$$
Since $B(q_i,r_0)$ is disjoint from the target union,
$$
\mathbb P\left(
Y_1\notin \bigcup_{j=1}^m B(y_j,r_0)
\,\middle|\,
\mathcal F_s
\right)
\ge
c_0 .
$$

Now set
$
y_j
:=
\sqrt K\left(z_j-\phi_{s+1/K}\right),
\qquad
j=1,\ldots,m.
$
The event
$$
\dist(\bar a_{s+1/K},\mathcal{Z})\le \frac{r_0}{\sqrt K}
$$
is exactly
$$
Y_1\in \bigcup_{j=1}^m B(y_j,r_0).
$$
Therefore
$$
\mathbb P\left(
\dist(\bar a_{s+1/K},\mathcal{Z})> \frac{r_0}{\sqrt K}
\,\middle|\,
\mathcal F_s
\right)
\ge
c_0.
$$
For all $K$ sufficiently large so that $r_0/\sqrt K\le R$, we have on this
event
$$
\dist(\bar a_{s+1/K}, \mathcal{Z})^2\wedge R^2
\ge
\frac{r_0^2}{K}.
$$
Consequently,
$$
\begin{aligned}
\mathbb E\left[
\dist(\bar a_{s+1/K},\mathcal{Z})^2\wedge R^2
\,\middle|\,
\mathcal F_s
\right]
&\ge
\frac{r_0^2}{K}
\mathbb P\left(
\dist(\bar a_{s+1/K},\mathcal{Z})> \frac{r_0}{\sqrt K}
\,\middle|\,
\mathcal F_s
\right) \\
&\ge
\frac{c_0r_0^2}{K}.
\end{aligned}
$$
Taking expectations gives, for every $t>K^{-1}$ and all sufficiently large
$K$,
$$
\mathbb E\left[\dist(\bar a_t, \mathcal{Z})^2\wedge R^2\right]
\ge
\frac{c_0r_0^2}{K}.
$$

It remains to handle the one-dimensional action case $d=1$. The local
Krylov estimate above was stated for dimensions at least two, so we embed
the scalar rescaled process into a two-dimensional one.

Work conditionally on $\mathcal F_s$. Enlarge the probability space, if
necessary, by adding a one-dimensional Brownian motion $G$, independent of
$\mathcal F_s$ and of the Brownian motion driving $Y$. Define the augmented
process
\[
\widehat Y_\ell := (Y_\ell,G_\ell)\in\mathbb R^2,
\qquad 0\le \ell\le 1.
\]
If
\[
dY_\ell=\beta_K(\ell,Y_\ell)d\ell
+\widetilde\sigma_K(\ell,Y_\ell)dW_\ell,
\]
then $\widehat Y$ satisfies
\[
d\widehat Y_\ell
=
\widehat\beta_K(\ell,\widehat Y_\ell)d\ell
+
\widehat\Sigma_K(\ell,\widehat Y_\ell)d\widehat W_\ell,
\]
where
\[
\widehat W_\ell:=(W_\ell,G_\ell),
\qquad
\widehat\beta_K(\ell,u,v):=(\beta_K(\ell,u),0),
\]
and
\[
\widehat\Sigma_K(\ell,u,v)
:=
\begin{pmatrix}
\widetilde\sigma_K(\ell,u) & 0\\
0 & 1
\end{pmatrix}.
\]
Therefore
\[
\widehat\Sigma_K\widehat\Sigma_K^\top
=
\begin{pmatrix}
\widetilde\sigma_K(\ell,u)^2 & 0\\
0 & 1
\end{pmatrix},
\]
and Assumption \ref{ass:vol_bdd} implies
\[
(\kappa\wedge 1)I_2
\preceq
\widehat\Sigma_K\widehat\Sigma_K^\top
\preceq
(\Lambda\vee 1)I_2 .
\]
Moreover, on the two-dimensional ball $B_2^{(2)}:=\{(u,v):u^2+v^2<4\}$,
we have $|u|\le 2$. Since $\beta_K(\ell,\cdot)$ is $1$-Lipschitz and
$\beta_K(\ell,0)=0$,
\[
\|\widehat\beta_K(\ell,u,v)\|
=
|\beta_K(\ell,u)|
\le |u|
\le 2
\qquad\text{on }B_2^{(2)}.
\]
Thus the two-dimensional local Krylov lower estimate applies to
$\widehat Y$, with constants depending only on $\kappa$ and $\Lambda$.

We now perform the pigeonhole argument in the first coordinate. Choose
$m+1$ deterministic points $q_0,\ldots,q_m\in(-1,1)$ with pairwise distance
at least $\Delta_m>0$, and set
\[
r_0:=\Delta_m/8\wedge 1/4.
\]
Let $\zeta_1,\ldots,\zeta_m\in\mathbb R$ be arbitrary. Similar to before, since there are only $m$ target points but $m+1$ candidate
points, there exists an index $i$ such that
\[
|q_i-\zeta_j|\ge 2r_0
\qquad\text{for all }j=1,\ldots,m.
\]
Equivalently,
\[
(q_i-r_0,q_i+r_0)
\cap
\bigcup_{j=1}^m(\zeta_j-r_0,\zeta_j+r_0)
=
\varnothing .
\]

Apply the two-dimensional Krylov estimate to the ball in $\mathbb R^2$
centered at $(q_i,0)$ with radius $r_0$. Since the estimate gives a ball of
radius $2\rho$, take $\rho=r_0/2$. Then, conditionally on $\mathcal F_s$,
\[
\mathbb P\!\left(
\widehat Y_1\in B_{r_0}^{(2)}((q_i,0)),\
\widehat{\mathcal E}_2
\,\middle|\,\mathcal F_s
\right)
\ge
N^{-1}(r_0/2)^\nu
=:c_0,
\]
where
\[
\widehat{\mathcal E}_2
:=
\{\widehat Y_\ell\in B_2^{(2)}\text{ for all }0\le \ell\le 1\}.
\]
Here the index $i$ may be chosen $\mathcal F_s$-measurably; the bound is
uniform over the finitely many possible choices of $i$.

On the event
$
\widehat Y_1\in B_{r_0}^{(2)}((q_i,0)),
$
the first coordinate satisfies
$
Y_1\in(q_i-r_0,q_i+r_0).
$
By the choice of $q_i$, this interval is disjoint from
\[
\bigcup_{j=1}^m(\zeta_j-r_0,\zeta_j+r_0).
\]
Hence
$
\operatorname{dist}(Y_1,\{\zeta_1,\ldots,\zeta_m\})\ge r_0
$
on this event. Therefore
\[
\mathbb P\!\left(
\operatorname{dist}(Y_1,\{\zeta_1,\ldots,\zeta_m\})\ge r_0
\,\middle|\,\mathcal F_s
\right)
\ge c_0 .
\]

Finally, take
$
\zeta_j:=\sqrt K\,(z_j-\phi_T), j=1,\ldots,m.
$
Since
\[
\bar a_T=\phi_T+\frac{Y_1}{\sqrt K},
\]
we have the exact identity
\[
\operatorname{dist}(\bar a_T,\mathcal{Z})
=
\frac{1}{\sqrt K}
\operatorname{dist}(Y_1,\{\zeta_1,\ldots,\zeta_m\}).
\]
Thus
\[
\mathbb P\!\left(
\operatorname{dist}(\bar a_T, \mathcal{Z})\ge \frac{r_0}{\sqrt K}
\,\middle|\,\mathcal F_s
\right)
\ge c_0 .
\]
For all sufficiently large $K$ such that $r_0/\sqrt K\le R$, it follows that
\[
\mathbb E\!\left[
\operatorname{dist}(\bar a_T, \mathcal{Z})^2\wedge R^2
\,\middle|\,\mathcal F_s
\right]
\ge
\frac{r_0^2}{K}
\mathbb P\!\left(
\operatorname{dist}(\bar a_T,\mathcal{Z})\ge \frac{r_0}{\sqrt K}
\,\middle|\,\mathcal F_s
\right)
\ge
\frac{c_0r_0^2}{K}.
\]
This proves the same clipped lower bound in dimension one. The constants
depend only on $m,\kappa,\Lambda,R$ and, when combined with the short-time
case, on $\mu_0$, but not on the locations or separations of the target
points.

Combining the two regimes, for all sufficiently large $K$,
$$
\mathbb E\left[\dist(\bar a_t, \mathcal{Z})^2\wedge R^2\right]
\ge
\frac{1}{K}
\min\left\{
\frac{\delta^2\wedge R^2}{2},
c_0r_0^2
\right\}.
$$
Taking the infimum over all admissible $f,\sigma,t$ proves the result.
\end{proof}

\subsection{Lower bound failure without upper volatility control}

Finally, we show that the upper bound $\sigma(t,x)\sigma(t,x)^\top \preceq \Lambda I_d$ in Assumption \ref{ass:vol_bdd} is not merely technical for the lower-bound argument. Without this bound, state-dependent volatility can accelerate concentration near the target set and violate the \(K^{-1}\) lower bound.

\begin{prop}
\label{prop:unbounded_sigma_counterexample}
The upper bound $\sigma(t,x)\sigma(t,x)^\top\preceq \Lambda I_d$ in Proposition~\ref{prop:lb_diffusion} cannot be dropped. In dimension $d=1$, let
$$
A=\{-1,1\}.
$$
For every $K\ge 1$, there exists a conservative weak solution of
$$
dX_t=\sigma_K(X_t)dB_t,
\qquad
X_0\sim \mathcal N(0,1),
$$
such that the drift is identically zero, hence $K$-Lipschitz, and
$$
\sigma_K(x)^2\ge 1
\qquad
\text{for all }x\in\mathbb R,
$$
but
$$
\inf_{t\ge 0}\E[d(X_t,A)^2]\lesssim \frac{1}{K^2}.
$$
Consequently, lower ellipticity alone does not imply the lower bound
$$
\inf_t \E[d(X_t,A)^2]\gtrsim \frac1K.
$$
\end{prop}

The proof follows the classical
speed-measure theory for one-dimensional diffusions. We use the standard
Brownian time-change construction; see 
\cite[Ch.~5, \S5.5A]{KaratzasShreve1991}. We also use the fact that a
regular one-dimensional diffusion in natural scale with finite speed measure
has invariant probability measure given by the normalized speed measure and is
ergodic; see  \cite[Ch.~V, \S\S47,53]{RogersWilliams2000}.

\begin{proof}
Let
$$
r(x):=d(x,A),
\qquad
\varepsilon_K:=K^{-1},
$$
and define
$$
q_K(x):=\frac{1}{1+(r(x)/\varepsilon_K)^6}.
$$
Set
$$
\sigma_K(x)^2:=\frac{1}{q_K(x)}
=
1+\left(\frac{r(x)}{\varepsilon_K}\right)^6.
$$
Then $\sigma_K(x)^2\ge 1$ for all $x$, but $\sup_x \sigma_K(x)=\infty$.

We construct the diffusion by a time change of Brownian motion. This is a standard one-dimensional time-change construction; see
\cite[Ch.~5, \S5.5A]{KaratzasShreve1991} for the
method of time change and see \cite[Ch.~V, \S47]{RogersWilliams2000}
for the speed-measure/time-substitution formulation. Specifically, let $W$ be a one-dimensional Brownian motion with $W_0\sim \mathcal N(0,1)$, and define
$$
S_u:=\int_0^u q_K(W_s)\,ds.
$$
Since $q_K$ is strictly positive on neighborhoods of $-1$ and $1$, and Brownian motion spends infinite occupation time in every nonempty open interval,
$$
S_u\to\infty
\qquad
\text{a.s. as }u\to\infty.
$$
Hence the inverse clock
$
\tau_t:=\inf\{u:S_u>t\}
$
is finite for every finite $t$. Define
$
X_t:=W_{\tau_t}.
$
Then $X_t$ is finite for every finite $t$, so the process is conservative. Moreover,
$$
d\langle X\rangle_t
=
d\tau_t
=
\frac{1}{q_K(X_t)}dt
=
\sigma_K(X_t)^2dt.
$$
Therefore, by the martingale representation theorem, $X$ is a weak solution of $dX_t=\sigma_K(X_t)dB_t$. 

Now define \(Z_K:=\int_{\mathbb R}q_K(x)\,dx\) and
\(\pi_K(dx):=Z_K^{-1}q_K(x)\,dx\). Since
\(q_K(x)\asymp \varepsilon_K^6|x|^{-6}\) as \(|x|\to\infty\), \(Z_K<\infty\).
The time-changed process is a regular diffusion in natural scale with speed
measure \(m_K(dx)=2q_K(x)\,dx\). Thus \(\pi_K\) is the normalized speed
measure. Since \(m_K(\mathbb R)<\infty\), the diffusion is positive recurrent
and ergodic; see Rogers--Williams~\cite[Ch.~V, \S\S47,53]{RogersWilliams2000}. Let \(P_t h(x):=\mathbb E_x[h(X_t)]\) denote the Markov semigroup of \(X\).
By ergodicity with invariant law \(\pi_K\),
\[
P_t h\to \int h\,d\pi_K
\qquad\text{in }L^2(\pi_K),\qquad h\in L^2(\pi_K).
\]

Let $\gamma$ be the $\mathcal N(0,1)$ law and let $\varphi$ be its density. Since $\pi_K$ has positive density everywhere,
$$
h_K:=\frac{d\gamma}{d\pi_K}
=
\frac{Z_K\varphi}{q_K}.
$$
Because $q_K^{-1}$ grows only polynomially while $\varphi$ decays Gaussianly, we have $h_K\in L^2(\pi_K)$. Also
$$
f(x):=d(x,A)^2=r(x)^2
$$
belongs to $L^2(\pi_K)$, since in the tails
$$
r(x)^4q_K(x)\asymp \varepsilon_K^6 |x|^{-2},
$$
which is integrable. Therefore, by symmetry and $L^2(\pi_K)$ ergodicity,
$$
\E_{\gamma}[f(X_t)]
=
\langle h_K,P_tf\rangle_{L^2(\pi_K)}
=
\langle P_th_K,f\rangle_{L^2(\pi_K)}
\to
\int f\,d\pi_K.
$$
Thus the invariant expectation is the actual limit of finite-time expectations from the required Gaussian initialization.

It remains to estimate $\int f\,d\pi_K$. Since $r(x)=d(x,A)$ has level sets of uniformly bounded multiplicity on $\mathbb R$,
$$
Z_K
=
\int_{\mathbb R}\frac{dx}{1+(r(x)/\varepsilon_K)^6}
\asymp
\varepsilon_K,
$$
and
$$
\int_{\mathbb R}r(x)^2q_K(x)\,dx
\lesssim
\int_0^\infty \frac{s^2}{1+(s/\varepsilon_K)^6}\,ds
\asymp
\varepsilon_K^3.
$$
Therefore
$$
\int d(x,A)^2\,\pi_K(dx)
=
\frac{\int r(x)^2q_K(x)\,dx}{Z_K}
\lesssim
\varepsilon_K^2
=
\frac1{K^2}.
$$
Since $\E_{\gamma}[d(X_t,A)^2]\to \int d(x,A)^2\,\pi_K(dx)$, there exists a finite deterministic time $t_K$ such that $\E_{\gamma}[d(X_{t_K},A)^2]\lesssim \frac1{K^2}$.
Hence
$$
\inf_{t\ge 0}\E[d(X_t,A)^2]\lesssim \frac1{K^2},
$$
which proves the claim.
\end{proof}

\section{Proofs for Section~\ref{sec:statistical-estimation}}
\label{app:statistical-generalization}

\subsection{Auxiliary  facts about ReLU network class}
\label{app:relu}

Here we present some  elementary properties  of  \(\mathcal F_{K,s}\), which will be   used in the proof of Theorem~\ref{thm:lipschitz-selector-complete-tradeoff}.

\begin{lem}[Properties of the ReLU drift class]
\label{lem:basic-relu-class}
Every \(f\in\mathcal F_{K,s}\) satisfies, for all $t\in [0,T]$, 
$a,a'\in \sR^d$ and $x,x'\in [0,1]^m$,
\begin{eqnarray}
 \|f(t,0,x)\|&\le& \ell_1K,\label{eq:zero-bound-section4}\\
 \|f(t,a,x)-f(t,a',x)\|&\le& K\|a-a'\|,\label{eq:action-lip-section4}\\
 \|f(t,a,x)-f(t,a,x')\|&\le& \ell_2K\|x-x'\|.
\label{eq:state-lip-section4}
\end{eqnarray}
Consequently, \(\|f(t,a,x)\|\le K(\ell_1+\|a\|)\). Under Assumptions~\ref{ass:vol_bdd} and \ref{ass:lipschitsz}, the state-conditioned action SDE has a unique strong solution for every \(f\in\mathcal F_{K,s}\) and \(x\in[0,1]^m\).
\end{lem}

\begin{proof}
Note that \eqref{eq:zero-bound-section4} is part of the definition. For the action Lipschitz bound, fix \(t,x,a,a'\), set \(v=a-a'\), and write \(a_r=a'+rv\). Since \(\Psi_\theta\) is continuous and piecewise affine, \(r\mapsto\Psi_\theta(t,a_r,x)\) is absolutely continuous and
\[
 \Psi_\theta(t,a,x)-\Psi_\theta(t,a',x)
 =\int_0^1 J_a\Psi_\theta(t,a_r,x)v\,dr,
\]
where \(J_a\Psi_\theta\) is the action-slope matrix on the relevant affine region for a.e. \(r\). The bound \(\|M^2_{\theta,S}\|_{\mathrm{op}}\le1\) gives \eqref{eq:action-lip-section4}. The proof of \eqref{eq:state-lip-section4} is identical along the line segment between \(x\) and \(x'\), using \(\|M^3_{\theta,S}\|_{\mathrm{op}}\le\ell_2\). The linear-growth bound follows from \eqref{eq:zero-bound-section4} and \eqref{eq:action-lip-section4}. Finally, the linear-growth and global Lipschitz bounds on \(f\), together with the boundedness and Lipschitz assumptions on \(\sigma\), give the stated strong well-posedness.
\end{proof}

\begin{proof}[Proof of Lemma~\ref{lem:relu-selector-approx}]
The existence and approximation properties of \(h_s\) follow from the standard finite-element/ReLU approximation result for Lipschitz maps on the cube, obtained by nodal interpolation on a Kuhn triangulation and exact ReLU realization of continuous piecewise-affine finite-element functions; see \cite{yarotsky2017error,petersen2018optimal}. Applying the scalar construction coordinate by coordinate and summing the squared coordinate errors gives \eqref{eq:h-error-section4}. The bounds \eqref{eq:h-bound-section4}--\eqref{eq:h-lip-section4} follow from the boundedness and Lipschitz properties of the clipped finite-element interpolant.

It remains to verify that the associated mean-reverting reference drift belongs to \(\mathcal F_{K,s}\). The map \(x\mapsto h_s(x)\) is a ReLU network. The affine map \(a\mapsto-a\) is exactly representable by ReLUs since
\[
 -a_j={\rm ReLU}(-a_j)-{\rm ReLU}(a_j),\qquad j=1,\ldots,d.
\]
Thus \eqref{eq:comp-section4} is a ReLU network in \((t,a,x)\), with no dependence on \(t\). The additional action-identity branch and final affine readout increase the size, depth, width, and weight bounds by numerical factors. The factor \((1-e^{-T})^{-1}\) in \eqref{eq:p-s-4-section4} covers \((1-e^{-KT})^{-1}\) because \(K\ge1\), and the constant \(C_0\) in the architecture bounds is chosen large enough to absorb these fixed realization overheads.

The action slope is \(-I_d\), so its operator norm is \(1\). By \eqref{eq:h-lip-section4},
\[
 \mathrm{Lip}\left(\frac{h_s}{1-e^{-KT}}\right)
 \le \frac{C_0^m\sqrt d\,L_{g^\star}}{1-e^{-T}}
 \le \ell_2.
\]
Also, by \eqref{eq:h-bound-section4},
\[
 \sup_{t,x}\left\|\frac{h_s(x)}{1-e^{-KT}}\right\|
 \le \frac{\sqrt d\,R_0}{1-e^{-T}}
 \le \ell_1.
\]
Hence \(f_{K,s}\in\mathcal F_{K,s}\).
\end{proof}

\subsection{Proof of Theorem~\ref{thm:lipschitz-selector-complete-tradeoff}}
\label{app:proof_stats}

\label{app:proof-lipschitz-selector-complete-tradeoff}

We organize the proof into five parts. First, we reduce empirical maximization
to a comparison inequality and bound the comparison term using the
mean-reverting reference drift. The next three parts establish the localization,
covering, and finite-horizon stability estimates. The final part converts these
estimates into a uniform-convergence bound and substitutes it into the
comparison inequality.

\medskip
\noindent\textbf{Step 1: Empirical reduction and reference-drift approximation.}
For \(f\in\mathcal F_{K,s}\), write
\[
G_f(\xi):=\sum_{j=0}^{\infty}\gamma^j r(X_j^f,A_j^f),
\qquad
PG_f:=\mathbb E G_f(\xi)=V(f),
\qquad
P_nG_f:=\frac1n\sum_{i=1}^nG_f(\xi_i)=\widehat V_n(f).
\]
For every \(f\), \(|G_f|\le r_{\max}/(1-\gamma)\). Let \(f_0\in\mathcal F_{K,s}\) be arbitrary. Since \(\widehat f_{n,K,s}\) is an exact empirical maximizer,
\[
P_nG_{\widehat f_{n,K,s}}\ge P_nG_{f_0}.
\]
Therefore,
\begin{equation}
\label{eq:erm-section4}
 V^\star(\rho)-V(\widehat f_{n,K,s})
 \le
 V^\star(\rho)-V(f_0)
 +2\sup_{f\in\mathcal F_{K,s}}\left|(P_n-P)G_f\right|.
\end{equation}
If an \(\eta\)-maximizer is used instead, the right-hand side gains \(\eta\).

\medskip

We next record the performance identity used to control the comparison term. For every Markov policy \(\pi\) and every \(x\in\mathcal X\),
\begin{equation}
\label{eq:perf-section4}
 V^\star(x)-V^\pi(x)
 =\mathbb E_x^\pi\sum_{j=0}^\infty
 \gamma^j\{V^\star(X_j)-Q^\star(X_j,A_j)\}.
\end{equation}
Indeed,
\[
 \mathbb E_x^\pi\{V^\star(X_0)-Q^\star(X_0,A_0)\}
 =V^\star(x)-\mathbb E_x^\pi r(X_0,A_0)-\gamma \mathbb E_x^\pi V^\star(X_1),
\]
and
\[
 V^\pi(x)=\mathbb E_x^\pi r(X_0,A_0)+\gamma \mathbb E_x^\pi V^\pi(X_1).
\]
Adding the two displays yields the one-step identity
\[
 V^\star(x)-V^\pi(x)
 =\mathbb E_x^\pi\{V^\star(X_0)-Q^\star(X_0,A_0)\}
 +\gamma \mathbb E_x^\pi\{V^\star(X_1)-V^\pi(X_1)\}.
\]
Iteration for \(H\) steps gives
\[
 V^\star(x)-V^\pi(x)
 =\mathbb E_x^\pi\sum_{j=0}^{H-1}
 \gamma^j\{V^\star(X_j)-Q^\star(X_j,A_j)\}
 +\gamma^H\mathbb E_x^\pi\{V^\star(X_H)-V^\pi(X_H)\}.
\]
Since \(|V^\star|\vee |V^\pi|\le r_{\max}/(1-\gamma)\), the remainder tends to zero as \(H\to\infty\), proving \eqref{eq:perf-section4}.

\medskip

We now apply the above identity to the mean-reverting reference drift. Let \(h_s\) and \(f_{K,s}\) be as in \eqref{eq:comp-section4} and Lemma~\ref{lem:relu-selector-approx}, and set \(f_0=f_{K,s}\). By Lemma~\ref{lem:relu-selector-approx}, \(f_0\in\mathcal F_{K,s}\). Fix \(x\in[0,1]^m\). Let \(\bar a_t^x\) denote the solution of the action diffusion driven by \(f_0\) at state \(x\). Define
\[
 \mathfrak m_t(x):=\frac{1-e^{-Kt}}{1-e^{-KT}}h_s(x),
 \qquad 0\le t\le T,
\]
and set \(Y_t:=\bar a_t^x-\mathfrak m_t(x)\). Since \(\mathfrak m_T(x)=h_s(x)\), direct differentiation gives
\[
 dY_t=-KY_t\,dt+\sigma(t,\bar a_t^x,x)dB_t.
\]
It\^o's formula and Assumption~\ref{ass:vol_bdd} imply
\[
 \frac{d}{dt}\mathbb E\|Y_t\|^2
 \le -2K\mathbb E\|Y_t\|^2+d\Lambda,
 \qquad
 \mathbb E\|Y_0\|^2=m_2(\mu_0).
\]
Gronwall's inequality yields
\begin{equation}
\label{eq:oracle-mse-section4}
 \sup_{x\in[0,1]^m}\mathbb E\|\bar a_T^x-h_s(x)\|^2
 \le e^{-2KT}m_2(\mu_0)+\frac{d\Lambda}{2K}(1-e^{-2KT}).
\end{equation}
Combining \eqref{eq:oracle-mse-section4} with Lemma~\ref{lem:relu-selector-approx},
\begin{equation}
\label{eq:oracle-mse-2-section4}
 \sup_{x\in[0,1]^m}\mathbb E\|\bar a_T^x-g^\star(x)\|^2
 \le C\left(e^{-2KT}m_2(\mu_0)+\frac{d\Lambda}{K}+C_0^mmdL_{g^\star}^2s^{-2/m}\right).
\end{equation}
By \eqref{eq:perf-section4} and \eqref{eq:quadratic-upper-gap-section4},
\begin{equation}
\label{eq:oracle-section4}
 V^\star(\rho)-V(f_0)
 \le
 C\left(e^{-2KT}m_2(\mu_0)+\frac{d\Lambda}{K}+C_0^mmdL_{g^\star}^2s^{-2/m}\right),
\end{equation}
where \(C\) depends only on the fixed problem constants in Theorem~\ref{thm:lipschitz-selector-complete-tradeoff}.

\medskip
\noindent\textbf{Step 2: Path localization under second moments.}
For the localization estimates, write
\[
D_{K,T}:=
\left(m_2(\mu_0)+K^2\ell_1^2T^2+d\Lambda T\right)^{1/2}.
\]
We first prove a path localization estimate.

\begin{lem}[Path localization]
\label{lem:loc-section4}
For every \(f\in\mathcal F_{K,s}\) and every \(x\in[0,1]^m\),
\begin{equation}
\label{eq:path-moment-section4}
 \mathbb E\left[\sup_{0\le t\le T}\|\bar a_t^{f,x}\|^2\right]
 \le C e^{C_3KT}D_{K,T}^2.
\end{equation}
Consequently, for every \(R\ge1\),
\begin{equation}
\label{eq:path-tail-section4}
 \mathbb P\left(\sup_{0\le t\le T}\|\bar a_t^{f,x}\|>R\right)
 \le
 \frac{C e^{C_3KT}D_{K,T}^2}{R^2}.
\end{equation}
\end{lem}

\begin{proof}
Fix \(f\in\mathcal F_{K,s}\) and \(x\in[0,1]^m\). By Lemma~\ref{lem:basic-relu-class},
\[
 \|f(t,a,x)\|\le K(\ell_1+\|a\|).
\]
Let
\[
 S_t:=\sup_{0\le u\le t}\|\bar a_u^{f,x}\|,
 \qquad
 M_t:=\int_0^t\sigma(u,\bar a_u^{f,x},x)dB_u.
\]
The integral form of the SDE gives
\[
 S_T
 \le
 \|\bar a_0\|+K\ell_1T+K\int_0^T S_u\,du+
 \sup_{0\le u\le T}\|M_u\|.
\]
By Gronwall's inequality,
\[
 S_T
 \le
 e^{KT}\left(\|\bar a_0\|+K\ell_1T+
 \sup_{0\le u\le T}\|M_u\|\right).
\]
Therefore,
\[
 \mathbb E S_T^2
 \le
 C e^{C_3KT}
 \left(m_2(\mu_0)+K^2\ell_1^2T^2+
 \mathbb E\sup_{0\le u\le T}\|M_u\|^2\right).
\]
Doob's \(L^2\) maximal inequality and It\^o's isometry give
\[
 \mathbb E\sup_{0\le u\le T}\|M_u\|^2
 \le
 4\mathbb E\|M_T\|^2
 \le
 4d\Lambda T.
\]
Thus \eqref{eq:path-moment-section4} follows from the definition of \(D_{K,T}\), and \eqref{eq:path-tail-section4} follows from Markov's inequality.
\end{proof}

\medskip
\noindent\textbf{Step 3: Localized covering of the drift class.}
For \(R\ge1\), define
\[
 Z_R:=[0,T]\times B_R^d\times[0,1]^m,
 \qquad
 d_{K,R}(f,g):=\sup_{z\in Z_R}\frac{\|f(z)-g(z)\|}{K},
\]
where \(B_R^d\) is the Euclidean ball in \(\mathbb R^d\). Also define
\begin{equation}
\label{eq:rho-section4}
 \rho_R:=1+T+R+\sqrt m.
\end{equation}

\begin{lem}[Localized covering number]
\label{lem:cover-section4}
For every \(R\ge1\) and every \(0<\varepsilon\le1\),
\begin{equation}
\label{eq:cover-section4}
 \log \mathcal N(\varepsilon,\mathcal F_{K,s},d_{K,R})
 \le
 p_{s,1}\log\left(
 1+\frac{2\rho_R(p_{s,1}p_{s,4})^{p_{s,2}+2}}{\varepsilon}
 \right).
\end{equation}
Moreover, after replacing \(\varepsilon\) by \(2\varepsilon\), the cover may be chosen to consist of elements of \(\mathcal F_{K,s}\).
\end{lem}

\begin{proof}
Consider the ambient parameter box corresponding to networks with at most \(p_{s,1}\) nonzero scalar parameters, depth at most \(p_{s,2}\), width at most \(p_{s,3}\), and parameter magnitudes bounded by \(p_{s,4}\). The standard layer-by-layer telescoping estimate for ReLU networks on bounded input sets gives, for any two parameter vectors \(\theta,\theta'\) in this ambient box,
\begin{equation}
\label{eq:param-lip-section4}
 \sup_{z\in Z_R}\|\Psi_\theta(z)-\Psi_{\theta'}(z)\|
 \le
 \rho_R(p_{s,1}p_{s,4})^{p_{s,2}+1}\|\theta-\theta'\|_\infty.
\end{equation}
This deterministic perturbation estimate is a standard ingredient in covering-number bounds for ReLU networks; see \cite{bartlett2019nearly,petersen2018optimal}.
Thus an \(\ell_\infty\) grid on \([-p_{s,4},p_{s,4}]^{p_{s,1}}\) with mesh size
\[
 \eta:=\varepsilon\{\rho_R(p_{s,1}p_{s,4})^{p_{s,2}+1}\}^{-1}
\]
 induces an external \(\varepsilon\)-cover of the ambient network class on \(Z_R\). Its cardinality is at most
\[
 \left(1+\frac{2p_{s,4}}{\eta}\right)^{p_{s,1}}
 \le
 \left(1+\frac{2\rho_R(p_{s,1}p_{s,4})^{p_{s,2}+2}}{\varepsilon}\right)^{p_{s,1}}.
\]
Since \(\mathcal F_{K,s}\) is a subclass of this ambient class, the same upper bound holds for the external covering number of \(\mathcal F_{K,s}\) under \(d_{K,R}\). Equivalently, it bounds the \(\varepsilon\)-packing number of \(\mathcal F_{K,s}\). A maximal \(2\varepsilon\)-separated subset of \(\mathcal F_{K,s}\) is a \(2\varepsilon\)-cover by elements of \(\mathcal F_{K,s}\) and has cardinality bounded by the preceding packing bound. This proves \eqref{eq:cover-section4}.
\end{proof}

\medskip
\noindent\textbf{Step 4: Finite-horizon stability.}
For the stability estimates, write
\[
L_{K,T}:=
1+(1+\ell_2)\sqrt{KT}\,e^{C_3KT},
\qquad
M_{K,H,T}:=
C_3H\{C_3L_{K,T}\}^{H}.
\]
For \(H\ge1\), define
\[
 G_H(f):=\sum_{j=0}^{H-1}\gamma^j r(X^f_j,A^f_j).
\]
Let \(\Pi_R:\mathbb R^d\to B_R^d\) denote Euclidean projection. For \(f\in\mathcal F_{K,s}\) define
\[
 f^R(t,a,x):=f(t,\Pi_R(a),x),
 \qquad
 \sigma^R(t,a,x):=\sigma(t,\Pi_R(a),x).
\]
Let \(\widetilde G_H^R(f)\) denote the \(H\)-step return generated by replacing, at every within-step action diffusion, \((f,\sigma)\) with \((f^R,\sigma^R)\).

\begin{lem}[Projected return stability]
\label{lem:proj-section4}
For all \(f,g\in\mathcal F_{K,s}\),
\begin{equation}
\label{eq:proj-section4}
 \|\widetilde G_H^R(f)-\widetilde G_H^R(g)\|_{L^2(P)}
 \le
 M_{K,H,T}\, d_{K,R}(f,g).
\end{equation}
\end{lem}

\begin{proof}
%Fix \(f,g\in\mathcal F_{K,s}\). Fix two initial states \(x,x'\in[0,1]^m\), and synchronously couple the projected action diffusions driven by \(f\) and \(g\) using the same Brownian motion and the same initial action. Let the two projected action processes be \(\widetilde a_t^f\) and \(\widetilde a_t^g\), and set \(Z_t:=\widetilde a_t^f-\widetilde a_t^g\). Then \(Z_0=0\).

Fix \(f,g\in\mathcal F_{K,s}\) and two states
\(x,x'\in[0,1]^m\). On the same probability space, let
\(\zeta\sim\mu_0\) and let \(B\) be a Brownian motion. Define
\(\widetilde a^{f,x}\) and \(\widetilde a^{g,x'}\) as the projected
within-step action diffusions
\begin{eqnarray*}
    d\widetilde a^{f,x}_t
&=&
f^R(t,\widetilde a^{f,x}_t,x)\,dt
+
\sigma^R(t,\widetilde a^{f,x}_t,x)\,dB_t,
\qquad
\widetilde a^{f,x}_0=\zeta,\\
d\widetilde a^{g,x'}_t
&=&
g^R(t,\widetilde a^{g,x'}_t,x')\,dt
+
\sigma^R(t,\widetilde a^{g,x'}_t,x')\,dB_t,
\qquad
\widetilde a^{g,x'}_0=\zeta.
\end{eqnarray*}
Set \(Z_t:=\widetilde a^{f,x}_t-\widetilde a^{g,x'}_t\). Then
\(Z_0=0\).

Using Lemma~\ref{lem:basic-relu-class}, the Lipschitz property of \(\Pi_R\), the definition of \(d_{K,R}\), and Assumption~\ref{ass:lipschitsz}, It\^o's formula gives
\[
 \frac{d}{dt}\mathbb E\|Z_t\|^2
 \le
 C_3K\mathbb E\|Z_t\|^2
 +C_3K(1+\ell_2)^2\|x-x'\|^2
 +C_3K d_{K,R}(f,g)^2.
\]
Since \(Z_0=0\), Gronwall's inequality gives
\[
 \left(\mathbb E\|Z_T\|^2\right)^{1/2}
 \le
 C_3(1+\ell_2)\sqrt{KT}\,e^{C_3KT}
 \left\{\|x-x'\|+d_{K,R}(f,g)\right\}.
\]
Let
\[
 D_j:=\left(\mathbb E\|\widetilde X_j^f-\widetilde X_j^g\|^2\right)^{1/2},
 \qquad
 E_j:=\left(\mathbb E\|\widetilde A_j^f-\widetilde A_j^g\|^2\right)^{1/2}.
\]
The preceding estimate implies
\[
 E_j\le C_3L_{K,T}\{D_j+d_{K,R}(f,g)\}.
\]
By the synchronous \(L^2\) representation of the transition kernel in Assumption~\ref{ass:lipschitsz},
\[
 D_{j+1}\le C_3L_{K,T}\{D_j+d_{K,R}(f,g)\}.
\]
Since \(D_0=0\), induction gives, for all \(0\le j\le H\),
\[
 D_j+E_j\le C_3\big(C_3L_{K,T}\big)^Hd_{K,R}(f,g).
\]
Using the Lipschitz continuity of \(r\) and summing over \(j=0,\ldots,H-1\),
\[
 \|\widetilde G_H^R(f)-\widetilde G_H^R(g)\|_{L^2(P)}
 \le C_3H\big(C_3L_{K,T}\big)^Hd_{K,R}(f,g),
\]
which is \eqref{eq:proj-section4}.
\end{proof}

\begin{lem}[Original return stability]
\label{lem:orig-section4}
For all \(H\ge1\), all \(R\ge1\), and all \(f,g\in\mathcal F_{K,s}\),
\begin{equation}
\label{eq:orig-section4}
 \|G_H(f)-G_H(g)\|_{L^2(P)}
 \le
 M_{K,H,T}d_{K,R}(f,g)
 +C\frac{r_{\max}}{1-\gamma}
 \frac{\sqrt H\,D_{K,T}e^{C_3KT}}{R}.
\end{equation}
\end{lem}

\begin{proof}
Fix \(f\in\mathcal F_{K,s}\). For one decision step, Lemma~\ref{lem:loc-section4} gives, uniformly over \(x\in[0,1]^m\),
\[
 \mathbb P\left(\sup_{0\le t\le T}\|\bar a_t^{f,x}\|>R\right)
 \le
 \frac{Ce^{C_3KT}D_{K,T}^2}{R^2}.
\]
For an \(H\)-step trajectory, let \(E_R^f(H)\) be the event that at least one of the \(H\) within-step action paths exits \(B_R^d\). By the union bound,
\[
 \mathbb P(E_R^f(H))\le \frac{CHe^{C_3KT}D_{K,T}^2}{R^2}.
\]
On the complement of \(E_R^f(H)\), the original trajectory and projected trajectory agree through time \(H\). Since both \(G_H(f)\) and \(\widetilde G_H^R(f)\) are bounded in absolute value by \(r_{\max}/(1-\gamma)\),
\[
 \|G_H(f)-\widetilde G_H^R(f)\|_{L^2(P)}
 \le
 C\frac{r_{\max}}{1-\gamma}
 \frac{\sqrt H\,D_{K,T}e^{C_3KT}}{R}.
\]
The same bound holds for \(g\). The triangle inequality and Lemma~\ref{lem:proj-section4} prove \eqref{eq:orig-section4}.
\end{proof}

\medskip
\noindent\textbf{Step 5: Entropy of truncated returns and uniform convergence.}
Let \(\mathcal G_H:=\{G_H(f):f\in\mathcal F_{K,s}\}\). Fix \(0<\varepsilon\le r_{\max}/(1-\gamma)\) and define
\begin{equation}
\label{eq:Reps-section4}
 R_{\varepsilon,H}:=
 \max\left\{1,\frac{C(r_{\max}/(1-\gamma))\sqrt H\,D_{K,T}e^{C_3KT}}{\varepsilon}\right\}.
\end{equation}
By Lemma~\ref{lem:orig-section4}, if
\[
 d_{K,R_{\varepsilon,H}}(f,g)
 \le
 \frac{\varepsilon}{2M_{K,H,T}},
\]
then \(\|G_H(f)-G_H(g)\|_{L^2(P)}\le\varepsilon\). Therefore, Lemma~\ref{lem:cover-section4} implies
\begin{align}
\label{eq:GHcover-section4}
 \log \mathcal N(\varepsilon,\mathcal G_H,L^2(P))
 &\le
 p_{s,1}\log\left(
 1+
 \frac{
 C\{1+T+\sqrt m+R_{\varepsilon,H}\}(p_{s,1}p_{s,4})^{p_{s,2}+2}M_{K,H,T}
 }{\varepsilon}
 \right).
\end{align}

We now take \(H=H_n\). If \(\gamma=0\), then \(G_f=G_{H_n}(f)\). If \(0<\gamma<1\), then, by $H_n
=
\max\left\{
1,
\left\lceil
\frac{\log(8\sqrt n)}{\log(1/\gamma)}
\right\rceil
\right\}$,
\[
 \gamma^{H_n}\le \frac1{8\sqrt n}.
\]
Thus, in all cases,
\begin{equation}
\label{eq:disc-trunc-section4}
 \sup_{f\in\mathcal F_{K,s}}|G_f-G_{H_n}(f)|
 \le
 \frac{r_{\max}}{1-\gamma}\frac1{8\sqrt n}.
\end{equation}

For \((r_{\max}/(1-\gamma))/n\le\varepsilon\le r_{\max}/(1-\gamma)\), \eqref{eq:GHcover-section4}, \eqref{eq:Reps-section4}, and \eqref{eq:Gamma-section4} imply
\begin{equation}
\label{eq:entropy-final-section4}
 \log \mathcal N(\varepsilon,\mathcal G_{H_n},L^2(P))
 \le
 p_{s,1}\Gamma_n(K,s,m,d,T)
 \log\left(\frac{er_{\max}}{(1-\gamma)\varepsilon}\right).
\end{equation}
The factor \(n^2\) in \eqref{eq:Gamma-section4} dominates the powers of \(\varepsilon^{-1}\) generated by \(R_{\varepsilon,H_n}\) and by the covering radius.

We use the following bounded entropy maximal inequality. If \(\mathcal G\) is separable, \(|g|\le r_0\) for all \(g\in\mathcal G\), and
\[
 \log \mathcal N(\varepsilon,\mathcal G,L^2(P))\le E\log(e\,c/\varepsilon),
 \qquad r_0/n\le\varepsilon\le r_0,
\]
then, with probability at least \(1-\delta\),
\begin{equation}
\label{eq:emp-section4}
 \sup_{g\in\mathcal G}|(P_n-P)g|
 \le
 C r_0\left\{
 \sqrt{\frac{E+\log(2/\delta)}{n}}
 +\frac{E+\log(2/\delta)}{n}
 \right\}.
\end{equation}
This follows from symmetrization and Dudley's entropy integral bound
\cite[Theorem~5.22]{wainwright2019high}, together with the standard
bounded-difference concentration inequality for suprema of bounded
empirical processes.

Applying \eqref{eq:emp-section4} to \(\mathcal G_{H_n}\) with \(r_0=r_{\max}/(1-\gamma)\) and \(E=p_{s,1}\Gamma_n(K,s,m,d,T)\), we obtain, with probability at least \(1-\delta\),
\begin{align}
\label{eq:uc-trunc-section4}
 \sup_{f\in\mathcal F_{K,s}}|(P_n-P)G_{H_n}(f)|
 &\le
 C\frac{r_{\max}}{1-\gamma}
 \left\{
 \sqrt{\frac{p_{s,1}\Gamma_n(K,s,m,d,T)+\log(2/\delta)}{n}}
 \right.
\notag\\
 &\hspace{3.2cm}\left.
 +\frac{p_{s,1}\Gamma_n(K,s,m,d,T)+\log(2/\delta)}{n}
 \right\}.
\end{align}
By \eqref{eq:disc-trunc-section4},
\[
 \sup_{f\in\mathcal F_{K,s}}
 |(P_n-P)(G_f-G_{H_n}(f))|
 \le
 \frac{r_{\max}}{1-\gamma}\frac1{4\sqrt n},
\]
which is absorbed into the right side of \eqref{eq:uc-trunc-section4}. Hence the same bound holds with \(G_f\) in place of \(G_{H_n}(f)\), after increasing the numerical constant.

\medskip

Finally, combining \eqref{eq:erm-section4}, \eqref{eq:oracle-section4}, and the uniform-convergence bound from Step 5 gives
\begin{align*}
 V^\star(\rho)-V(\widehat f_{n,K,s})
 \le C_1\Bigg[&
 e^{-2KT}m_2(\mu_0)+\frac{d\Lambda}{K}+C_0^mmdL_{g^\star}^2s^{-2/m} \\
 &+\frac{r_{\max}}{1-\gamma}
 \sqrt{\frac{p_{s,1}\Gamma_n(K,s,m,d,T)+\log(2/\delta)}{n}} \\
 &+\frac{r_{\max}}{1-\gamma}
 \frac{p_{s,1}\Gamma_n(K,s,m,d,T)+\log(2/\delta)}{n}
 \Bigg],
\end{align*}
which is \eqref{eq:ks-tradeoff-main}. This proves Theorem~\ref{thm:lipschitz-selector-complete-tradeoff}.

\section{Proof for Section \ref{sec:policy gradient}}\label{app:policy-gradient}

This appendix proves the policy-gradient formula used in Section~\ref{sec:experiments}. The argument first represents the diffusion policy through a reference driftless diffusion and a Girsanov change of measure, then differentiates expectations with respect to the drift parameter, and finally inserts this derivative into the standard performance-difference identity for discounted MDPs.

\subsection{Constructing policy from a reference probability space using Girsanov's theorem }
Consider a reference probability space $(\Omega', \sF'=(\cF'_t)_{t\in[0,T]},\mathbb{Q})$ where $\{B_s^\sQ\}_{s\in[0,T]}$ is a $d$-dimensional $\sQ$-Brownian motion and 
\begin{equation}\label{diffusion:Q}
    \d \bar a_s =\sigma(s,\bar a_s,x)\cdot \d B^{\sQ}_s,    
\end{equation}
where the existence of the strong solution is guaranteed by the Lipschitz property of $\sigma$. For fixed $\theta\in\sR^k$ and $x\in\cX$, we defined the process $\{M_t(\theta,x)\}_{t\in[0,T]}$ by 
\begin{equation}\label{def:m}
    \begin{aligned}M_t(\theta,x):=\exp\Big(\int_0^tf(&s, \bar a_s,\theta,x)^\top\big(\sigma^\top(s,\bar a_s,x)\big)^{-1}\d B^\sQ_s
    \\
    &-\frac12\int_0^t{f(s, \bar  a_s,\theta,x)^\top}\big((\sigma\sigma^\top)(s,\bar a_s,x)\big)^{-1}f(s, \bar a_s,\theta,x)\d s\Big).
    \end{aligned}
\end{equation}
The following lemma verifies that the conditions of \cite[Corollary 5.14 in Chapter 3]{karatzas1991brownian} hold with an even weaker assumption than Assumption \ref{assum_light:f_theta}.
\begin{lem}\label{lem:mtgcond}
    Assume that for all $x\in\cX$, $\theta\in\sR^k$, there exists a constant $C$ such that 
    $$
    \big\|f(s,a,\theta,x)\big\|\leq C (|a|+1).
    $$
    % and further assume that there exists a constant $C_2$ such that
    % $$
    % C_2^{-1}I_d\prec\sigma(s,a)^2\prec C_2I_d,\text{ for all }s\in[0,T],\text{ and }a\in \sR^d.
    % $$
    Then for any $x\in\cX$, $\theta\in\sR^k$, there exists a sequence of real numbers $0=t_0<t_1<...<t_n=T$  such that
    $$
        \sE^\sQ\bigg[\exp\bigg(\int_{t_i}^{t_{i+1}}f(s, \bar a_s,\theta,x)^\top\big((\sigma\sigma^\top)(s,\bar a_s,x)\big)^{-1}f(s, \bar a_s,\theta,x)\d s\bigg)\bigg]<\infty,
    $$
    for $i=0,1,..,n-1,$ $x\in\cX$, $\theta \in\sR^d$. Consequently, $\{M_t(\theta,x)\}_{t\in[0,T]}$ is a $\sQ$-martingale.
\end{lem}
\begin{proof}
    By the linear growth condition of $f$, we get
    \begin{align*}
        \int_{t_i}^{t_{i+1}}&f(s, \bar a_s,\theta,x)^\top\big((\sigma\sigma^\top)(s,\bar a_s,x)\big)^{-1}f(s, \bar a_s,\theta,x)\d s\leq C\int_{t_i}^{t_{i+1}}\big|f(s, \bar a_s,\theta,x)\big|^2\d s
        \\
        &\leq C\int_{t_i}^{t_{i+1}}\big(|\bar a_s|^2+1\big)\d s\leq C(t_{i+1}-t_{i})\Big(\sup_{s\in[t_i,t_{i+1}]}|\bar a_s|^2+1\Big),
    \end{align*}
    where the second inequality holds since $(|\bar a_s|+1)^2\leq 2(|\bar a_s|^2+1)$ and constant $C$ may vary line by line. Then, we have
    \begin{align*}
    \exp\bigg(\int_{t_i}^{t_{i+1}}f(s, \bar a_s,\theta,x)^\top\big((\sigma\sigma^\top)(s,\bar a_s,x)\big)^{-1}f(s, \bar a_s,\theta,x)\d s\bigg)&\leq \exp\bigg(C(t_{i+1}-t_{i})\Big(\sup_{s\in[t_i,t_{i+1}]}|\bar a_s|^2+1\Big)\bigg)
    \\
    &=\sup_{s\in[t_i,t_{i+1}]}\exp\Big(C(t_{i+1}-t_i)\big(|\bar a_s|^2+1\big)\Big).
    \end{align*}
    Note that $\d\bar a_t=\sigma(t,\bar a_t,x) \cdot\d B_t^\sQ$ and $\sigma(s,a,x)\prec CI_d$, using \cite[Theorem 1.3]{hajek1985mean}, \eqref{initial} in Assumption \ref{assum_light:f_theta} and the fact that $\phi(a)=\exp\big(C(t_{i+1}-t_{i})(|a|^2+1)\big)$ is a non-decreasing convex function, we know 
    $$
        \sE^\sQ\Big[\exp\Big(C(t_{i+1}-t_i)\big(|\bar a_s|^2+1\big)\Big)\Big]\leq \sE\big[\exp\big(C(t_{i+1}-t_i)(|Z|^2+1)\big)\big],
    $$
    where $Z\sim N(0,CsI_d)$. Set $t_{i+1}-t_i$ small enough (for instance $t_{i+1}-t_i= \frac1{4C^2T}$), we have
    \begin{align*}
    \sE\big[\exp\big(C(t_{i+1}-t_i)(|Z|^2+1)\big)\big]<\infty.
    \end{align*}
    Define $Y_s:=\exp\big(C(t_{i+1}-t_i)(|\bar a_s|^2+1)\big)$, then,
    again by $\d\bar a_t=\sigma(t,\bar a_t,x) \cdot\d B_t^\sQ$, and using the fact that $Y_s=\phi(a_s)$ with $\phi(a)=\exp\big(C(t_{i+1}-t_{i})(|a|^2+1)\big)$ is a nondecreasing convex function, we get that $\{\sqrt{Y_s}\}_{s\in[0,T]}$ is a submartingale (\cite[Problem 3.7 in Chapter 1]{karatzas1991brownian}), then using Doob's martingale inequality (\cite[Theorem 3.8 in Chapter 1]{karatzas1991brownian}), we achieve
    $$
        \sE\bigg[\sup_{s\in[t_i,t_{i+1}]}\exp\Big(C(t_{i+1}-t_i)\big(|\bar a_s|^2+1\big)\Big)\bigg]\leq \sE\Big[\sup_{s\in[0,T]}Y_s\Big]\leq 4\sE\big[Y_T\big]<\infty,
    $$
    provided that $t_{i+1}-t_i $ small enough, and the final statements holds due to \cite[Corollary 5.14 in Chapter 3]{karatzas1991brownian}.
\end{proof}

It is not hard to see that our Assumption \ref{assum_light:f_theta} is stronger than the assumptions of Lemma \ref{lem:mtgcond}. Therefore, the Novikov's condition holds, and we can apply Girsanov theorem to define a family of new probability measures $\{\sP^{\theta,x}\}_{\theta\in\sR^k,x\in\cX}$ by
\begin{equation}\label{equ:girsanov}
    \frac{\d \sP^{\theta,x}}{\d \sQ}=M_t(\theta,x),
\end{equation}
and get that
$$
B_t^{\theta,x}:=B_t^\sQ-\int_0^t\sigma(s,\bar a_s,x)^{-1}f(s,\bar a_s,\theta,x)\d s,\text{ is a Brownian motion under }\sP^{\theta,x}.
$$
Therefore, 
$$
\d \bar a_s=\sigma(s,\bar a_s,x)\cdot\d B_s^\sQ =f(s,\bar a_s,\theta,x)\d s +\sigma(s,\bar a_s,x)\cdot \d B_s^{\theta,x}. 
$$
Finally, for any fixed $\theta$ and $x$, the SDE above satisfies the standard SDE with Lipschitz drift and diffusion coefficients, the weak uniqueness holds \cite[Proposition 2.13 in Chapter 5]{karatzas1991brownian}, and therefore, $\bar a_T$ under $\sP^{\theta,x}$ should share same distribution of $\bar a^{\theta,x}_T$ in \eqref{eq:sde-policy}. We summarize our discussion in this subsection as the following proposition.

\begin{prop}
    Suppose Assumption \ref{assum_light:f_theta}, \eqref{sigmabound} and \eqref{sigmalipschitz} hold, consider the SDE \eqref{eq:sde-policy} posed on the probability space $(\Omega,\sF=(\cF_t)_{t\in[0,T]},\sP)$ supporting the $\sP$-Brownian motion $\{B_s\}_{s\in[0,T]}$ and the SDE \eqref{diffusion:Q} posed on the reference probability space $(\Omega',\sF'=(\cF'_t)_{t\in[0,T]},\sQ)$ supporting the $\sQ$-Brownian motion $\{B_s^\sQ\}_{s\in[0,T]}$. Then, a family of probability measure $\{\sP^{\theta,x}\}_{\theta\in\sR^k,x\in\cX}$ are well-defined by \eqref{equ:girsanov}, and 
    $$
    {\sP^{\theta,x}}\circ(\bar a_T)^{-1}={\sP}\circ(\bar a^{\theta,x}_T)^{-1}.
    $$.
\end{prop}

%\YZ{we need to assume $f$ is Lipschitz in $a$}

%\noindent\textbf{We proceed to compute the distribution of $\bar a_T$ with respect to $\theta$.}

\subsection{Derivative for  $\sE^{\sP^{\theta,x}}\big[\varphi(\bar a_T)\big]$ when $\varphi$ is a continuous bounded function independent of $\theta$.}
Let $\varphi$ be a continuous bounded function independent of $\theta$ and by \eqref{equ:girsanov}, we have 
\begin{align*}
\sE^{\sP^{\theta,x}}\big[\varphi(\bar a_T)\big] 
&= \sE^\sQ\big[M_T(\theta,x)\varphi(\bar a_T)\big],
\end{align*}
where $M_t(\theta,x)$ is defined in \eqref{def:m} and further, 
$$
\d M_t(\theta,x) = M_t(\theta,x)f(t, \bar a_t,\theta,x)^\top\sigma(t,\bar a_t,x)^{-1}\d B_t^\sQ, \quad \d \bar a_t = \sigma(t,\bar a_t,x)\cdot \d B^\sQ_t,\text{ for all $t\in[0,T]$.}
$$
We first present two technical lemmas which are crucial for the follow-up analysis.

\begin{lem}\label{lem:34_rescaled}
    Suppose that for every fixed $M>0$ and $\beta\in \mathbb{R}^{k}$,
    \begin{equation}\label{reg_cond}
    \lim_{\eps\to 0}\sup_{\theta\in \mathbb{R}^{d}, x\in \cX}\E^{\mathbb{Q}}\left[\sup_{0\leq s\leq T, 0\leq t\leq \eps}\|\nabla_{\theta} f(s, \bar a_{s}, \theta+t\beta, x)-\nabla_{\theta} f(s, \bar a_{s}, \theta, x)\|_{\op}^{M}\right]=0
    \end{equation}
    Then for any $p\geq 1$,
    \begin{equation}
        \lim_{\epsilon\to 0}\sup_{\theta\in\sR^d,x\in\cX}\sE^\sQ\left[\sup_{0\leq t\leq T}\bigg|\dfrac{1}{\eps}\left(\frac{M_T(\theta+\epsilon\beta,x)}{M_T(\theta,x)}-1\right)- Y_T({\theta,\beta},x)\bigg|^p\right]=0.
    \end{equation}   
\end{lem}
\begin{proof}
We define the following quantities:
\begin{align}
r_{1}(\eps, s, \bar a_{s}, \beta;\theta, x)=\;&f(s, \bar a_{s}, \theta+\eps\beta, x)-f(s, \bar a_{s},\theta, x)-\eps\langle \nabla_{\theta} f(s, \bar a_{s},\theta, x), \beta\rangle;\\
r_{2}(\eps, s, \bar a_{s}, \beta;\theta, x)=\;&\eps^{2}\langle \nabla_{\theta} f, \beta\rangle^{\top}\big((\sigma\sigma^\top)(s, \bar a_{s},x)\big)^{-1}\langle \nabla_{\theta}f, \beta\rangle+2\eps\langle \nabla_{\theta} f, \beta\rangle^{\top}\big((\sigma\sigma^\top)(s, \bar a_{s},x)\big)^{-1}r_{1}(\eps)\nonumber\\
+&r_{1}(\eps)^{\top}\big((\sigma\sigma^\top)(s, \bar a_{s},x)\big)^{-1}r_{1}(\eps)+2r_{1}(\eps)^{\top}\big((\sigma\sigma^\top)(s, \bar a_{s},x)\big)^{-1}f(s, \bar a_{s}, \theta, x),
\end{align}
where $f$ is shorthand for $f(s, \bar a_{s}, \theta, x)$ and $r_{1}$ is shorthand for the quantity defined in the preceding equation. The reason for defining $r_{1}$ is clear: it is the error of the linear approximation of $f$ at $\theta$. Similarly, $r_{2}$ is the remainder term for approximating the quadratic term $f(\theta+\eps\beta)^{\top}((\sigma\sigma^\top)(s, \bar a_{s}, x))^{-1}f(\theta+\eps\beta)$ by $f(\theta)^{\top}\big((\sigma\sigma^\top)(s, \bar a_{s}, x)\big)^{-1}f(\theta)$. We have
$$\dfrac{M_{t}(\theta+\eps\beta, x)}{M_{t}(\theta, x)}=\exp\left(\eps Y_{t}(\theta, \beta, x)+\int_{0}^{t}r_{1}(\eps, s, \bar a_{s}, \beta;\theta, x)\d B_{s}^{Q}-\frac12\int_{0}^{t}r_{2}(\eps, s, \bar a_{s},\beta; \theta, x)\d s\right).$$
Now we use the following inequality: for every $z\in \mathbb{R}$:
$$|e^{z}-1-z|\leq z^{2}e^{|z|},$$
which gives us that
\begin{align*}
&\left|\dfrac{M_{t}(\theta+\eps\beta, x)}{M_{t}(\theta, x)}-1-\left(\eps Y_{t}(\theta, \beta, x)+\int_{0}^{t}r_{1}(\eps, s, \bar a_{s}, \beta;\theta, x)dB_{s}^{Q}-\frac12\int_{0}^{t}r_{2}(\eps, s, \bar a_{s},\beta; \theta, x)\d s\right)\right|\\
\leq&\left(\eps Y_{t}(\theta, \beta, x)+\int_{0}^{t}r_{1}(\eps, s, \bar a_{s}, \beta;\theta, x)\d B_{s}^{Q}-\frac12\int_{0}^{t}r_{2}(\eps, s, \bar a_{s},\beta; \theta, x)\d s\right)^{2}\cdot\\
\leq&\exp\left(\left|\eps Y_{t}(\theta, \beta, x)+\int_{0}^{t}r_{1}(\eps, s, \bar a_{s}, \beta;\theta, x)\d B_{s}^{Q}-\frac12\int_{0}^{t}r_{2}(\eps, s, \bar a_{s},\beta; \theta, x)\d s\right|\right).
\end{align*}

We bound the $M=M(p)$-th moment of the rightmost term, taken supremum over $0\leq t\leq T$, when $\eps$ is sufficiently small ($M$ is constant with respect to $\eps)$. Firstly, by using AM-GM, we will bound the $3M$-th moment of each exponential term. To this end,
\begin{align*}
\mathbb{E}\left[\sup_{0\leq t\leq T}\exp\left(3M\eps|Y_{t}(\theta, \beta, x)|\right)\right]\leq&\E\left[\sup_{0\leq t\leq T}\exp\left(3M\eps\left|\int_{0}^{t}\langle \nabla_{\theta}f, \beta\rangle^{\top}(\sigma^\top(s, \bar a_{s}, x))^{-1}\d B_{s}^{Q}\right|\right)\right]\\
+&\E\left[\sup_{0\leq t\leq T}\exp\left(3M\eps\left|\int_{0}^{t}\langle \nabla_{\theta} f, \beta\rangle^{\top}\big((\sigma\sigma^\top)(s, \bar a_{s}, x)\big)^{-1}f \d s\right|\right)\right].
\end{align*}
For the first term, we further use the following crude bound:
\begin{align*}
\E&\left[\sup_{0\leq t\leq T}\exp\left(3M\eps\left|\int_{0}^{t}\langle \nabla_{\theta}f, \beta\rangle^{\top}\big(\sigma^\top(s, \bar a_{s}, x)\big)^{-1}\d B_{s}^{Q}\right|\right)\right]
\\
\leq\;&\E\left[\sup_{0\leq t\leq T}\exp\left(3M\eps\int_{0}^{t}\langle \nabla_{\theta}f, \beta\rangle^{\top}\big(\sigma^\top(s, \bar a_{s}, x)\big)^{-1}\d B_{s}^{Q}\right)\right]\\
&+\E\left[\sup_{0\leq t\leq T}\exp\left(-3M\eps\int_{0}^{t}\langle \nabla_{\theta}f, \beta\rangle^{\top}\big(\sigma^\top(s, \bar a_{s}, x)\big)^{-1}\d B_{s}^{Q}\right)\right].
\end{align*}
To bound the right-hand side, we let
$$Y_{t}^{(1)}(\theta, \beta, x)=\int_{0}^{t}\langle \nabla_{\theta} f,\beta\rangle^{\top}\big(\sigma^\top(s, \bar a_{s}, x)\big)^{-1}\d B_{s}^{Q},$$
so that $\{Y_{t}^{(1)}\}_{t\geq 0}$ is a martingale. First, by $L^{2}$ martingale inequality we get that
$$\E\left[\sup_{0\leq t\leq T}\exp\left(3M\eps\cdot Y_{t}^{(1)}(\theta,\beta, x)\right)\right]\leq 4\E\left[\exp\left(3M\eps\cdot Y_{T}^{(1)}(\theta, \beta, x)\right)\right].$$
By Burkholder-Davis-Gundy inequality, for every $p\geq 1$ we have
$$\E\left[|Y_{T}^{(1)}(\theta,\beta, x)|^{p}\right]\leq (C_{0}p)^{p}\E[\langle Y^{(1)}(\theta,\beta, x)\rangle_{T}^{p/2}],$$
for an absolute constant $C_{0}$. Using linear growth on $\nabla_{\theta} f$, we get that
$$\langle Y^{(1)}(\theta,\beta, x)\rangle_{T}\leq C_{2}\|\beta\|^{2}T\sup_{0\leq s\leq T}\{C(\|\bar a_{s}\|^{2}+1)\}.$$
Taking expectations on both sides, using $L^{p}$ martingale inequality on $\bar a_{s}$ and operator norm boundedness of $\sigma$, we get that for $p\geq 2$,
$$\E\left[|Y_{T}^{(1)}(\theta,\beta, x)|^{p}\right]\leq(C_{0}p)^{p}\cdot C_{1}^{p}\|\beta\|^{p}d^{p}T^{p}.$$
Hence, by taking $\eps$ such that $M\eps$ is small enough, and noticing that our proof does not change when we have the $-$ sign, we get from Taylor series that
$$\sup_{\theta, x}\E\left[\sup_{0\leq t\leq T}\exp\left(3M\eps\left|\int_{0}^{t}\langle \nabla_{\theta}f, \beta\rangle^{\top}\big(\sigma^\top(s, \bar a_{s}, x)\big)^{-1}\d B_{s}^{Q}\right|\right)\right]<\infty.$$
Next, we deal with the same term with $r_{1}$ in place of $Y_{t}^{(1)}(\theta,\beta, x)$. Define
$$Y_{t}^{(2)}(\eps, \theta,\beta, x)=\int_{0}^{t}r_{1}(\eps, s, \bar a_{s}, \beta; \theta, x)\d B_{s}^{Q},$$
and $(Y_{t}^{(2)}(\theta, \beta, x))$ is a martingale. Proceeding similarly as before, we use the bound
\begin{align*}
\E\left[\sup_{0\leq t\leq T}\exp\left(3M\cdot |Y_{t}^{(2)}(\eps, \theta, \beta, x)|\right)\right]\leq\;&4\E\left[\exp\left(3M\cdot |Y_{T}^{(2)}(\eps, \theta, \beta, x)|\right)\right].
\end{align*}
By Taylor series and monotone convergence,
$$\E\left[\sup_{0\leq t\leq T}\exp\left(3M\cdot |Y_{t}^{(2)}(\eps, \theta, \beta, x)|\right)\right]\leq4\sum_{k=0}^{\infty}\dfrac{(3M)^{k}\E[|Y_{T}^{(2)}(\eps, \theta, \beta, x)|^{k}]}{k!}.$$
For $k\geq 2$, we have by Burkholder-Davis-Gundy inequality that
$$\E[|Y_{T}^{(2)}(\eps, \theta, \beta, x)|^{k}]\leq (C_{0}k)^{k}\E[\langle Y^{(2)}(\eps, \theta, \beta, x)\rangle_{T}^{k/2}],$$
where
$$\langle Y^{(2)}(\eps, \theta, \beta, x)\rangle_{T}=\int_{0}^{T}\|r_{1}(\eps, s, \bar a_{s}, \beta, \theta, x)\|^{2}\d s.$$
We bound $\|r_{1}\|$ as follows
\begin{align*}
\|r_{1}(\eps, s, \bar a_{s},\beta, \theta, x)\|=&\left\|\int_{0}^{\eps}\langle\nabla_{\theta} f(s, \bar a_{s}, \theta+t\beta, x)-\nabla_{\theta}f(s, \bar a_{s}, \theta, x), \beta\rangle \d t\right\|\\
\leq&\int_{0}^{\eps}\|\nabla_{\theta} f(s, \bar a_{s}, \theta+t\beta, x)-\nabla_{\theta} f(s, \bar a_{s},\theta, x)\|_{\op}\cdot \|\beta\|\d t\\
\leq&\eps\cdot \sup_{0\leq t\leq \eps}\|\nabla_{\theta} f(s, \bar a_{s}, \theta+t\beta, x)-\nabla_{\theta} f(s, \bar a_{s},\theta, x)\|_{\op}\cdot \|\beta\|\\
\leq&\eps\|\beta\|\cdot C(\|\bar a_{s}\|+1).
\end{align*} 
Raising to the $k$-th power, we get that
$$\langle Y^{(2)}(\eps, \theta, \beta, x)\rangle_{T}^{k/2}\leq \eps^{k/2}\|\beta\|^{k/2}C^{k}T^{k/2}\sup_{0\leq s\leq T}(\|\bar a_{s}\|^{k}+1).$$
From $L^{k}$ martingale and Burkholder-Davis-Gundy inequalities, we get that
$$\E\left[\sup_{0\leq s\leq T}\|\bar a_{s}\|^{k}\right]\leq 4\E\left[\|\bar a_{T}\|^{k}\right]\leq 4\cdot \E\left[\langle \bar a\rangle_{T}^{k/2}\right]\leq 4d^{k}C_{2}^{k/2}T^{k/2}.$$
Consequently,
$$\E\left[|Y_{T}^{(2)}(\eps,\theta, \beta, x)|^{k}\right]\leq (C_{0}k)^{k}\cdot \eps^{k/2}\|\beta\|^{k/2}C^{k}d^{k}C_{2}^{k/2}T^{k},$$
and so
$$\E\left[\sup_{0\leq t\leq T}\exp\left(3M\cdot |Y_{t}^{(2)}(\eps, \theta, \beta, x)|\right)\right]\leq4\sum_{k=0}^{\infty}\dfrac{(3M)^{k}(C_{0}k)^{k}\cdot \eps^{k/2}\|\beta\|^{k/2}C^{k}d^{k}C_{2}^{k/2}T^{k}}{k!},$$
and by making $\eps$ small enough, the series on the RHS converges. Lastly, we bound the term concerning $r_{2}$. We have
\begin{align*}
\E&\left[\sup_{0\leq t\leq T}\exp\left(3M\left|\int_{0}^{t}r_{2}(\eps, s, \bar a_{s}, \theta, \beta, x)ds\right|\right)\right]
\\
\leq&\;\E\left[\sup_{0\leq t\leq T}\exp\left(3M\int_{0}^{t}|r_{2}(\eps, s, \bar a_{s}, \theta, \beta, x)|ds\right)\right]\\
=&\;\E\left[\exp\left(3M\int_{0}^{T}|r_{2}(\eps, s, \bar a_{s}, \theta,\beta, x)|ds\right)\right]\\
\leq&\;\E\left[\sup_{0\leq s\leq T}\exp\left(3MT\cdot |r_{2}(\eps, s, \bar a_{s}, \theta, \beta, x)|\right)\right].
\end{align*}
By triangle inequality,
\begin{align*}
|r_{2}(\eps, s, \bar a_{s}, \theta, \beta, x)|\leq C(\eps+\eps^{2})\|\beta\|^{2}(\|\bar a_{s}\|^{2}+1)\leq C\eps\|\beta\|^{2}(\|\bar a_{s}\|^{2}+1),
\end{align*}
and hence
\begin{align*}
\E&\left[\sup_{0\leq t\leq T}\exp\left(3M\left|\int_{0}^{t}r_{2}(\eps, s, \bar a_{s}, \theta, \beta, x)ds\right|\right)\right]
\\
\leq\;&\E\left[\sup_{0\leq s\leq T}\exp\left(3MT\cdot C\eps\|\beta\|^{2}(\|\bar a_{s}\|^{2}+1)\right)\right]\\
\leq\;&4\E\left[\exp\left(3MT\cdot C\eps\|\beta\|^{2}(\|\bar a_{T}\|^{2}+1)\right)\right]\\
=\;&4\exp(3MT\cdot C\eps\|\beta\|^{2})\E\left[\exp\left(3MT\cdot C\eps\|\beta\|^{2}\cdot \|\bar a_{T}\|^{2}\right)\right]\\
\leq\;&4\exp(3MT\cdot C\eps\|\beta\|^{2})\E\left[\exp\left(3MT\cdot C\eps\|\beta\|^{2}\cdot \|\mathcal{N}(0, C_{2}TI_{d})\|^{2}\right)\right].
\end{align*}
By making $\eps$ small enough as a function of $(M, T, \beta)$, we get that the RHS is finite. We summarize that for each fixed $M$ and $\eps=\eps(M)$ small enough,
$$\sup_{\theta, x}\E\left[\exp\left(M\eps \left|Y_{t}(\theta, \beta, x)+\int_{0}^{t}r_{1}(\eps, s, \bar a_{s}, \beta;\theta, x)\d B_{s}^{Q}-\frac12\int_{0}^{t}r_{2}(\eps, s, \bar a_{s},\beta; \theta, x)\d s\right|\right)\right]<\infty.$$
Next, we show that for every fixed $M=M(p)$, 
$$\lim_{\eps\to 0}\sup_{\theta, x}\E\left[\left|\dfrac{1}{\eps}\left(\eps Y_{t}(\theta, \beta, x)+\int_{0}^{t}r_{1}(\eps, s, \bar a_{s}, \beta;\theta, x)\d B_{s}^{Q}+\int_{0}^{t}r_{2}(\eps, s, \bar a_{s},\beta; \theta, x)\d s\right)^{2}\right|^{M}\right]=0.$$
However, this is clear from the fact that upper bounds of $\|r_{1}\|, \|r_{2}\|$ are linear in $\eps$. Now it suffices to show that
$$\lim_{\eps\to 0}\sup_{\theta, x}\E\left[\sup_{0\leq t\leq T}\left|\int_{0}^{t}\dfrac{1}{\eps}r_{1}(\eps, s, \bar a_{s}, \theta, \beta, x)dB_{s}^{Q}+\int_{0}^{t}\dfrac{1}{\eps}r_{2}(\eps, s, \bar a_{s}, \beta, \theta, x)\right|^{M}\right]=0.$$
By $L_{p}$-martingale inequality, we get that
\begin{align*}
\E\left[\sup_{0\leq t\leq T}\left|\int_{0}^{t}\dfrac{1}{\eps}r_{1}(\eps, s, \bar a_{s}, \theta, \beta, x)\d B_{s}^{Q}\right|^{M}\right]\leq\;& 4\E\left[\left|\int_{0}^{T}\dfrac{1}{\eps}r_{1}(\eps, s, \bar a_{s}, \theta, \beta, x)\d B_{s}^{Q}\right|^{M}\right]\\
\leq\;&4C(M)\cdot \E\left[\left(\int_{0}^{T}\dfrac{1}{\eps^{2}}\|r_{1}(\eps, s, \bar a_{s}, \theta, \beta, x)\|^{2}\d s\right)^{M/2}\right].
\end{align*}
Using the upper bound on $\|r_{1}\|$ earlier, we get
\begin{align*}
&\E\left[\sup_{0\leq t\leq T}\left|\int_{0}^{t}\dfrac{1}{\eps}r_{1}(\eps, s, \bar a_{s}, \theta, \beta, x)\d B_{s}^{Q}\right|^{M}\right]\\
\leq\;& 4C(M)\cdot T^{M/2}\|\beta\|^{M}\cdot \E\left[\sup_{0\leq s\leq T, 0\leq t\leq \eps}\|\nabla_{\theta} f(s, \bar a_{s}, \theta+t\beta, x)-\nabla_{\theta} f(s, \bar a_{s}, \theta, x)\|_{\op}^{M}\right].
\end{align*}
Next, for $r_{2}$, note that triangle inequality gives us
$$\dfrac{1}{\eps}|r_{2}(\eps, s, \bar a_{s}, \theta, \beta, x)|\leq C\eps\|\beta\|^{2}(\|\bar a_{s}\|^{2}+1)+\dfrac{1}{\eps}\left\|r_{1}(\eps, s, \bar a_{s}, \theta, \beta, x)\right\|C(\|\bar a_{s}\|+1).$$
We get
\begin{align*}
&\E\left[\sup_{0\leq t\leq T}\left|\int_{0}^{t}\dfrac{1}{\eps}r_{2}(\eps, s, \bar a_{s}, \theta, \beta, x)\d s\right|^{M}\right]\\
\leq\;& \E\left[\left(\int_{0}^{T}\dfrac{1}{\eps}|r_{2}(\eps, s, \bar a_{s}, \theta, \beta, x)|\d s\right)^{M}\right]\\
\leq\;& C(M)\cdot \left\{\E\left[(CT\cdot \eps \|\beta\|^{2})^{M}\sup_{0\leq s\leq T}(\|\bar{a}_{s}\|^{2M}+1)\right]+\E\left[T^{M}\sup_{0\leq s\leq T}\left(\dfrac{1}{\eps}\|r_{1}(\eps, s, \bar a_{s}, \theta, \beta, x)\|\right)^{M}\cdot C(\|\bar a_{s}\|^{M}+1)\right]\right\}\\
\leq\;& O_{T, \beta}(\eps^{M})+C(M)T^{M}\E\left[\sup_{0\leq s\leq T}\left(\dfrac{1}{\eps}\|r_{1}(\eps ,s, \bar a_{s}, \theta, \beta, x)\|\right)^{2M}\right]^{1/2}\E\left[\sup_{0\leq s\leq T}C(\|\bar a_{s}\|^{2M}+1)\right]^{1/2}\\
\leq\;& O_{T,\beta}(\eps^{M})+C(M)T^{M}\|\beta\|^{M}\E\left[\sup_{0\leq s\leq T, 0\leq t\leq \eps}\|\nabla_{\theta}f(s, \bar a_{s}, \theta+t\beta, x)-\nabla_{\theta} f(s, \bar a_{s}, \theta, x)\|_{\op}^{2M}\right]^{1/2}.
\end{align*}
By Condition \eqref{reg_cond}, we are done.
\end{proof}

\begin{lem}\label{lem:Mbdd_weak} Suppose Assumption \ref{assum_light:f_theta}, \eqref{sigmabound} and \eqref{sigmalipschitz} hold. Then there exists $c=c(T)>0$ small enough so that with $p=1+c$,
$$\sup_{\theta\in \mathbb{R}^{k}, x\in \mathcal{X}}\E^{\mathbb{Q}}\left[\sup_{0\leq t\leq T}M_{t}(\theta, x)^{p}\right]<\infty.$$
\end{lem}
\begin{proof}
By the proof of Theorem 1 from \cite{grigelionis2003finiteness}, we get that
\begin{align*}
\E^{\mathbb{Q}}\left[M_{t}(\theta, x)^{p}\right]\leq&\;\E\left[\exp\left(\dfrac{p(\sqrt{p}+\sqrt{p-1})\sqrt{p-1}}{2}\langle M(\theta, x)\rangle_{t}\right)\right]\\
=&\;\E\left[\exp\left(\dfrac{p(\sqrt{p}+\sqrt{p-1})\sqrt{p-1}}{2}\cdot \int_{0}^{t}f(s, \bar a_{s}, \theta, x)^{\top}\sigma_{s}^{-2}f(s, \bar a_{s},\theta, x)\text{d}s\right)\right]\\
\leq&\;\E\left[\exp\left(\dfrac{p(\sqrt{p}+\sqrt{p-1})\sqrt{p-1}}{2}\cdot \sup_{0\leq s\leq t}C(\|\bar a_{s}\|^{2}+1)\cdot t\right)\right]\\
\leq&\;\E\left[\sup_{0\leq s\leq t}\exp\left(\dfrac{p(\sqrt{p}+\sqrt{p-1})\sqrt{p-1}}{2}\cdot Ct\cdot (\|\bar a_{s}\|^{2}+1)\right)\right]\\
\overset{(a)}\leq&\;4\E\left[\exp\left(\dfrac{p(\sqrt{p}+\sqrt{p-1})\sqrt{p-1}}{2}\cdot Ct\cdot (\|\bar a_{t}\|^{2}+1)\right)\right],
\end{align*}
where in $(a)$ we use $L_{2}$-martingale inequality (here it does not matter if the inequality is vacuous - i.e. the last expression is $\infty$). Next, noting that $\text{d}\bar a_{t}=\sigma(t, \bar a_{t},x)\text{d} t$, either by the diffusion mean-comparison theorem or Burkholder-Davis-Gundy inequality in the proof of Lemma \ref{lem:34_rescaled}, we can choose $p>1$ so that the last expression above is finite. Hence there exists $c=c(T)$ small enough, such that for $p=1+c$,
$$\E^{Q}[M_{t}(\theta, x)^{p}]\leq C(T),$$
for a constant $C=C(T)$. By $L_{p}$-martingale inequality, we get that
$$\E^{\mathbb{Q}}\left[\sup_{0\leq t\leq T}M_{t}(\theta, x)^{p}\right]\leq \left(\dfrac{p}{p-1}\right)^{p}C(T).$$
Since our argument does not matter on $\theta, x$,
$$\sup_{\theta\in \mathbb{R}^{k}, x\in \mathcal{X}}\E^{\mathbb{Q}}\left[\sup_{0\leq t\leq T}M_{t}(\theta, x)^{p}\right]<\infty.$$
\end{proof}

To compute the gradient of 
$\sE^{\sP^{\theta,x}}\big[\varphi(\bar a_T)\big] $ in $\theta$, we would like to compute that, for any $\beta\in\sR^k$,
% \lim_{\epsilon\to0}\frac{1}{\epsilon}\bigg(\sE^\sQ\left[M^{\theta+\epsilon\beta}_T\varphi_{\theta+\epsilon\beta}(\bar a_T)\right]-\sE^\sQ\left[M^{\theta}_T\varphi_{\theta+\epsilon\beta}(\bar a_T)\right]\bigg),......
\begin{equation*}
   \lim_{\epsilon\to 0}\frac{1}{\epsilon}\Big(\sE^{\sP^{\theta+\epsilon\beta,x}}\big[\varphi(\bar a_T)\big]-\sE^{\sP^{\theta,x}}\big[\varphi(\bar a_T)\big]\Big)=\lim_{\epsilon\to0}\frac{1}{\epsilon}\sE^\sQ\left[\big(M_T({\theta+\epsilon\beta},x)-M_T({\theta},x)\big)\varphi(\bar a_T)\right].
\end{equation*}
Define $Y_t(\theta,\beta,x)$ by
\begin{equation*}
    \begin{aligned}
    Y_t(&\theta,\beta,x) := \int_0^t{\big\langle\nabla_\theta f(s,\bar a_s,\theta,x),\beta\big\rangle^\top }{\big(\sigma^\top(s,\bar a_s,x)\big)^{-1}}\d B_s^\sQ
    \\
    &-\int_0^t{\big\langle\nabla_\theta f(s,\bar a_s,\theta,x),\beta\big\rangle^\top \big((\sigma\sigma^\top)(s,\bar a_s,x)\big)^{-1} f(s,\bar a_s,\theta,x)}\d s,
    \end{aligned}
\end{equation*}
for $t\geq 0$ and $\langle\nabla_\theta f,\beta\rangle=\big(\langle\nabla_\theta f_1,\beta\rangle,..,\langle\nabla_\theta f_d,\beta\rangle\big)^\top\in\sR^d$. Recall $\bar a_t=\int_0^t\sigma(s,\bar a_s,x) \cdot \d B_s^\sQ$, $f,\nabla_\theta f$ have linear growth on $a$ uniformly in $s,\theta,x$, it is not hard to see that 
\begin{equation}\label{bddY}
\sup_{\theta\in\sR^d,x\in\cX}\sE^\sQ\Big[\sup_{0\leq t\leq T}\big|Y_t(\theta,\beta,x)\big|^p\Big]<\infty,    
\end{equation}
for any $\beta\in\sR^k$ and any finite $p\geq 1$. Set $V_t({\theta,\beta},x):=Y_t({\theta,\beta},x)M_t(\theta,x)$, and use It\^o's formula, we have
\begin{equation*}
    \begin{aligned}
    \d V_t(&{\theta,\beta},x)=M_t(\theta,x)\d Y_t({\theta,\beta},x)+Y_t({\theta,\beta},x)\d M_t(\theta,x)+\d \big\langle M({\theta},x),Y({\theta,\beta},x)\big\rangle_t
    \\
    &=\bigg({M_t(\theta,x)\big\langle\nabla_\theta f(t,\bar a_t,\theta,x),\beta\big\rangle^\top }{\big(\sigma^\top(t,\bar a_t,x)\big)^{-1}}+{V_t(\theta,\beta,x)f(t,\bar a_t,\theta,x)^\top}{\big(\sigma^\top(t,\bar a_t,x)\big)^{-1}}\bigg)\d B_t^\sQ.
    \end{aligned}
\end{equation*} 
% Set $V_t^{\theta,\beta}(x)$ being the solution of the following SDE,
% \begin{equation*}
%     \d V^{\theta,\beta}_t(x)=\bigg(\frac{\big\langle\nabla_\theta f_\theta(t,\bar a_t,x),\beta\big\rangle M_t(\theta,x)}{\sigma_t}+\frac{f_\theta(t,\bar a_t,x)V^{\theta,\beta}_t(x)}{\sigma_t}\bigg)\d B_t^\sQ.
% \end{equation*}
Next lemma proves a similar result as in \cite[Lemma 4.7]{carmona2016lectures}.
\begin{lem}\label{lem:34_weak} For any $\beta\in \mathbb{R}^{k}$, there exists $c=c(T)>0$ independent of $\beta$ such that for $p=1+c$, we have
$$\lim_{\epsilon\to 0}\sup_{\theta\in\sR^d,x\in\cX}\sE^\sQ\left[\sup_{0\leq t\leq T}\bigg|\frac{M_t(\theta+\epsilon\beta,x)-M_t(\theta,x)}{\epsilon}-V_t({\theta,\beta},x)\bigg|^p\right]=0$$
\end{lem}
\begin{proof}
We choose $c_{1}=c_{1}(T)$ from Lemma \ref{lem:Mbdd_weak} and set $c=c_{1}/2$. With $p=1+c$, we have
\begin{align*}
&\sE^\sQ\left[\sup_{0\leq t\leq T}\bigg|\frac{M_t(\theta+\epsilon\beta,x)-M_t(\theta,x)}{\epsilon}-V_t({\theta,\beta},x)\bigg|^p\right]\\
\leq &\E^{\sQ}\left[\sup_{0\leq t\leq T}\left|\dfrac{1}{\eps}\left(\dfrac{M_{t}(\theta+\eps\beta, x)}{M_{t}(\theta, x)}-1\right)-Y_{t}(\theta, \beta, x)\right|^{p}\cdot \sup_{0\leq t\leq T}M_{t}(\theta, x)^{p}\right]\\
\leq&\E^{\sQ}\left[\sup_{0\leq t\leq T}M_{t}(\theta, x)^{1+c_{1}}\right]^{\frac{p}{1+c_{1}}}\cdot \E^{\sQ}\left[\sup_{0\leq t\leq T}\left|\dfrac{1}{\eps}\left(\dfrac{M_{t}(\theta+\eps\beta, x)}{M_{t}(\theta, x)}-1\right)-Y_{t}(\theta, \beta, x)\right|^{\frac{1+c_{1}}{c}}\right]^{\frac{c}{1+c_{1}}}
\end{align*}
where the last line uses Holder's inequality with conjugates $(1+c_{1})/p$ and $(1+c_{1})/c$. Now, we use Lemma \ref{lem:34_rescaled} with $(1+c_{1})/c<\infty$ to obtain that
$$\lim_{\eps\to 0}\sup_{\theta\in \mathbb{R}^{k}, x\in \mathcal{X}}\E^{\sQ}\left[\sup_{0\leq t\leq T}\left|\dfrac{1}{\eps}\left(\dfrac{M_{t}(\theta+\eps\beta, x)}{M_{t}(\theta, x)}-1\right)-Y_{t}(\theta, \beta, x)\right|^{\frac{1+c_{1}}{c}}\right]^{\frac{c}{1+c_{1}}}=0,$$
and Lemma \ref{lem:Mbdd_weak} to obtain that
$$\sup_{\theta\in \mathbb{R}^{k}, x\in \mathcal{X}} \E^{\sQ}\left[\sup_{0\leq t\leq T}M_{t}(\theta, x)^{1+c_{1}}\right]^{\frac{p}{1+c_{1}}}<\infty,$$
which concludes the statement.
\end{proof}

Together with the boundedness of the $\varphi$, we have 
\begin{equation*}
    \lim_{\epsilon\to 0}\frac{1}{\epsilon}\sE^\sQ\big[\big(M_T(\theta+\epsilon\beta,x)-M_T(\theta,x)\big)\varphi(\bar a_T)\big]=\sE^\sQ\big[V_T(\theta,\beta,x)\varphi(\bar a_T)\big],
\end{equation*}
Define  $Y_t^{}(\theta,\beta,x)=\frac{V_t^{}(\theta,\beta,x)}{M_t(\theta,x)}$ and use Ito's formula, we have
\begin{equation*}
    \begin{aligned}
    \d Y_t^{}(\theta,\beta,x) &= \frac{\d V_t^{}(\theta,\beta,x)}{M_t(\theta,x)}-\frac{V_t^{}(\theta,\beta,x)}{(M_t(\theta,x))^2}\d M_t(\theta,x)-\frac{\d \big\langle V^{}(\theta,\beta,x),M(\theta,x)\big\rangle_t}{M_t(\theta,x)^2}
    \\
    &\qquad\qquad+\frac{V_t(\theta,\beta,x)}{(M_t(\theta,x))^3}\d \big\langle M(\theta,x)\big\rangle_t(x)
    \\
    &={\big\langle\nabla_\theta f_\theta(t,\bar a_t,x),\beta\big\rangle }(\sigma^\top)^{-1}\d B_t^\sQ-{\big\langle\nabla_\theta f_\theta(t,\bar a_t,x),\beta\big\rangle (\sigma\sigma^\top)^{-1}f_\theta(t,\bar a_t,x)}\d t.
    \end{aligned}
\end{equation*}
Thus,
\begin{equation*}
    \begin{aligned}
    \lim_{\epsilon\to 0}\frac{1}{\epsilon}\sE^\sQ&\left[\big(M^{}_T(\theta+\epsilon\beta,x)-M^{}_T(\theta,x)\big)\varphi(\bar a_T)\right]=\sE^\sQ\big[M_T(\theta ,x)Y_T^{}(\theta,\beta,x)\varphi(\bar a_T)\big]
    \\
    &=\sE^{\sP_\theta}\big[Y_T^{}(\theta,\beta,x)\varphi(\bar a_T)\big]
    \\
    &=\sE^{\sP_\theta}\bigg[\int_0^T\Big({\big\langle\nabla_\theta f_\theta(t,\bar a_t,x),\beta\big\rangle }(\sigma^\top)^{-1}\d B_t^\sQ-{\big\langle\nabla_\theta f_\theta(t,\bar a_t,x),\beta\big\rangle (\sigma\sigma^\top)^{-1}f_\theta(t,\bar a_t,x)}\d t\Big)\varphi(\bar a_T)\bigg]
    \\
    &=\sE^{\sP_\theta}\bigg[\bigg(\int_0^T{\big\langle\nabla_\theta f_\theta(t,\bar a_t,x),\beta\big\rangle }{(\sigma^\top)^{-1}}\d B_t\bigg)\varphi(\bar a_T)\bigg].
    \end{aligned}
\end{equation*}
To conclude, let $U^\varphi(\theta):=\sE^{\sP_\theta}\big[\varphi(\bar a_T)\big]$, using the previous calculation, we have
$$
\nabla_\theta U^\varphi(\theta)=\sE^{\sP_\theta}\bigg[\bigg(\int_0^T{\nabla_\theta f_\theta(t,\bar a_t,x)}\big(\sigma^\top(t,\bar a_t,x)\big)^{-1}\d B_t\bigg)\varphi(\bar a_T)\bigg].
$$

\subsection{Derivative for value function $V^{\pi_\theta}(\rho)$.}
% Next, we proceed to the policy gradient. In particular ,we focus on the case of infinite time horizon and no entropy. That is
% $$
% V^{\pi_\theta}(x)=\sE^{\pi_\theta}\bigg[\sum_{h=0}^\infty \gamma^h \left(c(x_h,a_h)
% +\tau \operatorname{KL}(\pi_\theta(\cdot|x_h)|\mu)\right)|x_0=x\bigg]\text{ and }V^{\pi_\theta}(\rho) = \int_SV^{\pi_\theta}(x)\rho (\d x),
% $$
Next, we proceed to the policy gradient. Recall that
$$
V^{\pi_\theta}(x)=\sE^{\pi_\theta}\bigg[\sum_{h=0}^\infty \gamma^h r(x_h,a_h)\Big|x_0=x\bigg]\text{ and }V^{\pi_\theta}(\rho) = \int_SV^{\pi_\theta}(x)\rho (\d x).
$$
Using \cite[Lemma 2.3]{kerimkulov2025fisher} , we have 
\begin{align*}
  &V^{\pi_{\theta+\epsilon\beta}}(\rho)
  -V^{\pi_{\theta}}(\rho)
  \\
  &=\frac{1}{1-\gamma}\int_\cX\bigg(\int_\cA Q^{\pi_{\theta}} (x,a) (\pi_{\theta+\epsilon \beta}-\pi_{\theta})\big(\d a\big|x\big)
 \bigg)d_\rho^{\pi_{\theta+\epsilon\beta}} (\d x).
\end{align*}
Therefore, we have 
\begin{align*}
  &\frac{1}\epsilon \Big(V^{\pi_{\theta+\epsilon\beta}}(\rho)
  -V^{\pi_{\theta}}(\rho)\Big)
  \\
  &=\underbrace{\frac{1}{1-\gamma}\int_\cX\frac1\epsilon\bigg(\int_\cA Q^{\pi_{\theta}} (x,a) (\pi_{\theta+\epsilon \beta}-\pi_{\theta})\big(\d a\big|x\big)
 \bigg)\Big(d_\rho^{\pi_{\theta+\epsilon\beta}} (\d x)-d_\rho^{\pi_{\theta}} (\d x)\Big)}_{I_{1,\epsilon}}
 \\
 &\qquad+\underbrace{\frac{1}{1-\gamma}\int_\cX\frac1\epsilon\bigg(\int_\cA Q^{\pi_{\theta}} (x,a) (\pi_{\theta+\epsilon \beta}-\pi_{\theta})\big(\d a\big|x\big)
 \bigg)d_\rho^{\pi_{\theta}} (\d x)}_{I_{2,\epsilon}}.
\end{align*}
By Lemma \ref{lem:grad:fix} whose proof only requires boundedness of the test function, we have
\begin{align*}
    \lim_{\epsilon\to 0}\sup_{x\in\mathcal{X}}\;&\Bigg|\frac1{\epsilon}\int_\cA Q^{\pi_{\theta}} (x,a)
  (\pi_{\theta+\epsilon \beta}-\pi_{\theta})\big(\d a\big|x\big)
  \\
  &-\sE^{\sP_\theta}\bigg[\bigg(\int_0^T{\langle\nabla_\theta f_\theta(t,\bar a_t,x),\beta\rangle}\big(\sigma^\top(t,\bar a_t,x)\big)^{-1}\d B_t\bigg)Q^{\pi_{\theta}} (x,\bar a_T)\bigg]\Bigg|=0.
\end{align*}
Then, by Dominated Convergence theorem, we have
$$
I_{2,\epsilon}\to \dfrac{1}{1-\gamma}\int_\cX\sE^{\sP_\theta}\bigg[\bigg(\int_0^T{\nabla_\theta f_\theta(t,\bar a_t^x,x)}\big(\sigma^\top(t,\bar{a}^x_t,x)\big)^{-1}\d B_t\bigg)Q^{\pi_{\theta}} (x,\bar a_T^x)
   \bigg]d_\rho^{\pi_\theta} (\d x),
$$
as $\epsilon\to 0$, and 
\begin{align*}
    &\Bigg|I_{1,\epsilon}-\dfrac{1}{1-\gamma}\int_\cX\sE^{\sP_\theta}\bigg[\bigg(\int_0^T{\nabla_\theta f_\theta(t,\bar a_t^x,x)}\big(\sigma^\top(t,\bar{a}^x_t,x)\big)^{-1}\d B_t\bigg)Q^{\pi_{\theta}} (x,\bar a_T^x)
   \bigg]\Big(d_\rho^{\pi_{\theta+\epsilon\beta}} (\d x)-d_\rho^{\pi_{\theta}} (\d x)\Big)\Bigg|
   \\
   &\;\leq \sup_{x\in\mathcal{X}}\;\Bigg|\frac1{\epsilon}\int_\cA Q^{\pi_{\theta}} (x,a)
  (\pi_{\theta+\epsilon \beta}-\pi_{\theta})\big(\d a\big|x\big)-\sE^{\sP_\theta}\bigg[\bigg(\int_0^T{\langle\nabla_\theta f_\theta(t,\bar a_t,x),\beta\rangle}\big(\sigma^\top(t,\bar a_t,x)\big)^{-1}\d B_t\bigg)Q^{\pi_{\theta}} (x,\bar a_T)\bigg]\Bigg|
  \\
  &\qquad\cdot\dfrac{1}{1-\gamma}\int_\cX\Big(d_\rho^{\pi_{\theta+\epsilon\beta}} (\d x)+d_\rho^{\pi_{\theta}} (\d x)\Big)\to 0\text{, when }\epsilon\to 0.
\end{align*}
Moreover, by Assumption \eqref{assum_light:f_theta}, we know that 
$$
\sE^{\sP_\theta}\bigg[\bigg(\int_0^T{\nabla_\theta f_\theta(t,\bar a_t^x,x)}\big(\sigma^\top(t,\bar{a}^x_t,x)\big)^{-1}\d B_t\bigg)Q^{\pi_{\theta}} (x,\bar a_T^x)
   \bigg]
$$
is uniformly bounded for all $x\in\mathcal X$. It suffices to show that 
$$
\int_\mathcal{X}\psi(x)\big(d_\rho^{\pi_{\theta+\epsilon\beta}} (\d x)-d_\rho^{\pi_{\theta}} (\d x)\big)\to 0.
$$
for any bounded measurable function $\psi(x)$. Recall the definition of $d_\rho^\pi(\d x)$, we have
$$
d_\rho^\pi(\d x)=\int_\mathcal{X}d^\pi(\d x|x')\rho(\d x')=(1-\gamma)\sum_{n=0}^\infty\gamma^n P_{\pi}^n(\d x|x')\rho(\d x')
$$
and 
$$
P_\pi^1(\d x|x')=\int_AP(\d x|x',a)\pi(\d a|x'),\quad P_\pi^n(\d x|x')=\int_\mathcal{X}P_\pi^{n-1}(\d x|x'')P_\pi^1(\d x''|x').
$$
By Lemma \ref{lem:34_weak}, there exist constant $C,\epsilon_0>0$ depending only on $\|\psi\|_\infty$ such that for $\epsilon<\epsilon_0$
\begin{align*}
&\sup_{x'\in\mathcal{X}}\bigg|\int_{\mathcal{X}}\psi(x)\big(P_{\pi_{\theta+\epsilon\beta}}^1(\d x|x')-P_{\pi_\theta}^1(\d x|x')\big)\bigg|
\\
&=\sup_{x'\in\mathcal{X}}\Bigg|\int_A\bigg|\int_X\psi(x)P(\d x|x',a)\bigg|\big({\pi_{\theta+\epsilon\beta}}(\d a|x')-\pi_\theta(\d a|x')\big)\Bigg|
\leq C\epsilon, 
\end{align*}
Then, if we have shown $\sup_{x'\in\mathcal{X}}\big|\int_{\mathcal{X}}\psi(x)\big(P_{\pi_{\theta+\epsilon\beta}}^{n-1}(\d x|x')-P_{\pi_\theta}^{n-1}(\d x|x')\big)\big|\leq C(n-1)\epsilon$, by induction, we get
\begin{align*}
&\sup_{x'\in\mathcal{X}}\bigg|\int_{\mathcal{X}}\psi(x)\big(P_{\pi_{\theta+\epsilon\beta}}^n(\d x|x')-P_{\pi_\theta}^n(\d x|x')\big)\bigg|
\\
&\leq \sup_{x'\in\mathcal{X}}\bigg|\int_A\bigg|\int_X\psi(x)P^{n-1}_{\pi_{\theta+\epsilon\beta}}(\d x|x'')\bigg|\big(P^1_{\pi_{\theta+\epsilon\beta}}(\d x''|x')-P^1_{\pi_{\theta}}(\d x''|x')\big)
\\
&\qquad+\sup_{x'\in\mathcal{X}}\bigg|\int_A\bigg|\int_X\psi(x)\big(P^{n-1}_{\pi_{\theta+\epsilon\beta}}(\d x|x'')-P^{n-1}_{\pi_{\theta}}(\d x|x'')\big)\bigg|P^1_{\pi_{\theta}}(\d x''|x')\big)\leq C(n-1)\epsilon+C\epsilon=Cn\epsilon, 
\end{align*}
Finally, combining the estimate
\[
\sup_{x'\in\mathcal X}
\left|
\int_{\mathcal X}
\psi(x)
\big(
P_{\pi_{\theta+\epsilon\beta}}^n
-
P_{\pi_\theta}^n
\big)(dx|x')
\right|
\le Cn\epsilon,
\]
with the definition of the discounted occupancy measure,
\[
d_\rho^\pi(dx)
=
(1-\gamma)
\sum_{n=0}^{\infty}
\gamma^n
P_\pi^n(dx|x')\rho(dx'),
\]
we obtain
\begin{align*}
&\left|
\int_{\mathcal X}
\psi(x)
\big(
d_\rho^{\pi_{\theta+\epsilon\beta}}
-
d_\rho^{\pi_\theta}
\big)(dx)
\right|
\\
&\le
(1-\gamma)
\sum_{n=0}^{\infty}
\gamma^n
\sup_{x'\in\mathcal X}
\left|
\int_{\mathcal X}
\psi(x)
\big(
P_{\pi_{\theta+\epsilon\beta}}^n
-
P_{\pi_\theta}^n
\big)(dx|x')
\right|
\\
&\le
C\epsilon(1-\gamma)
\sum_{n=0}^{\infty}
n\gamma^n
=
\frac{C\gamma}{1-\gamma}\epsilon
\longrightarrow 0,
\end{align*}
as $\epsilon\to0$. Therefore,
\[
\int_{\mathcal X}
\psi(x)
\big(
d_\rho^{\pi_{\theta+\epsilon\beta}}
-
d_\rho^{\pi_\theta}
\big)(dx)
\longrightarrow 0,
\]
which completes the proof.

\end{document}